\documentclass[12pt]{article}
\usepackage{amsmath}
\usepackage{graphicx,psfrag,epsf,color}
\usepackage{enumerate}
\usepackage{natbib}
\usepackage{subcaption}

\usepackage[unicode=true,pdfusetitle,
   bookmarks=true, 
   bookmarksnumbered=false, 
   bookmarksopen=false,
   breaklinks=false, 
   pdfborder={0 0 1}, 
   backref=false, 
   colorlinks=false]{hyperref}
   
\usepackage{setspace}
\usepackage{amsthm}
\usepackage{amssymb}
\usepackage{multirow}
\usepackage{array}
\usepackage{bm}
\usepackage{booktabs}
\usepackage{float}
\usepackage[font=small,labelfont=bf]{caption}
\setcitestyle{citesep={;}}
\usepackage[total={6.4in,8.6in}]{geometry}
\newtheorem{theorem}{Theorem}
\newtheorem{assumption}{Assumption}[section]
\newtheorem{corollary}[theorem]{Corollary}
\newtheorem{proposition}[theorem]{Proposition}

\usepackage{authblk}

\usepackage{algorithm}
\usepackage{algorithmicx}
\usepackage{algpseudocode}

\begin{document}

\def\spacingset#1{\renewcommand{\baselinestretch}%
{#1}\small\normalsize} \spacingset{1}
%
%
\title{\bf \LARGE Foresighted Online Policy Optimization with Interference}
\author[1]{Liner Xiang} 
\author[2]{Jiayi Wang} 
\author[1]{Hengrui Cai}
\affil[1]{Department of Statistics, University of California, Irvine} 
\affil[2]{Department of Mathematical Sciences, University of Texas at Dallas} 
 \date{}
 \maketitle  

\baselineskip=21pt

\begin{abstract} 
Contextual bandits, which leverage the baseline features of sequentially arriving individuals to optimize cumulative rewards while balancing exploration and exploitation, are critical for online decision-making. 
Existing approaches typically assume no interference, where each individual’s action affects only their own reward. Yet, such an assumption can be violated in many practical scenarios, and the oversight of interference can lead to short-sighted policies that focus solely on maximizing the immediate outcomes for individuals, which further results in suboptimal decisions and potentially increased regret over time. To address this significant gap, we introduce the \underline{f}o\underline{r}esighted \underline{o}nline policy with i\underline{nt}erference (FRONT) that innovatively considers the long-term impact of the current decision on subsequent decisions and rewards. 
The proposed FRONT method employs a sequence of exploratory and exploitative strategies to manage the intricacies of interference, ensuring robust parameter inference and regret minimization. 
Theoretically, we establish a tail bound for the online estimator and derive the asymptotic distribution of the parameters of interest under suitable conditions on the interference network. We further show that FRONT attains sublinear regret under two distinct definitions, capturing both the immediate and consequential impacts of decisions, and we establish these results with and without statistical inference. The effectiveness of FRONT is further demonstrated through extensive simulations and a real-world application to urban hotel profits.
\end{abstract}


\section{Introduction}
\label{sec:intro}

In the online decision-making process, contextual bandit algorithms \citep{langford2007} aim to maximize cumulative rewards for sequentially arriving individuals by taking the optimal action based on their baseline features, while also carefully balancing the trade-off between exploitation and exploration. Contextual bandits have been widely applied in various fields, such as recommendation systems \citep{li2011, bouneffouf2012}, precision medicine \citep{tewari2017,durand2018,lu2021}, and dynamic pricing \citep{misra2019,tajik2024}. Most existing works \citep[see e.g., ][]{chambaz2017,chen2021,bibaut2021,zhan2021,dimakopoulou2021,zhang2021,khamaru2021,ramprasad2023,shen2024} typically assume that the mean outcome of interest is determined solely by the individual's current action and characteristics. Based on this primary assumption, a considerable number of methods have been developed by directly modeling the mean outcome function such as Upper Confidence Bound \citep{li2011} and Thompson Sampling \citep{agrawal2013thompson} (also see \citet{tewari2017} for a comprehensive survey). However, in reality, the current action may also influence the outcomes of subsequent individuals. This effect, known as \textit{interference} \citep{cox1958} in the causal inference literature, introduces demanding and yet-to-be-resolved challenges in online decision making.

\begin{figure}[!t]
    \centering
    \vspace{-0.5cm}
    \begin{subfigure}{\linewidth}
        \centering
        \includegraphics[width=0.96\linewidth]{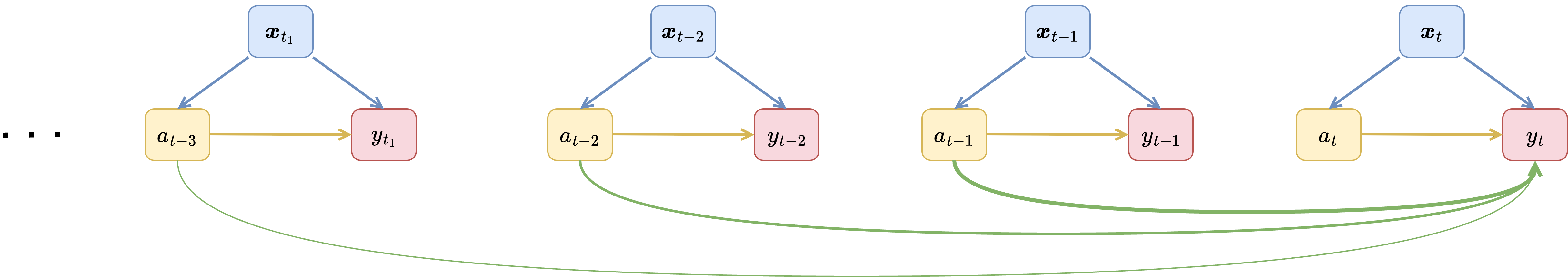}
        \caption{}
        \label{fig:causal1}
    \end{subfigure}

    \vspace{0.3cm}

    \begin{subfigure}{\linewidth}
        \centering
        \includegraphics[width=0.96\linewidth]{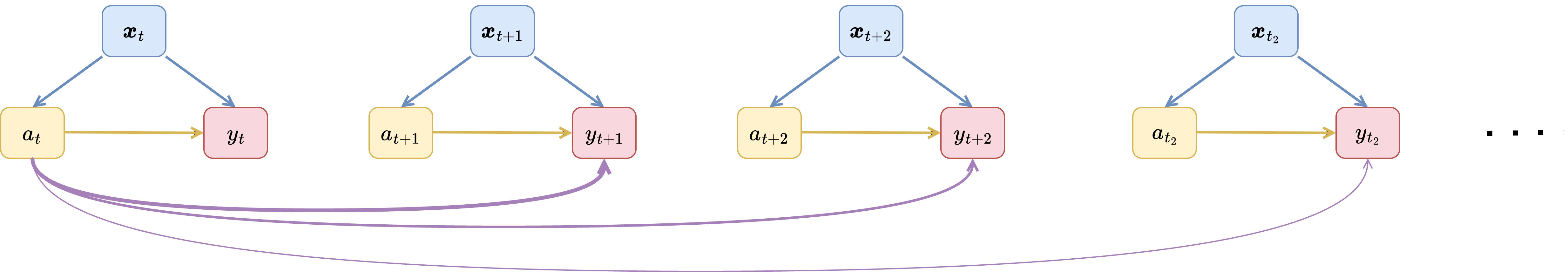}
        \caption{}
        \label{fig:causal2}
    \end{subfigure}

    \caption{The causal graphs illustrate \textit{interference} in online decision making. Individuals arrive sequentially with triples ${\bm{x}_t, a_t, y_t}$, where $\bm{x}_t$ are context features, $a_t$ is the chosen action, and $y_t$ the outcome. \textbf{Upper plot (a):} Green arrows show how past actions influence the outcome of the $t$-th individual.  \textbf{Lower plot (b):} Purple arrows show how the $t$-th individual affects subsequent outcomes. Line width reflects weight magnitude. }
    \label{fig:causal}
\end{figure}

Such an interference phenomenon among sequential individuals is common in many real-world applications. Using HPV vaccination as an example, social media platforms deliver personalized messages to boost vaccination rates \citep{hopfer2022,athey2023}. These messages create network effects: when one individual receives treatment (e.g., a vaccine advertisement), their subsequent online/offline interactions may influence others' vaccination decisions. Hotel pricing \citep[see e.g.,][]{cho2018} provides another example: while setting prices across multiple future dates, hotels must account for how current pricing trends (whether increasing or decreasing) impact future profitability. In these scenarios, the current individual's outcome depends further on past individuals' actions (see illustration in Figure~\ref{fig:causal}a), and the current individual's action would affect future outcomes (see illustration in Figure~\ref{fig:causal}b). 
The interference arising from the \textit{long-term impact of prior actions} has been recognized as a central challenge in sequential decision-making \citep[see e.g.,][]{agarwal2024,xu2024}. Traditional contextual bandit algorithms, which model the mean outcome solely as a function of the current action and context, overlook these dependencies. Consequently, they may yield \textit{short-sighted} policies that improve immediate rewards but ultimately lead to suboptimal performance or even greater regret when considering the population as a whole.

Despite extensive literature on causal inference and contextual bandits, research that directly addresses our specific focus, online sequential decision-making under interference, remains limited. Initial studies predominantly explore the \textit{offline} interference problem, where data is pre-collected, extensively reviewed by \cite{bajari2023}. These studies often consider relaxing the stable unit treatment value assumption (SUTVA), which is commonly imposed in causal inference (see \cite{rosenbaum2007} and \cite{forastiere2021} and the references therein). In these offline settings, connections and relationships between individuals can be \textit{fully observed} from existing databases, the ordering of individuals and their neighborhoods might be considered \textit{exchangeable}, assuming homogeneous interference effects \citep[see e.g.,][]{forastiere2021,qu2021,bargagli2025}. Consequently, techniques like partial interference and exposure mapping are extensively utilized to estimate causal effects in both randomized experiments \citep{sobel2006,hudgens2008,aronow2012,liu2014,aronow2017} and observational studies \citep{manski2013,forastiere2021,qu2021,lee2024,bargagli2025}. Here, partial interference assumes that interference occurs only within known networks. Exposure mapping formalizes how an individual’s potential outcome depends on a summary statistic of the treatments received by their neighbors, such as the count or proportion of treated neighbors. However, in \textit{online} settings, the information about \textit{subsequent individuals is unknown}, and the order of arrivals becomes crucial. As a result, all these aforementioned offline methods that require completely observed network of individuals cannot be applied to online scenarios with adaptively collected data. 

On the side of online decision making, most contextual bandit works treat each decision \textit{independently} with no interference assumption, and focus on either minimizing regret \citep[see e.g.,][]{li2011,agarwal2024} or deriving parameter inference \citep[see e.g.,][]{chen2021,shen2024}. Recently, several studies considered more general reward models for \textit{batched bandits} where the action of one individual can affect the rewards of others in the batch by extending the multi-agent cooperative game. Pioneering works \citep{bargiacchi2018,verstraeten2020,dubey2020} developed algorithms that help multiple individuals in the batch make coordinated decisions at the same time step. \cite{jia2024} investigated batched multi-armed bandits with known grid-structured interference patterns, employing cluster-level homogeneous action assignments. \cite{agarwal2024} generalized this framework to sparse network interference settings. Meanwhile, \cite{xu2024} considered within-cluster heterogeneous actions in batched contextual bandits. Unlike existing approaches that focus on batched bandits with group-level interference handling at the same time, we consider a distinct framework: we explicitly model interference over sequential decision points and derive individualized policies optimized for long-term outcomes. 

Our \underline{contributions} are as follows. {First}, to the best of our knowledge, we are the first to propose an optimal \textit{foresight} policy for online decision-making under interference, which we name \textbf{\underline{f}o\underline{r}esighted \underline{o}nline policy with i\underline{nt}erference (FRONT)}. As illustrated in Figure~\ref{fig:causal}, our work addresses interference over time, where the actions of earlier individuals affect subsequent ones. We propose the online additive outcome model with heterogeneous treatment effects and homogeneous interference effects. Simply including all neighbors' actions as features would lead to high-dimensional models even with an infinite dimension. To overcome this, we extend the exposure mapping function to capture interference in the online setting. Specifically, we define a very general weighted-average exposure mapping for the interference effect, allowing a broad applicability in real-world contexts \citep{zhao2015,hopfer2022}. Our proposed method, FRONT, explicitly incorporates the long-term effects of current actions into each decision-making step, utilizing the critical role of foresight in achieving optimal performance.

{Second}, we develop an online estimator together with a parameter inference procedure to quantify uncertainty under the proposed outcome model.
Our framework \textbf{simultaneously accounts for two distinct dependence structures}: \textit{the sequential dependence} arising from policy updates and adaptive data collected online, and \textit{the spatial dependence} induced by the network interference from growing neighboring individuals' actions. To break these two layers of dependence and ensure \textbf{valid inference}, we develop a two-level mechanism to carry out the exploration: implement the $\epsilon$-Greedy strategy with appropriate rate to introduce random actions and set a force-pull trigger to vary the interference effect. Leveraging this mechanism, we establish the tail bound and consistency for the online estimator. Remarkably, even under interference and limited sample settings, we further derive conditions on the weight network for interference that guarantee the asymptotic normality of the parameter estimates. We further illustrate these conditions with practical examples of weight networks.

{Third}, we consider two definitions of regret for our proposed policy on cumulative rewards and prove that \textit{both regrets are sublinear}. Regret is traditionally defined as the cumulative difference generated by decisions of the optimal policy and the proposed policy up to a certain horizon \citep{chen2021}. With interference, each decision may also contribute to regret in future time steps. We  \textbf{establish regret based on \textit{two} perspectives}: one focusing on the outcomes observed so far, the other accounting for the total regret caused by actions taken, including those that may manifest in the future. We further examine two cases: one with valid statistical inference, where parameter estimates achieve the $T^{-1/2}$ convergence rate (i.e., $\sqrt{T}$-consistency), and one without imposing additional conditions for such inference. Across both interference weight structures, we show that regret remains \textit{sublinear under both definitions}, highlighting the robustness of our method in long-term decision-making with interference.

The rest of this paper is organized as follows. In Section~\ref{sec:setup}, we formulate the problem setup, followed by the intuitive policies of the oracle, and propose our working model. In Section~\ref{sec:FRONT}, we present the FRONT method, which maximizes long-term cumulative rewards, along with its detailed implementation. Section~\ref{sec:theory} presents the assumptions, main theoretical results and statistical inference for the online estimators, followed by the regret bound of the proposed FRONT method. Numerical simulations and real data applications related to a hotel chain are provided in Section~\ref{sec:simulation} and Section~\ref{sec:real}, respectively, to validate our inference and demonstrate the strong performance of FRONT. In Section~\ref{sec:dis}, we present practical strategies for other working models, summarize this work, and discuss directions for future research. Proofs of all the main results and additional simulations are included in Appendix.

\section{Problem Formulation}
\label{sec:setup}
In this section, we first present the framework for the contextual bandit problem with interference in Section~\ref{subsec:framework}. Then, in Section~\ref{subsec:naive}, we describe the intuitive policies of the oracle and propose our working model.  

\subsection{Framework}
\label{subsec:framework} 

Suppose a sequence of individuals arrives in an order. Let $T$ denote the total number of individuals that can go to infinity. At each time $t$, where $t \in [T] = \left\{1, 2, \dots, T\right\}$, a new individual arrives, and we observe their contextual covariates $\bm{x}_t \in \mathbb{R}^{d_0}$. Here, $\bm{x}_t$ includes $1$ for the intercept, and $\bm{x}_t \stackrel{\text{i.i.d}}{\sim} \mathcal{P}_X$ for some unknown ${d_0}$-dimensional distribution. We consider the action set individuals can take as $\left\{0,1\right\}$. After selecting an action $a_t \in \left\{0,1\right\}$, we observe the reward $y_t$. Denote all the observed information up to time $t$ as $\mathcal{H}_{t} = \left\{\bm{x}_1, a_1, y_1, \bm{x}_2, a_2, y_2, \cdots, \bm{x}_t, a_t, y_t\right\}.$  In the online setting, the $t$-th individual may be influenced by those who arrived earlier, as they were observed prior to the decision-making time $t$. 
We thus define the neighborhood of the $t$-th individual as a subset of the individuals from the first up to the $(t-1)$-th. 

Including all the actions of prior individuals or neighbors as features leads to high and even infinite dimensionality in the online setting when $T$ goes to infinity. To address this issue, we consider that the effects are mediated through an exposure mapping function, consistent with established settings in the existing literature \citep{van2014,aronow2017,forastiere2021}. A simple example is the proportion of neighbors taking action $i$, i.e., $\sum_{s=1}^{t-1}\mathbb{I}(a_s=i)/(t-1)$, where $\mathbb{I}(\cdot)$ is the indicator function. More generally, to allow arbitrary influences from past actions, we define the \textit{interference action} as
\begin{equation}
\label{equ:kappadef}
     \kappa_t = \sum_{s=1}^{t-1} w_{ts} a_s,
\end{equation}
where $\{w_{ts}\}_{s=1}^{t-1}$ are weights that capture the influence of previous neighbors on the $t$-th individual. In particular, $w_{ts}=0$ whenever the $s$-th individual exerts no influence on the $t$-th individual. For instance, if we consider interference as the average action of the most recent $N$ individuals, then $w_{ts} = 1/N$ for $t-N \le s \le t-1$, and $w_{ts}=0$ otherwise. Moreover, following the existing literature \citep[see e.g.,][]{forastiere2021,xu2024}, we assume that the weight array $\{w_{ts}\}_{t=1,\,s<t}^{\infty}$ is known. Examples of weight arrays satisfying the theoretical conditions for valid statistical inference are given in Corollary~\ref{theor:gt}.

The interference effect from previous subjects operates through $\kappa_t$, which influences the current outcome $y_t$. Accordingly, we model $y_t$ as a function of $\bm{x}_t$, $\kappa_t$, and $a_t$ as 
\begin{equation*}
y_t \equiv \mu(\bm{x}_t,  \kappa_t, a_t)+e_t,
\end{equation*}
where $\mu(\bm{x}_t,  \kappa_t, a_t) \equiv \mathbb{E}(y_t|\bm{x}_t,  \kappa_t, a_t)$ is the conditional mean outcome, also known as the Q-function \citep{murphy2003,sutton2018}. The term $e_t$ represents $\sigma$-subgaussian noise, and conditioned on $a_t$ is independent of contextual covariates $\bm{x}_t$, all previous information $\mathcal{H}_{t-1}$, and thus the interference action $ \kappa_t$. We further specify the conditional variance as $\mathbb{E}(e_t \mid a_t = i) = \sigma_i^2 < \sigma^2$ for $i \in {0,1}$. Let $\mathcal{P}_{e_i}$ denote the conditional distribution of $e_t$ given $a_t = i$. It is worth noting, to our knowledge, that there is \textit{no} existing work that models $y_t$ using $ \kappa_t$ over time $t$.  

We aim to identify the optimal sequence of actions that maximizes the cumulative expected outcome
\begin{equation}
\label{obj_main}
\arg\max_{a_1,a_2,\dots} \sum_{t=1}^\infty \mu(\bm{x}_t, \kappa_t, a_t),
\end{equation}
yielding the optimal actions $\left\{a_t^\ast\right\}_{t=1}^\infty$, where the time horizon is set to be infinity. 
Equation~\eqref{obj_main} can be solved by backward inductive reasoning \citep[see e.g.,][]{chakraborty2014} if a predetermined endpoint $T$ is given. Yet, the induced optimal policy is \textit{intractable}, since it not only relies on the full specification of the mean outcome model but also may require knowledge of future contextual information, which is impractical in online settings. We provide a more detailed discussion in Section~\ref{sec:dis}. In the following section, we formalize the intuitive existing approaches within our online learning framework, and derive the optimal policy that maximizes the long-term cumulative benefits in Equation~\eqref{obj_main} under the proposed working model.

\subsection{Intuitive Policies and Working Model}
\label{subsec:naive}

We begin by examining two intuitive policies that could be applied to the online decision-making problem. In classical contextual bandit frameworks, interference effects are typically neglected \citep{chen2021,shen2024}, resulting in a restrictive modeling of the mean outcome as $f(\bm{x}_t, a_t) := \mathbb{E}(y_t |\bm{x}_t, a_t)$ only depending on the observed feature and the action taken at the current time step. In this case, we consider a policy from a suboptimal policy class that evaluates the expected outcomes of available actions without accounting for their potential impact on subsequent individuals. At each time step $t$, we estimate the parameters using all available historical data $\mathcal{H}_{t-1}$ \citep{deshpande2018,zhang2020,chen2021}.
We denote this approach as the \emph{na\"ive} policy, formally expressed as $\mathbb{I}\{\tilde{f}_{t-1}(\bm{x}_t, 1)-\tilde{f}_{t-1}(\bm{x}_t, 0)\ge 0 \}$, where $\tilde{f}_{t-1}(\bm{x}_t, a_t)$ is the estimated Q-function with plugged-in parameter estimates. 

Next, we consider an approach that accounts for interference effects in the mean outcome model, while still restricting attention to the suboptimal policy class and overlooking the interference effect for the subsequent arriving individuals. Such a policy focuses on maximizing only the immediate outcome at time $t$, yielding what we name the \emph{myopic} policy. The policy is expressed as $\mathbb{I}\left\{\widehat{\mu}_{t-1}(\bm{x}_t, \kappa_t,1) - \widehat{\mu}_{t-1}(\bm{x}_t, \kappa_t,0) \ge 0\right\}$, where $\widehat{\mu}_{t-1}(\bm{x}_t, \kappa_t,a_t)$ denotes the conditional mean outcome estimated correctly using historical information $\mathcal{H}_{t-1}$.

In settings with interference, the action $a_t$ at time $t$ influences not only the immediate outcome of the $t$-th individual but also the outcomes of all subsequent individuals whose interference weight from $t$ is nonzero. Thus, deriving the optimal long-term policy requires accounting for the cumulative effects of $a_t$ on future outcomes. However, obtaining a general closed-form solution is infeasible for two reasons. First, it would require backward induction that heavily depends on the specific structure of the conditional mean outcome model under interference with a predetermined endpoint. Second, more importantly, the solution may also require the subsequent contextual information to be obtained in advance, which is impossible in the online setting. 

To address these challenges, we propose a new working model for the online decision making with interference, which is motivated from real applications, enables a closed-form optimal policy, and advances the existing theoretical analysis of bandits. This makes the approach both analytically tractable and practically implementable in online settings. Specifically, following the recent literature on online inference \citep{deshpande2018,zhang2020,chen2021,shen2024,xu2024}, we begin with an online additive outcome model for the conditional mean outcome, and further extend it to incorporate interference effects. To capture potential nonlinear relationships among contextual variables, we apply a basis transformation to the raw context vector \( \bm{x} = (x_1, x_2, \ldots, x_{d_0}) \in \mathbb{R}^{d_0} \), and accordingly define a feature map $\phi: \mathbb{R}^{d_0} \rightarrow \mathbb{R}^{d_1}$ \citep{shi2020breaking,papini2021leveraging,dubey2022scalable}. 
Our proposed model is directly motivated by the homogeneous interference patterns observed in real-world applications. For instance, in the HPV vaccination example in Section~\ref{sec:intro}, the effect of receiving vaccine-related advertisements is independent of individual characteristics such as sex or race. Similarly, in the hotel pricing example, historical price adjustments exert homogeneous effects on future profitability, regardless of seasonal variations, occupancy rates, etc.
Motivated by these settings, we then propose \textbf{the online additive outcome model with heterogeneous treatment effects and homogeneous interference effects}, formally specified as follows:
\begin{equation}
\label{equ:conditionalmean}
   \mu(\bm{x}_t,  \kappa_t, a_t) = \mathbb{E}(y_t \vert \bm{x}_t,  \kappa_t, a_t) = (1-a
   _t)\phi(\bm{x}_t)^{\top}\bm{\beta}_0+a_t\phi(\bm{x}_t)^{\top}\bm{\beta}_1+ \kappa_t \gamma,
\end{equation}
where model parameters $\bm{\beta}_0$, $\bm{\beta}_1 \in \mathbb{R}^{d_1}$ are treatment effect parameters associated with choosing action 0 or 1, and $\gamma \in \mathbb{R}$ measures the homogeneous interference effect arising from neighbors' actions $\kappa_t = \sum_{s=1}^{t-1} w_{ts} a_s$. Discussions on how to extend our working model in Equation~\eqref{equ:conditionalmean} to more general cases can be found in Section~\ref{sec:other}.

\section{The Proposed Method: FRONT}
\label{sec:FRONT}

In this section, we derive the optimal policy and formally present the FRONT method, which maximizes long-term cumulative rewards. We further detail its implementation by extending the $\epsilon$-Greedy strategy, which employs a sequence of exploratory and exploitative strategies to manage the intricacies of interference, ensuring robust parameter inference and regret minimization. 

To begin with, we establish the optimal policy with interference under the proposed working model. Specifically, under the online additive outcome model with heterogeneous treatment effects and homogeneous interference effects in Equation~\eqref{equ:conditionalmean}, the optimal policy that accounts for interference corresponds to the solution of Equation~\eqref{obj_main} that maximizes the cumulative expected outcome, as stated in the following proposition.

\begin{proposition}[{\emph{Optimal policy with interference}}]
\label{prop:optimal}
     Under the conditional mean outcome model in Equation~\eqref{equ:conditionalmean}, the optimal policy $\pi^\ast(\cdot)$ at time $t$ is given by 
    \begin{equation}
    \label{equ:optimal_pi}
        a_t^\ast := \pi^\ast(\bm{x}_t) = \mathbb{I}\left\{\phi(\bm{x}_t)^{\top} (\bm{\beta}_1 - \bm{\beta}_0) + \left[\sum_{s=t+1}^{\infty} w_{st}\right] \gamma \ge 0 \right\}.
    \end{equation}
\end{proposition}

We denote $\zeta_t = \sum_{s=t+1}^{\infty} w_{st}$, and define the term $\zeta_t \gamma = \left[\sum_{s=t+1}^{\infty} w_{st}\right] \gamma$ in Equation~\eqref{equ:optimal_pi} as the \textbf{interference effect on subsequent outcome (ISO)} for the $t$-th individual. 
ISO represents the \textit{spillover} component of the $t$-th individual’s action that propagates to future outcomes, excluding its direct contribution to the current reward. As shown in Proposition~\ref{prop:optimal}, the optimal action at time $t$ depends on ISO. Unlike the myopic policy and classic contextual bandits, Equation~\eqref{equ:optimal_pi} explicitly incorporates ISO as an additional term in the decision rule, indicating the interference effect $\zeta_t\gamma$. Consequently, Equation~\eqref{equ:optimal_pi} can be interpreted as the difference in the long-term cumulative reward between choosing action 0 and action 1 at time $t$. ISO plays a key role in distinguishing two definitions of regret and in deriving the corresponding regret bounds presented in Section~\ref{sec:regret}. The proof of Proposition~\ref{prop:optimal} can be found in Appendix~\ref{app:optimal}.

We next detail the estimation of the optimal policy with interference based on Proposition~\ref{prop:optimal}. In online learning, we fit the proposed model in Equation~\eqref{equ:conditionalmean} and adopt the derived optimal policy with plug-in estimates to make decisions, balancing the trade-off between exploration and exploitation. Denote $\bm{z}_t = ((1-a_t)\phi(\bm{x}_t)^{\top},  a_t\phi(\bm{x}_t)^{\top}, \kappa_t)^{\top} \in \mathbb{R}^d$ and $\bm{\theta} = (\bm{\beta}_0^{\top}, \bm{\beta}_1^{\top}, \gamma)^{\top} \in \mathbb{R}^d$, where $d = 2d_1+1$. Then Equation~\eqref{equ:conditionalmean} can be simplified to $\mu(\bm{z}_t,a_t)=\bm{z}_t^{\top}\bm{\theta}$. We assume an initial warm-up phase with $T_0$ samples, during which individuals are assigned specific actions to ensure efficient exploration at the outset. At each time $t=T_0+1,T_0+2,\dots$, we estimate the parameters $\bm{\theta}$ based on $\mathcal{H}_{t}$ by
\begin{equation}
\label{equ:para}
    \widehat{\bm{\theta}}_{t} =(\widehat{\bm{\beta}}_{0,t}^{\top}, \widehat{\bm{\beta}}_{1,t}^{\top}, \widehat{\gamma}_{t})^{\top}= \left(\frac{1}{t} \sum_{s=1}^{t} \bm{z}_s \bm{z}_s^{\top}\right)^{-1} \left( \frac{1}{t} \sum_{s=1}^t \bm{z}_s y_s \right), 
\end{equation}
if the first term ${t}^{-1} \sum_{s=1}^t \bm{z}_s \bm{z}_s^{\top}$ is invertible. In this paper, we extend the $\epsilon$-Greedy method \citep{chambaz2017,chen2021} with interference to avoid getting trapped in a single action and being led in the wrong direction. To be specific, at each step after the warm-up samples, we first update the online estimators and then apply the $\epsilon$-Greedy to make random choices with probability $\epsilon_t$. The action $a_t$ is then generated by $\text{Bernoulli}(\widehat{\pi}_t)$, where
\begin{equation}
\label{equ:pit}
\widehat{\pi}_t (\bm{x}_t) = (1-\epsilon_t) \mathbb{I}\left\{\phi(\bm{x}_t)^{\top} (\widehat{\bm{\beta}}_{1,t-1} - \widehat{\bm{\beta}}_{0,t-1}) + \zeta_t \widehat{\gamma}_{t-1} \geq 0 \right\} + \frac{\epsilon_t}{2}.
\end{equation}
We name our method as \textbf{\underline{f}o\underline{r}esighted \underline{o}nline policy with i\underline{nt}erference (FRONT)}, with the detailed pseudocode provided in Algorithm~\ref{alg}. 

To enhance exploration efficiency and enable online statistical inference, particularly for the interference-related parameter $\gamma$, we further design the force pulling step. Previous methods apply force pulls at fixed time points \citep[see e.g.,][]{bastani2020,hu2022dtr}, where a single force pull often does not substantially perturb $ \kappa_t$. Instead, we conduct consecutive pulls whenever the condition in Step (7) in Algorithm~\ref{alg} is triggered to meet the clipping assumption (i.e., Assumption~\ref{ass:mineigen}, which serves as one theoretical ground for online statistical inference). The clipping parameter $C$, which appears in Assumption~\ref{ass:mineigen}, will be formally introduced in Section~\ref{sec:consistency}. Specifically, in Step (7), we perform $K$ consecutive force pulls from $a_t$ to $a_{t+K}$. If $ \kappa_{t-1} \leq \kappa_0$, we set the force pulls action to 1; otherwise, we set them to be 0, in order to increase the variability of $ \kappa_t$. The threshold $\kappa_0$ is chosen based on the weight array $\left\{w_{ts}\right\}_{t=1,s<t}^\infty$. We recommend choosing $\kappa_0$ such that force pulls of both actions 0 and 1 induce noticeable fluctuations in $\kappa_t$. 
Furthermore, during the warm-up period in Step (2), we increase the variability of \(  \kappa_t \) to ensure a robust initial estimator. We divide the warm-up duration \( T_0 \) into \( 2L \) intervals, alternating actions: even-numbered intervals take action 0, while odd-numbered intervals take action 1. This periodic assignment expands the variability of observed values for \(  \kappa_t \), reducing the need for future exploration. 

By choosing the suitable \( C \), \( \kappa_0 \), \( K \), \( T_0 \), \( L \), and \( \epsilon_t \), we prevent frequent triggers and reduce consecutive force pulls. In practice, to ensure sufficient exploration and maintain high variability of $\kappa_t$ during the warm-up period, we recommend using a relatively large $T_0$ combined with a suitable $L$, taking into account the structure of the weight array $\left\{w_{ts}\right\}_{t=1,s<t}^\infty$. A small value of $C$ is also preferred to avoid frequent triggers.
Given a known $\left\{w_{ts}\right\}_{t=1,s<t}^\infty$, the feasible range of $\kappa_t$ can be derived, and we suggest setting $\kappa_0$ to the midpoint of this range to balance the force pulls across both actions. The number of force pulls per trigger, $K$, should also be determined based on $\left\{w_{ts}\right\}_{t=1,s<t}^\infty$ as well to ensure adequate variability in $\kappa_t$. Notably, a small $K$ may lead to limited variation in $\kappa_t$, reducing exploration efficiency, whereas a large $K$ may introduce redundant forced pulls and increase cumulative regret. 
The theoretical order of the exploration probability $\epsilon_t$ that $t\epsilon_t^2 \to \infty$ is formally established in Section~\ref{sec:theory} with practical recommendations. 
The convergence of $\kappa_t$ under different weight arrays is provided in Section \ref{sec:infer}. We further detail theoretical recommendation on the force-pull set based on \( K \), \( T_0 \), and \( L \) in Section \ref{sec:infer_regret},
with sensitivity analyses of varying hyperparameters provided in Appendix~\ref{app:simulation}.

\begin{algorithm}[!t]
\caption{FRONT under $\epsilon$-Greedy}
\label{alg}
\begin{algorithmic}[0]
\setstretch{1.2}
\State \textbf{Input:} 
clipping parameter $C$, force pulls threshold $\kappa_0$, number of force pulls at one trigger $K$, warm-up period $T_0$, warm-up interval parameter $L$, exploration probability $\epsilon_t$;
\For{\(t = 1, 2, \ldots, T_0\)}
    \State (1) Sample $d_0$-dimensional context $\bm{x}_t \in \mathcal{P}_X$;
    \State (2) Set $a_t = 1$ when $t \in \bigcup_{i=1}^L ((i-1)T_0/L,\,(2i-1)T_0/{2L}]$, and $a_t = 0$ when $t \in \bigcup_{i=1}^L ((2i-1)T_0/{2L}, \, iT_0/L]$;
    \State (3) Update $ \kappa_t$ by Equation~\eqref{equ:kappadef};
\EndFor
\For{$t=T_0+1, T_0+2, \ldots$} 
\State (4) Sample $d_0$-dimensional context $\bm{x}_t \in \mathcal{P}_X$;
\State (5) Update $\widehat{\bm{\theta}}_{t-1}$ by Equation~\eqref{equ:para};
\State (6) Update $\widehat{\pi}_t(\cdot)$ by Equation~\eqref{equ:pit} and $a_t$ by following Bernoulli $(\widehat{\pi}_t)$;
\If{$\lambda_{\min}\left\{{t}^{-1} \sum_{s=1}^t \bm{z}_s^{\top} \bm{z}_s \right\} \le C\epsilon_t$}
\If{$ \kappa_{t}\le \kappa_0$}
\State (7a) Perform force pulls by setting $a_{t}=a_{t+1}=\ldots= a_{t+K}=1$;
\Else
\State (7b) Perform force pulls by setting $a_{t}=a_{t+1}=\ldots =a_{t+K}=0$;
\EndIf
\EndIf
\EndFor
\end{algorithmic}
\end{algorithm}

\section{Theory}
\label{sec:theory}

In this section, we establish the statistical properties of the proposed method. We first derive a tail bound for the online estimator and provide conditions on the exploration rate that guarantee consistency of the parameter estimates in Section~\ref{sec:consistency}. We then analyze the evolution of the interference action $\kappa_t$ in Section~\ref{sec:kappat} and establish asymptotic normality of the estimators in Section~\ref{sec:normal}. Finally, we present the regret rates of the proposed FRONT algorithm, both with and without additional conditions for statistical inference, in Section~\ref{sec:regret}. All proofs are collected in Appendix~\ref{app:thms}.

\subsection{Consistency of Estimated Parameters}
\label{sec:consistency}

We first show the consistency of the estimated parameters for our proposed online
additive outcome model with interference. Our theoretical analysis starts from the following assumptions, which enable us to establish the tail bound of the online estimator.

\begin{assumption}[{\emph{Boundness}}]
\label{ass:bound}
For all $\bm{x} \sim \mathcal{P}_X$, there exist positive constants $L_x$ and $\lambda$ such that $ \|\phi(\bm{x})\|_{\infty} \le L_x $ and $\lambda_{\min}\!\left(\mathbb{E}[\phi(\bm{x}) \phi(\bm{x})^{\top}]\right) > \lambda$. Moreover, there exists a positive constant $L_z \ge L_x$ such that $\sum_{s=1}^{t-1} \big|w_{ts}\big| \le L_z$ for all $t$, then $\|\bm{z}_t\|_{\infty} \le L_z$ for any $t$.
\end{assumption} 
\begin{assumption}[{\emph{Clipping}}]
\label{ass:mineigen}
For any time step $t \ge 1$, there exists a constant $C$ such that
\begin{equation*}
    \lambda_{\min}\left\{\frac{1}{t} \sum_{s=1}^t \bm{z}_s^{\top} \bm{z}_s \right\} > C\epsilon_t.
\end{equation*}
\end{assumption}

Assumption~\ref{ass:bound}, commonly used in contextual bandits \citep{zhang2020, chen2021, shen2024}, restricts the contextual covariates, ensuring that the mean of the martingale differences converges to zero. Assumption~\ref{ass:mineigen} is a technical condition necessary for the consistency and asymptotic normality of the least squares estimators \citep{deshpande2018,zhang2020,hadad2021,shen2024}. In practice, it can be achieved via the warm-up period and force-pull techniques in Step (2) and Step (7) of Algorithm~\ref{alg}. Let $C\epsilon_t$ as the clipping rate. We derive the consistency of the online estimator. 

\begin{theorem}[{\emph{Tail bound for the online estimator}}]
\label{theor:tail}
In the online decision-making using FRONT, suppose Assumptions~\ref{ass:bound} and \ref{ass:mineigen} are satisfied, for $\forall h > 0$, we have 
\begin{equation}
    \Pr\left\{\left\Vert  \widehat{\bm{\theta}}_{t} - \bm{\theta} \right\Vert_1 \le h \right\}\ge 1- 4d_1 \exp \left\{-\frac{t\epsilon_t^2 C^2 h^2}{2d^2\sigma^2 L_z^2}\right\} - 4\exp \left\{-\frac{t\epsilon_t^2 C^2 h^2}{8d^2\sigma^2 L_z^2}\right\}.
\end{equation}
\end{theorem}

Theorem~\ref{theor:tail} demonstrates that the consistency of the online estimator can be established if $t\epsilon_t^2 \to \infty$. The proofs can be found in Appendix~\ref{app:tailbound}. Corollary~\ref{coro:consistency} can be derived straightforwardly as follows.

\begin{corollary} [{\emph{Consistency of the online estimator}}]
\label{coro:consistency}
If conditions in Theorem~\ref{theor:tail} are satisfied and $t\epsilon_t^2 \to \infty$ as $t \to \infty$, then the online estimator is consistent, i.e., $\widehat{\bm{\theta}}_{t} \to \bm{\theta}$ as $t \to \infty$.
\end{corollary}
After ensuring that the exploration rate satisfies $t\epsilon_t^2 \to \infty$, we establish the consistency of the online estimator and obtain point estimates. We further explore the trend of the interference effect.

\subsection{Statistical Inference} 
\label{sec:infer}

In this section, we investigate the behavior of the interference action $\kappa_t$, which serves as the cornerstone of our statistical inference results and is essential for establishing the asymptotic normality of the online estimators. We first present sufficient conditions that guarantee the convergence of the optimal interference action $\kappa_t^\ast$ under the working model when the true parameter $\bm{\theta}$ is given (Theorem~\ref{theor:kappaast}). We then turn to the practical setting, establishing additional conditions that guarantee the convergence of the empirical interference action $\kappa_t$ under the estimated working model and our proposed method FRONT, thereby characterizing its limiting behavior as a combination of the optimal interference limit and the exploration rate limit (Theorem~\ref{theor:kappat}), with practical examples satisfying these conditions provided in Corollary~\ref{theor:gt}.
Building on these results, we further derive the asymptotic normality of the estimated parameters (Theorem~\ref{theor:infer}).

\subsubsection{Convergence of $\kappa_t$}
\label{sec:kappat}
To investigate the behavior of the interference action $\kappa_t$, we first establish the conditions for the convergence of the optimal interference action $\kappa_t^\ast$ induced by the optimal policy $\pi^\ast(\bm{x})$ as defined in Proposition~\ref{prop:optimal}, under the working model with true parameters as follows.

\begin{theorem}[{\emph{Convergence of the Optimal Interference Action}}]
\label{theor:kappaast}
Suppose that   \\$
\sum_{s=1}^{t-1} w_{ts}^2 \,\text{Var}(\pi^\ast(\bm{x}_s)) \;\to\; 0 
 \text{ as } t \to \infty, $
and  
$
\sum_{s=1}^{t-1} w_{ts}\,\mathbb{E}(\pi^\ast(\bm{x}_s)) \;\to\; \kappa_\infty^\ast < \infty$, where 
\[
\kappa_\infty^\ast
= \lim_{t \to \infty} \sum_{s=1}^{t-1} w_{ts}\,
\Pr\!\Bigg(
    \phi(\bm{x})^{\top} (\bm{\beta}_1-\bm{\beta}_0)
    + \Bigg[\sum_{k=s+1}^\infty w_{ks}\Bigg]\gamma
    \ge 0 \Bigg).
\]
Then
\[
\kappa_t^\ast \;\stackrel{p}{\longrightarrow}\; \kappa_\infty^\ast,
\]
where $\kappa_t^\ast = \sum_{s=1}^{t-1} w_{ts} a_s^\ast$ denotes the optimal interference action at time $t$.
\end{theorem}

Theorem~\ref{theor:kappaast} establishes that, in the optimal setting—where the ground truth parameter $\bm{\theta}$ is given—if the weighted expectation of past actions converges to a constant while the corresponding variance vanishes, then $\kappa_t^\ast$ converges in probability. Moreover, the theorem provides sufficient conditions on the weight array $\left\{w_{ts}\right\}_{t=1,,s < t}^\infty$ that guarantee such convergence. The detailed proof is presented in Appendix~\ref{app:ISO}.

Building on this theoretical result, we next derive the additional conditions under which $\kappa_t$ converges in the practical case, under our proposed FRONT method.

\begin{theorem}[{\emph{Convergence of $\kappa_t$}}]
\label{theor:kappat}
Suppose Assumptions~\ref{ass:bound} and~\ref{ass:mineigen} hold, together with the conditions in Corollary~\ref{coro:consistency} and Theorem~\ref{theor:kappaast}. 
Let $\mathcal{F}_t$ denote the set of force pulls executed up to time $t$. 
Further assume that, as $t \to \infty$: (i) $\epsilon_t \to \epsilon_\infty$; (ii) $\sum_{s=1}^{t-1} w_{ts}^2 \to 0$; (iii) $w_{ts} \to 0$ for each fixed $s$; and (iv) $\sum_{s \in \mathcal{F}_t}^{t-1} w_{ts}a_s \to 0$. Then
\[
\kappa_t \stackrel{p}{\longrightarrow} 
\kappa_\infty 
= \left(1-\epsilon_\infty\right)\kappa_\infty^\ast + \frac{\epsilon_\infty}{2}.
\]
If $\epsilon_\infty = 0$, condition (ii) can be replaced with the weaker requirement $\sum_{s=1}^{t-1} w_{ts}^2 \,\mathbb{E}(\pi^\ast(\bm{x_s}))\to 0$ as $t \to \infty$
\end{theorem}

Theorem~\ref{theor:kappat} provides sufficient conditions for the convergence of $\kappa_t$. These conditions can be interpreted as follows.  
(i) The exploration rate should converge to a constant, typically set to zero in practice to avoid excessive regret.  
(ii) For each specific individual, the sum of the squared interference weights assigned to previous individuals should converge to zero.  
(iii) For each specific individual, the weight assigned to subsequent individual should vanish as time increases.  
(iv) Force pulls are assumed not to affect the convergence of $\kappa_t$, which we provide empirical verification in simulations in Section~\ref{sec:simulation}.
See Appendix~\ref{app:kappa} for proofs.

Here we provide a specification of the weight array $\left\{w_{ts}\right\}_{t=1,,s<t}^\infty$ that meets conditions (ii) and (iii) of Theorem~\ref{theor:kappat}. In this setting, equal weights are assigned to the most recent $g(t)$ individuals, so that the $t$-th individual is influenced by the actions of individuals from $t-g(t)$ to $t-1$. The neighborhood size $g(t)$ increases with $t$ and satisfies $g(t)\to \infty$ as $t \to \infty$. Consequently, the weights take the form  
\[
w_{ts} = \frac{\mathbb{I}(t-g(t) \le s \le t-1)}{g(t)}.
\]
To illustrate, consider the hotel pricing example. Each new price offered to a customer depends on past prices. As time progresses, the hotel accumulates more pricing data from earlier customers, thereby expanding the neighborhood of past actions that influence subsequent decisions. The following theorem shows that both linear and square-root growth of the interference scale satisfy the assumptions imposed on the array $\{w_{ts}\}_{t=1,\,s<t}^\infty$ in Theorems~\ref{theor:kappaast} and~\ref{theor:kappat}. 

\begin{corollary}[{\emph{Growing interference scale with equal weights}}]
\label{theor:gt}
In the equal weights setting, if the interference scale satisfies
(i) $g(t) = \lfloor \rho t \rfloor$, or  
(ii) $g(t) = \lfloor \rho \sqrt{t} \rfloor$,  
then, under the same assumptions as in Theorem~\ref{theor:kappat}, all the conditions in Theorem~\ref{theor:kappaast} together with conditions (ii) and (iii) of Theorem~\ref{theor:kappat} are satisfied. If, in addition, conditions (i) and (iv) of Theorem~\ref{theor:kappat} also hold, then $\kappa_t$ converges.  
Here $0 < \rho < 1$, and $\lfloor \cdot \rfloor$ denotes the the greatest integer less than or equal to its argument.
\end{corollary}

Case (i) corresponds to a linearly growing neighborhood, and case (ii) corresponds to a square-root growth rate. Both scenarios demonstrate that the neighborhood can expand while assigning progressively smaller equal weights over time, guarantee the convergence of $\kappa_t$. The detailed proofs are provided in Appendix~\ref{app:gt}.

\subsubsection{Asymptotic Normality}
\label{sec:normal}
As established in Theorem~\ref{theor:kappat}, the interference action $\kappa_t$ converges to a constant $\kappa_{\infty}$. We now turn to the statistical inference of the estimated parameters online. 
Towards valid inference of the parameter estimators, the convergence of the conditional variance requires both the stability of the interference action $\kappa_t=\sum_{s=1}^{t-1} w_{ts} a_s$ and the convergence of the parameter of the interference effect on subsequent outcome $\zeta_t = \sum_{s=t+1}^{\infty} w_{st}$. The asymptotic normality of the estimators then follows from the Martingale Central Limit Theorem \citep{hall2014martingale}, with detailed proofs provided in Appendix~\ref{app:infer}. Furthermore, the convergence of $\zeta_t$ holds in both cases discussed in Corollary~\ref{theor:gt}, with the corresponding proofs given in Appendix~\ref{app:gt}.

\begin{theorem}[{\emph{Asymptotic normality for the online estimator}}]
\label{theor:infer} 
Suppose that the conditions in Theorem~\ref{theor:kappat} are satisfied, and assume that $\zeta_\infty = \lim_{t \to \infty}\zeta_t$ exists with $\kappa_{\infty} \neq 0$. Then
\[
\sqrt{t}\, (\widehat{\bm{\theta}}_{t}-\bm{\theta}) =\sqrt{t}
\begin{pmatrix}
        \widehat{\bm{\beta}}_{0,t}-\bm{\beta}_0 \\
        \widehat{\bm{\beta}}_{1,t}-\bm{\beta}_1 \\
        \widehat{\gamma}_t-\gamma 
    \end{pmatrix}
\;\stackrel{D}{\longrightarrow}\;
\mathcal{N}_d(\bm{0},\, S), 
\]
where
\[
S= \left(G_0+G_1\right)^{-1} \left(\sigma_0^2 G_0+\sigma_1^2G_1 \right)\left(G_0+G_1\right)^{-1},
\]
and
\[
G_0 \;=\; 
    \begin{pmatrix}
        \Sigma_{11}^0 & \bm{0} & \Sigma_{12}^0 \\
        \bm{0} & \bm{0} & \bm{0} \\
        \Sigma_{21}^0 & \bm{0} & \Sigma_{22}^0
    \end{pmatrix}, 
    \qquad
G_1 \;=\; 
    \begin{pmatrix}
        \bm{0} & \bm{0} & \bm{0} \\
        \bm{0} & \Sigma_{11}^1 & \Sigma_{12}^1 \\
        \bm{0} & \Sigma_{21}^1 & \Sigma_{22}^1
    \end{pmatrix}.
\]
The block matrices are given by
\begin{equation}
\label{equ:si}
\begin{aligned}
\Sigma_{11}^i &= \frac{\epsilon_{\infty}}{2} \int \phi(\bm{x})\phi(\bm{x})^{\top} \,d\mathcal{P}_X 
    + (1-\epsilon_{\infty}) \int_{\mathcal{X}_i} \phi(\bm{x})\phi(\bm{x})^{\top} \,d\mathcal{P}_X, \\
\Sigma_{12}^i &= (\Sigma_{21}^i)^\top = \kappa_{\infty} \left[ \frac{\epsilon_{\infty}}{2} \int \phi(\bm{x})^{\top}\,d\mathcal{P}_X
    + (1-\epsilon_{\infty}) \int_{\mathcal{X}_i} \phi(\bm{x})^{\top} \,d\mathcal{P}_X \right], \\
\Sigma_{22}^i &= \kappa_{\infty}^2 \left[ \frac{\epsilon_{\infty}}{2} 
    + \left(1-\frac{\epsilon_{\infty}}{2}\right) \Pr(\bm{x}\in\mathcal{X}_i) \right],
\end{aligned}
\end{equation}
with $\mathcal{X}_1 = \Big\{\bm{x} : \phi(\bm{x})^{\top}(\bm{\beta}_1-\bm{\beta}_0) + \zeta_\infty \gamma \ge 0 \Big\}$ and $\mathcal{X}_0 = \Big\{\bm{x} : \phi(\bm{x})^{\top}(\bm{\beta}_1-\bm{\beta}_0) + \zeta_\infty \gamma < 0 \Big\}$.
A consistent estimator for $S$ is given by 
\begin{equation}
\label{equ:siest}
    \left( \frac{1}{t} \sum_{s=1}^t \bm{z}_s \bm{z}_s^{\top} \right)^{-1} \left(\frac{1}{t}\sum_{s=1}^t \bm{z}_s \bm{z}_s^{\top} \widehat{e}_s^2 \right) \left( \frac{1}{t} \sum_{s=1}^t \bm{z}_s \bm{z}_s^{\top} \right)^{-1},
\end{equation}
where $\widehat{e}_s = y_s - \bm{z}_s^{\top} \widehat{\bm{\theta}}_{t}$.
\end{theorem}

As illustrated in Equation~\eqref{equ:si}, $G_0$ corresponds to the component associated with action~0, while $G_1$ corresponds to that associated with action~1. Moreover, the asymptotic variance of online estimator is related to $\kappa_\infty$, $\zeta_\infty$ and $\epsilon_\infty$. 
Theorem~\ref{theor:kappat} establishes that $\kappa_t$ converges under suitable conditions. However, if $\kappa_t$ fails to stabilize and continues to fluctuate, the conditions required for applying the Martingale Central Limit Theorem break down. In such cases, the dependence structure induced by interference becomes substantially more complex, invalidating the standard asymptotic arguments. Deriving valid inference in this setting would likely require new technical tools to handle the layered dependencies arising from both sequential decision-making and evolving interference. Developing such methods lies beyond the scope of this paper and remains an important direction for future research.

\subsection{Regret Analysis}
\label{sec:regret}

In this section, we discuss two definitions of cumulative regret, as illustrated in Figure~\ref{fig:regret}, both of which are reasonable and depend on different perspectives of understanding the online decision making under interference. We then derive the regret bounds corresponding to each definition.

\begin{figure}
    \centering
    \includegraphics[width=1\linewidth]{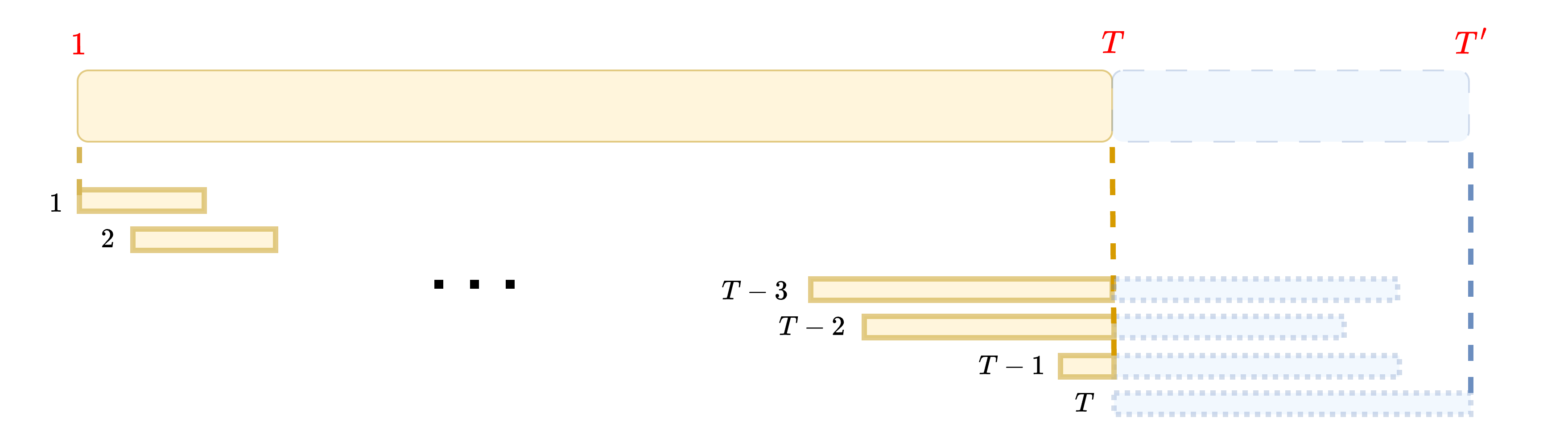}
    \caption{\textbf{ISO over time.} Each small rectangle represents the ISO generated by an individual’s action. Yellow rectangles indicate interference realized up to time $T$, while blue rectangles represent ongoing interference extending beyond $T$. By $T^\prime$, all interference caused before $T$ has manifested. Note that $T^\prime$ may be infinite.}
    \label{fig:regret}
\end{figure}

In online decision-making with interference, the regret caused by the $t$-th individual's action consists of two components: the regret from their own reward $y_t$ and the regret due to interference on future outcomes. ISO, defined as $\left[\sum_{s=t+1}^{\infty} w_{st}\right]\gamma$ in Proposition~\ref{prop:optimal}, represents the additional term in our optimal policy compared to classical contextual bandits that ignore interference. It captures the long-term influence of the current $t$-th action on subsequent individuals. Figure~\ref{fig:regret} illustrates ISO generated until $T$, with the last dotted section representing the regret on future rewards. Our two definitions of regret depend on the perspective—immediate or consequential—determined by whether the regret manifests by time $T$ or is caused before time $T$ but may not have appeared yet. 

We first define the cumulative regret $R_1(T)$ based on the first idea: whether the regret manifests by time $T$, as 
\begin{equation}
\label{equ:regret1}
    R_1(T) = \sum_{t=1}^T \mathbb{E} \left\{\mu(\bm{x}_t, \kappa_t^\ast, a_t^{\ast}) - \mu(\bm{x}_t,\kappa_t,a_t) \right\}.
\end{equation} 
This regret represents the difference in the conditional mean outcome between the optimal policy and FRONT, accumulated up to time $T$. As illustrated in Figure~\ref{fig:regret}, it corresponds to the yellow region. $R_1(T)$ can be expressed as 
\begin{align*}
    R_1(T) = \sum_{t=1}^{T} \mathbb{E} \left|\phi(\bm{x}_t)^{\top} (\bm{\beta}_1 - \bm{\beta}_0) + \left[\sum_{s=t+1}^T w_{st} \right] \gamma \right| \, \mathbb{I} \left\{ a_t \neq \mathbb{I} \left(\phi(\bm{x}_t)^{\top} (\bm{\beta}_1 - \bm{\beta}_0) + \zeta_t \gamma \ge 0\right) \right\}.
\end{align*}
The term $\sum_{s=t+1}^T w_{st}$ is truncated at time $T$, and bounded given the existence of $\zeta_\infty$. 

We then consider the cumulative regret $R_2(T)$ based on another idea—caused before time $T$, represented by the entire rectangle, including the dotted blue section in Figure~\ref{fig:regret}, as
\begin{equation}
\label{equ:regret2}
    R_2(T) = \sum_{t=1}^T \mathbb{E} \left\{\nu(\bm{x}_t, a_t^{\ast}) - \nu(\bm{x}_t,a_t) \right\},
\end{equation}
where \( \nu(\bm{x}_t, a_t) \) represents the total contribution to the long-term cumulative outcome from the \( t \)-th individual, i.e.,
$\nu(\bm{x}_t,a_t) = \bm{x}_t^{\top}\bm{\beta}_0 + a_t\bm{x}_t^{\top}(\bm{\beta}_1-\bm{\beta}_0)+ a_t \zeta_t \gamma.$ 
$R_2(T)$ is defined as the difference in the long-term outcome caused up to action $a_T$, between the optimal policy and FRONT. Then we can write $R_2(T)$ as
\begin{equation*}
    R_2(T) = \sum_{t=1}^{T} \mathbb{E} \left|\phi(\bm{x}_t)^{\top} (\bm{\beta}_1 - \bm{\beta}_0) + \zeta_t \gamma \right| \, \mathbb{I} \left\{ a_t \neq \mathbb{I} \left(\phi(\bm{x}_t)^{\top} (\bm{\beta}_1 - \bm{\beta}_0) + \zeta_t \gamma \ge 0 \right) \right\}.
\end{equation*}
 $R_1(T)$ and $R_2(T)$ can be analyzed in the same way as the regret in classical contextual bandits \citep{chen2021}, by decomposing it into two parts: exploration and estimation. The regret due to exploration can be directly bounded by $\mathcal{O}(\sum_{t=1}^T \epsilon_t)$. The estimation-induced regret can be derived based on the indicator function $\mathbb{I} \left\{ a_t \neq \mathbb{I} \left(\phi(\bm{x}_t)^{\top} (\bm{\beta}_1 - \bm{\beta}_0) + \zeta_t \gamma \ge 0\right) \right\}$ shared in both $R_1(T)$ and $R_2(T)$. To this end, we require a commonly assumed technical condition as follows.

\begin{assumption}[{\emph{Margin}}]
\label{ass:margin}
    For any $\bm{x} \sim \mathcal{P}_X$ and any $t \ge 1$, there exists a positive constant $M$, such that
\begin{equation*}
\Pr\left(0<\left|\phi(\bm{x})^{\top} (\bm{\beta}_1 - \bm{\beta}_0) + \zeta_t \gamma \right| < l\right) \le Ml, \quad \forall l>0.
\end{equation*}
\end{assumption}

Assumption~\ref{ass:margin} is a margin condition commonly proposed in the contextual bandit literature \citep{goldenshluger2013,chambaz2017,bastani2020,chen2021,shen2024,xu2024} to establish the regret bound. We make a minor modification to adapt it to our setting, as the boundaries between different policies are defined by $\phi(\bm{x})^{\top} (\bm{\beta}_1 - \bm{\beta}_0) + \zeta_t \gamma = 0$. Assumption~\ref{ass:margin} restricts the probability of encountering covariates close to these boundaries, thereby reducing the variance caused by incorrect decisions.  

Under Assumption~\ref{ass:margin} together with the convergence rate of estimated parameters, we establish bounds for $R_1(T)$ and $R_2(T)$ under two scenarios:  
(i) when the conditions of Theorem~\ref{theor:infer} hold, ensuring valid statistical inference and $\sqrt{T}$-consistency of the estimators (Section~\ref{sec:infer_regret}); and
(ii) when the inference is not available (Section~\ref{sec:noinfer_regret}).

\subsubsection{With Inference}
\label{sec:infer_regret}
Under the conditions of Theorem~\ref{theor:infer}—such as those satisfied by the two specifications of $g(t)$ in Corollary~\ref{theor:gt}—the parameter estimators are asymptotically normal and therefore $\sqrt{T}$-consistent. Building on this result, we establish the corresponding regret bound in the following theorem, with detailed proofs provided in Appendix~\ref{app:regret}.

\begin{theorem}[Regret bound]
\label{theor:regret}
Suppose the conditions in Theorem~\ref{theor:infer} and Assumption~\ref{ass:margin} are satisfied. When applying FRONT, the regrets defined by Equation~\eqref{equ:regret1} and Equation~\eqref{equ:regret2} can be bounded respectively as follows:
\begin{align*}
    R_1(T)  = \mathcal{O}_p \left(\sum_{t=1}^T \epsilon_t + T^{\frac{3}{4}}+|\mathcal{F}_t| \right), \quad 
    R_2(T)  = \mathcal{O}_p \left(\sum_{t=1}^T \epsilon_t + |\mathcal{F}_t| \right),
\end{align*}
where $|\mathcal{F}_T|$ denotes the cardinality of the force-pull set $\mathcal{F}_T$.
\end{theorem}

Under both definitions, the regret exhibits sublinear growth with an appropriate choice of $\epsilon_t$. To guarantee sufficient exploration, Corollary~\ref{coro:consistency} requires $t\epsilon_t^2 \to \infty$ as $t \to \infty$, which further implies $\mathcal{O}\,(\sum_{t=1}^T \epsilon_t ) \ge c \sqrt{t}$ for some positive constant $c$ \citep{chen2021}. In practice, we may choose $\epsilon_t = p \log(t)/\sqrt{t}$, $\;p \log \log(t)/\sqrt{t}$, or $\;p t^{-\alpha}$ with $0\le\alpha<\ 1/2$ to ensure $t\epsilon_t^2 \to \infty$.  
In the simulations reported in Section~\ref{sec:simulation}, we also observe that $|\mathcal{F}_t| = o_p(\sqrt{t})$ by appropriately tuning the parameters $T_0$, $K$, and $L$ in Algorithm~\ref{alg}.

Now suppose the exploration rate satisfies $\epsilon_t = \mathcal{O}(t^{-\alpha})$ with $0 \le \alpha < 1/2$. When $0 \le \alpha \le 1/4$, the term $\sum_{t=1}^T \epsilon_t$ dominates $T^{3/4}$, and thus both $R_1(T)$ and $R_2(T)$ share the same regret bound. When $1/4 < \alpha < 1/2$, the term $T^{3/4}$ dominates, so $R_1(T)$ admits a larger regret bound than $R_2(T)$.  

Both the na\"ive policy and myopic policy ignore ISO, potentially leading to suboptimal decisions at each time step $t$ and resulting in linear regret, as shown in the lower panel of Figure~\ref{fig:reward} (Section~\ref{sec:simulation}). The cumulative average regret $R_1(t)/t$ shown in this figure demonstrates that both policies exhibit regret converging to a non-zero constant, showing a $\mathcal{O}_p(T)$ regret rate.

\subsubsection{Without Inference}
\label{sec:noinfer_regret}
However, the conditions of Theorem~\ref{theor:infer} may not hold in practice, and thus the $\sqrt{T}$-consistency of the parameter estimators may fail. For example, if each individual $t$ is influenced only by its immediate predecessor, i.e., $\kappa_t = a_{t-1}$, then $\kappa_t$ coincides with the $(t-1)$-th action, which takes value $0$ or $1$ with probability given by Equation~\eqref{equ:pit}. In this case, $\kappa_t$ does not converge, so the existence of $\kappa_\infty$ fails, and the asymptotic normality of $\widehat{\bm{\theta}}_{t}$, $i=0,1$, cannot be established. Nevertheless, regret bounds can still be derived using the tail bound in Theorem~\ref{theor:tail}, as stated in the following theorem, with detailed proofs provided in Appendix~\ref{app:regret_N}.

\begin{theorem}[Regret bound without inference]
\label{theor:regret_N}
Suppose the conditions in Corollary~\ref{coro:consistency} and Assumption~\ref{ass:margin} are satisfied. When applying FRONT, the regrets defined by Equation~\eqref{equ:regret1} and Equation~\eqref{equ:regret2} can be bounded respectively as follows:
\begin{align*}
    R_1(T) = \mathcal{O}_p\left(\sum_{t=1}^T \epsilon_t + T^{\frac{3}{4}} \epsilon_T^{-\tfrac{1}{2}}+|\mathcal{F}_T|\right), \quad
    R_2(T) = \mathcal{O}_p\left(\sum_{t=1}^T \epsilon_t + T^{\tfrac{1}{4}} \epsilon_T^{-\tfrac{3}{2}} + |\mathcal{F}_T|\right).
\end{align*}
\end{theorem}

Relative to Theorem~\ref{theor:regret}, the extra term $T^{1/4} \epsilon_T^{-3/2}$ in $R_2(T)$ arises from the estimation error, which may not be dominated by the exploration regret $\sum_{t=1}^T \epsilon_t$ as in Theorem~\ref{theor:regret}, but instead depends on the exploration rate $\epsilon_t$. Moreover, in Theorem~\ref{theor:regret}, $R_1(T)$ includes an additional term of order $T^{3/4}$ compared with $R_2(T)$, whereas here the corresponding additional term becomes $T^{3/4} \epsilon_T^{-1/2}$. This difference arises because, without valid statistical inference, the online estimator in Theorem~\ref{theor:regret_N} fails to achieve $\sqrt{T}$-consistency, yielding only a slower convergence rate for the estimated parameters.

Suppose the exploration rate satisfies $\epsilon_t = \mathcal{O}(t^{-\alpha})$ with $0 \le \alpha < 1/2$. Then $\sum_{t=1}^T \epsilon_t = \mathcal{O}(T^{1-\alpha})$, while $T^{1/4} \epsilon_T^{-3/2} = \mathcal{O}(T^{1/4+3\alpha/2})$, revealing a trade-off between exploration and estimation error, governed by the choice of $\epsilon_t$.
When $0 \le \alpha \le 3/10$, the term $T^{1-\alpha}$ dominates $T^{1/4+3\alpha/2}$, so $R_2(T)$ is bounded by $\mathcal{O}\,(\sum_{t=1}^T \epsilon_t)$ with high probability.
When $3/10 < \alpha < 1/2$, the term $T^{1/4+3\alpha/2}$ dominates, and $R_2(T)$ is characterized by $\mathcal{O}_p \, (T^{1/3} \epsilon_T^{-3/2})$.  
Therefore, the optimal regret rate for $R_2(T)$ is $\mathcal{O}_p(T^{7/10})$, achieved when the exploration rate is set as $\epsilon_t = \mathcal{O}(T^{-3/10})$. Likewise, choosing $\epsilon_t = \mathcal{O}(T^{-1/6})$ yields an optimal regret rate of $\mathcal{O}_p(T^{5/6})$ for $R_1(T)$.

We now compare the optimal regret bound with results from the existing literature \citep[see, e.g.,][]{chen2021,xu2024}. Since our definition of $R_2(T)$ is consistent with the regret notion in classical contextual bandits, we focus on $R_2(T)$ here. The optimal regret bound is $\mathcal{O}_p(\sum_{t=1}^T \epsilon_t)$ under the condition $t \epsilon_t^2 \to \infty$, which reduces to $\widetilde{\mathcal{O}}_p(\sqrt{T})$ when the exploration rate is set as $\epsilon_t = p t^{-1/2} \log t$. 
In classical contextual bandits, regret arises from exploration and estimation error. The latter stabilizes quickly, yielding the well-known $\widetilde{\mathcal{O}}_p(\sqrt{T})$ regret.
In contrast, under interference, when the asymptotic normality of the estimated parameters cannot be established, the convergence rate depends on the clipping rate, resulting in slower decay of the estimation error.
Finally, we emphasize that the regret bounds derived here are upper bounds and therefore conservative, as they apply to the most general cases. Tighter bounds may be achievable if sharper convergence rates for the online estimators can be obtained, which we identify as an important direction for future research.

\section{Simulation Studies}
\label{sec:simulation}

In this section, we evaluate our proposed FRONT policy in comparison to the na\"ive policy and the myopic policy introduced in Section~\ref{subsec:naive}, and validate the statistical inference results for FRONT.

\subsection{Simulation Setup} 

The data are generated as follows. We consider two contextual features and set $d_0 = 3$ to include the intercept. The context is defined as $\bm{x} = (1, x_1, x_2)^{\top}$, where $x_1$ follows a truncated normal distribution with mean zero, variance one, and support $[-10, 10]$, and $x_2 \stackrel{\text{i.i.d.}}{\sim} \text{Uniform}(0, 2)$. 
We then apply a polynomial basis transformation to $\bm{x} \in \mathbb{R}^{d_0}$. 
Specifically, we construct a feature map $\phi(\bm{x})$ that includes all monomials up to total degree 2 \citep{dubey2022scalable}, given by
\[
\phi(\bm{x}) = \left(1, \, x_1, \, x_2, \, x_1^2, \, x_2^2, \, x_1 x_2 \right)^\top 
\in \mathbb{R}^{d_1}, \quad d_1 = 6.
\]
Then, $\bm{z}$ is thirteen-dimensional, including the interference action, which yields $d = 13$. 
The true parameters are set as 
\[
\bm{\theta} = (0.3,\,-0.1,\,0.3,\,0.5,\,-0.2,\,0.7,\,0.2,\,0.7,\,0.1,\,-0.3,\,0.5,\,0.3,\,0.6)^{\top}, 
\]
which yields
\[
\bm{\beta}_1 - \bm{\beta}_0 = (-0.1,\,0.8,\,-0.2,\,-0.8,\,0.7,\,-0.4)^{\top}, 
\quad \gamma = 0.6.
\]
Hence, Assumption~\ref{ass:bound} is satisfied. 
Additionally, the noise term is modeled as $e_t \mid a_t = i \stackrel{\text{i.i.d.}}{\sim} \mathcal{N}(0, 0.1^2)$ for $i \in \{0,1\}$.

For the weight array, we consider the special cases introduced in Corollary~\ref{theor:gt}:  
(1) $g(t) = \lfloor 0.2t \rfloor$,  
(2) $g(t) = \lfloor 5\sqrt{t} \rfloor$,  
and a larger-scale case  
(3) $g(t) = \lfloor 20t^{0.2} \rfloor$.  
Settings (1) and (2) have been verified to satisfy the conditions for the existence of $\zeta_\infty$ in Theorem~\ref{theor:infer}, as shown in Appendix~\ref{app:gt}, whereas case (3) will be verified numerically in Section~\ref{app:zetagplot}.  
For these three scenarios, Theorem~\ref{theor:infer} ensures the asymptotic normality of the online estimator, which we present in Section~\ref{sec:eva}, together with the policy effectiveness in Section~\ref{sec:policy}.

We also consider the case with a fixed-scale interference window under equal weights, which corresponds to one of the settings discussed in Section~\ref{sec:noinfer_regret}. Specifically, we take  
\[
\kappa_t = \frac{1}{N}\sum_{s=t-N}^{t-1} a_s, 
\]
with $N \in \{5,20,50\}$. 
We then demonstrate FRONT outperforms other methods under this setting, and also establish the consistency of the online estimator in Appendix~\ref{app:fix}.

\subsection{Policy Performance}
\label{sec:policy}

To ensure a fair comparison, all policies are implemented under identical configurations, including the same warm-up period, force pulls strategy, and $\epsilon$-Greedy exploration with identical parameters $\epsilon_t$. Specifically, the exploration rate is set as $\epsilon_t = {\log(t)}/{10\sqrt{t}}$ with a clipping parameter $C = 0.01$. The horizon is $T = 10{,}000$, with a warm-up period of $T_0 = 50$ and $L = 8$. We set $\kappa_0 = 0.5$, as we consider the average of the previous $g(t)$ individuals’ actions. The force-pull parameter is set to $K = 50$, and we run 500 replications following Algorithm~\ref{alg}. Under these settings, none of the three policies triggered force pulls in the growing interference scale case. The performance metric is the cumulative average reward, defined as the total reward accumulated up to a given time step divided by the number of steps.

For na\"ive policy, the mis-specification of the mean outcome model and the application of a short-sighted decision-making approach result in greater regret. As for the myopic policy, although it correctly specifies the model and estimates the parameters accurately, it tends to favor action 0 in many cases when $\phi(\bm{x})^\top (\bm{\beta}_1-\bm{\beta}_0)$ is negative. However, with $\gamma > 0$, choosing arm 1 introduces a positive influence on the cumulative reward due to ISO. Moreover, ISO is sufficiently large to alter the sign of the indicator term in FRONT. In this scenario, FRONT often makes different decisions, highlighting the critical role of incorporating ISO into decision-making. 

 \begin{figure}[!t]
    \centering
    \includegraphics[width=1.0\linewidth]{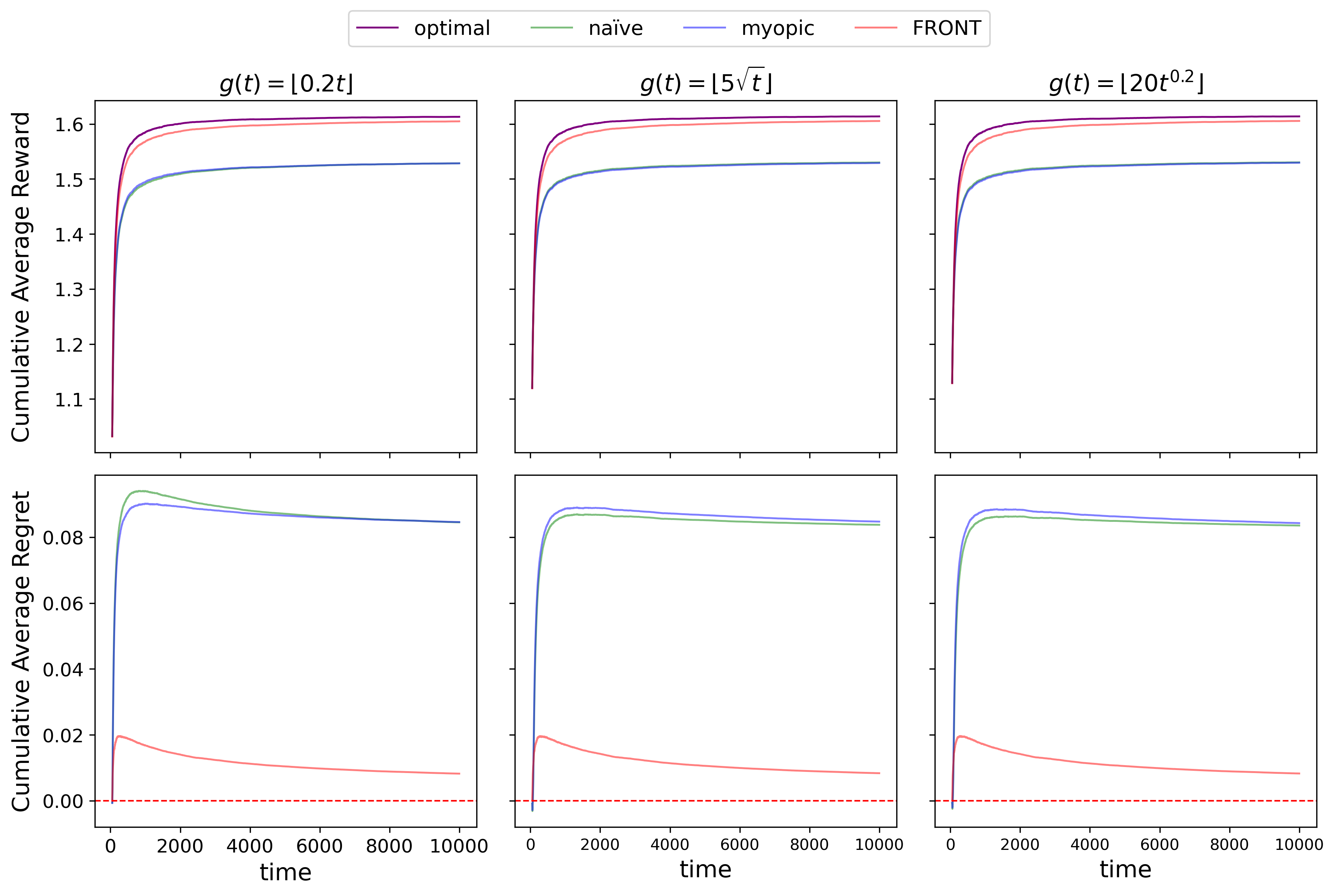}
    \caption{\textbf{Upper panel}: the cumulative average reward over time with growing interference scale $g(t)$. \textbf{Lower panel}: the cumulative average regret, quantifying the gap between the optimal policy and the other three policies (shown in the upper panel), with regret defined by $R_1(T)$ with Equation~\eqref{equ:regret1}.}
    \label{fig:reward}
\end{figure}

As shown in Figure~\ref{fig:reward}, which depicts the cumulative average reward trajectories under the growing interference scale setting, FRONT consistently achieves the lowest regret and steadily converges toward the optimal reward over time. The reward under FRONT converges to that of the optimal policy across all three scenarios of $g(t)$, demonstrating the universality of FRONT.  Additionally, we observe that the trajectories of the na\"ive policy and the myopic policy are close. This occurs because $\mathbb{I}\{\phi(\bm{x}_t)^{\top}(\tilde{\bm{\beta}}_{1,t-1} - \tilde{\bm{\beta}}_{0,t-1}) \ge 0\}$ and $\mathbb{I}\{\phi(\bm{x}_t)^{\top}(\widehat{\bm{\beta}}_{1,t-1} - \widehat{\bm{\beta}}_{0,t-1}) \ge 0\}$ lead to the same action although $\tilde{\bm{\beta}}_{0,t-1}$, $\tilde{\bm{\beta}}_{1,t-1}$ and $\widehat{\bm{\beta}}_{1,t-1}$, $\widehat{\bm{\beta}}_{1,t-1}$ are different estimations.

\subsection{Evaluation of Statistical Properties}
\label{sec:eva}

Next, we conduct experiments to validate the statistical inference results for FRONT. Under the same setting, the reported results for FRONT include the following metrics: the ratio of the average standard error (SE) to the Monte Carlo standard deviation (MCSD) across 500 replications, the average parameters estimation bias computed 500 experiments, and the coverage probability of the 95\% two-sided Wald confidence interval. The confidence intervals are constructed using the parameter estimates and their corresponding estimated standard errors from each experiment, based on Equation~\eqref{equ:siest} in Theorem~\ref{theor:infer}. The results are summarized in Figures~\ref{fig:infer1}, \ref{fig:infer2}, and \ref{fig:infer3}, for three scenarios: (1) $g(t) = \lfloor 0.2t \rfloor$; (2) $g(t) = \lfloor 5\sqrt{t} \rfloor$; (3) $g(t) = \lfloor 20t^{0.2} \rfloor$, respectively.

Figures \ref{fig:infer1}, \ref{fig:infer2}, and \ref{fig:infer3} show that FRONT performs well in all three cases: the ratio between SE and MCSD converges to 1, the average bias approaches 0, and the coverage probabilities for \(\bm{\theta}_0\) and \(\bm{\theta}_1\) are close to the nominal level of 95\%. Although we use force pulls to introduce variability in \(\kappa_t\), it still exhibits less variation compared to other covariates, which impacts the accuracy of the estimation. Moreover, a smaller interference scale $g_t$ yields better performance in parameter inference because greater variation in \(\kappa_t\) increases the variability of samples. Details on the behavior of $\kappa_t$ and the sensitivity analysis are given in Appendix~\ref{app:kappaplot} and Appendix~\ref{app:sensi}, where we present that $\kappa_t$ converges to a constant and that statistical inference holds under all examined settings.

\begin{figure}[!thp]
    \centering
    \vspace{-0.5cm}
    \includegraphics[width=1\linewidth]{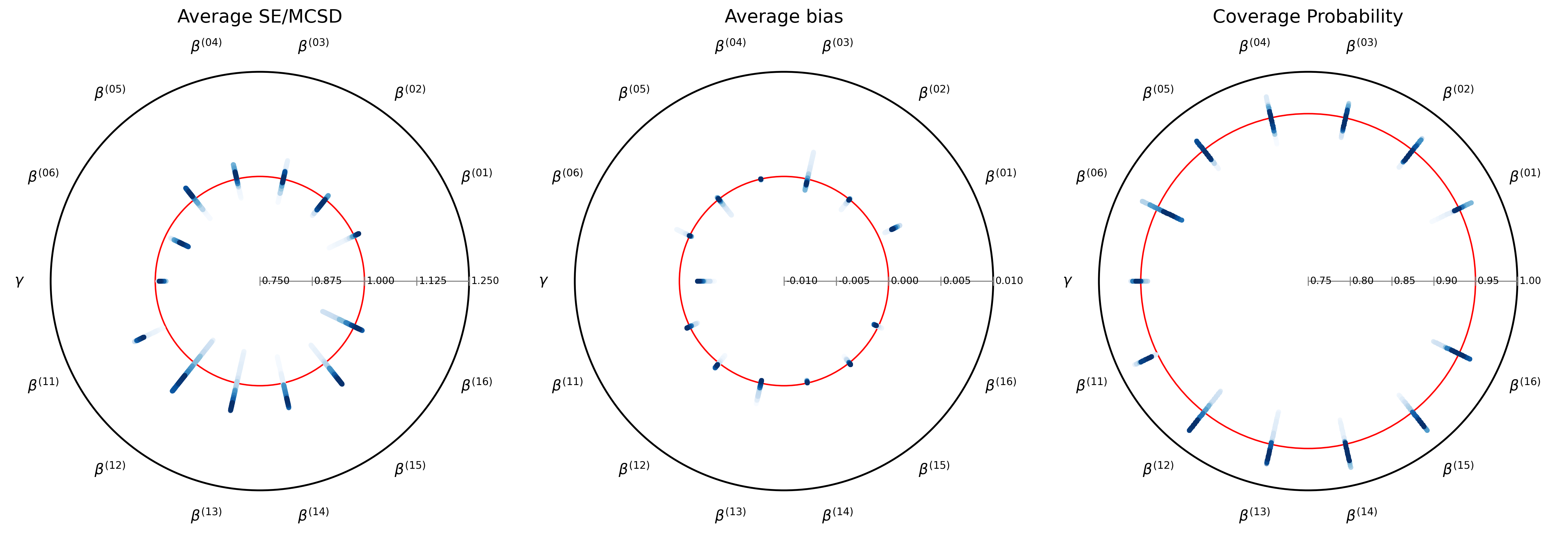}
    \caption{Consistency and asymptotic normality of the online estimator when $g(t) = \lfloor 0.2t \rfloor$ is shown in the figure. \textbf{Each ray represents a parameter, with dot colors transitioning from light to dark as the time steps increase.} 
    \textbf{Left panel:} Average SE/MCSD, with the red circle indicating the nominal level of 1. \textbf{Middle panel:} Average bias, with the red circle indicating the nominal level of 0. \textbf{Right panel:} Coverage probability, with the red circle indicating the nominal level of 95\%.}
    \label{fig:infer1}
\end{figure}

\begin{figure}[!thp]
    \centering
    \includegraphics[width=1\linewidth]{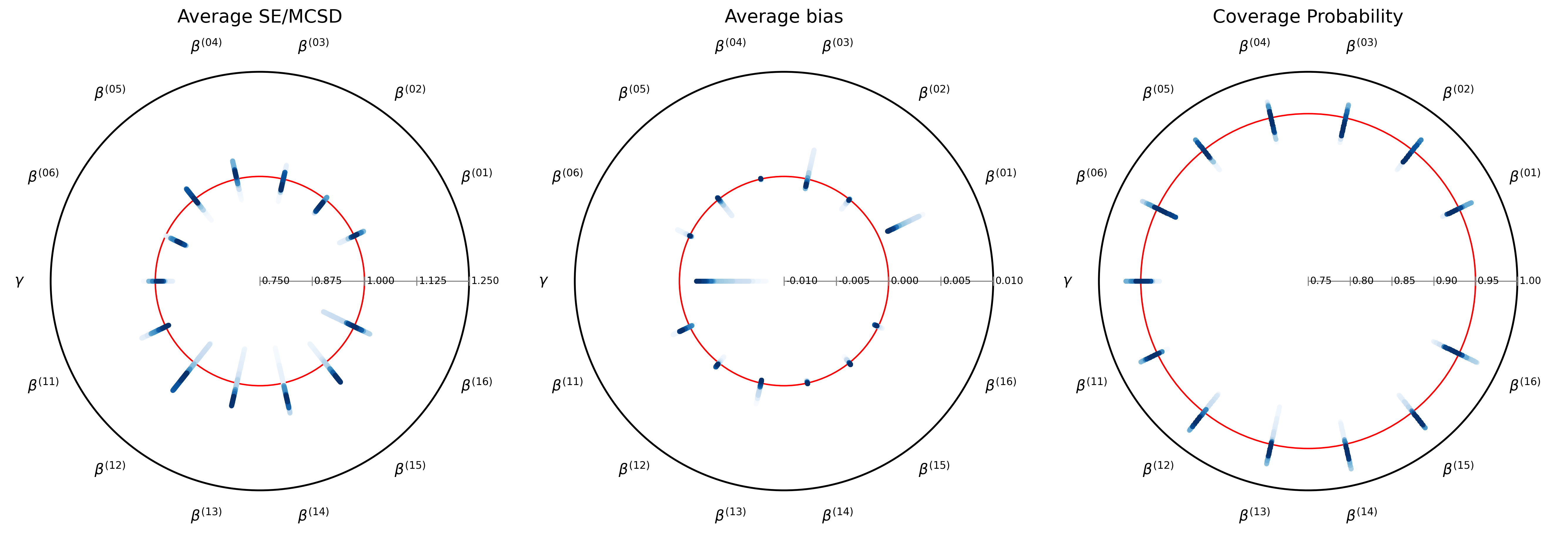}
    \caption{Consistency and asymptotic normality of the online estimator when $g(t)= \lfloor 5\sqrt{t} \rfloor$.}
    \label{fig:infer2}
\end{figure}

\begin{figure}[!thp]
    \centering
    \includegraphics[width=1\linewidth]{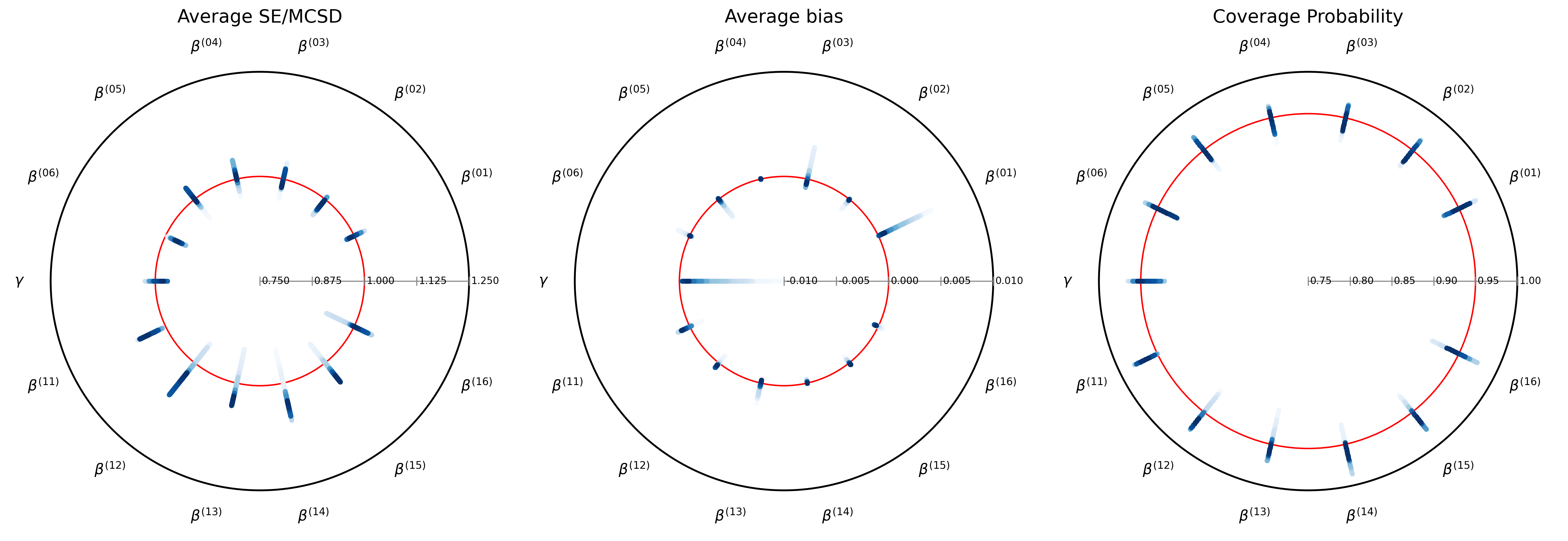}
    \caption{Consistency and asymptotic normality of the online estimator when $g(t) = \lfloor 20t^{0.2} \rfloor$.}
    \label{fig:infer3}
\end{figure}

\newpage 
\section{Real Data Analysis}
\label{sec:real}

In this section, we evaluate the performance of our proposed method, FRONT, using publicly available data from a large urban hotel chain (``Hotel 1") \citep{bodea2009}. This dataset can be used as a benchmark for assessing decision-making processes and choice-based revenue management algorithms. It includes records of room purchases and associated revenues with check-in dates between March 12, 2007, to April 15, 2007. 

We define the action space as binary, where $a_t=1$ represents an increase in the nightly rate decided by the hotel, and $a_t=0$ represents a decrease, compared to the average nightly rate for the specific room type. The lowest nightly rate is assumed as the cost for a given room type. The outcome $y_t$ is defined as the profit from each purchase, calculated by multiplying the difference between the nightly rate and the cost by the number of rooms and the length of stay. Additionally, we consider four context features: room type, advanced purchase dates, party size, and rate type (e.g., including other services, activities, or rewards). Hence we have $d_0=5$ including the intercept.

For this offline dataset, we retain only entries where the product was purchased by the customer and remove any records with invalid data, yielding a total of 1,961 observations. To mitigate the number of dummy variables and the imbalance they may cause, we treat room type and rate type as ordinal variables, ranking them by their corresponding average nightly rates. We further apply a logarithmic transformation to advance purchase dates for rescaling. As in the simulation study, we use a polynomial basis transformation $\phi(\bm{x})$ that includes all monomials up to total degree 2 for $\bm{x} \in \mathbb{R}^{d_0}$. This yields $\phi(\bm{x}) \in \mathbb{R}^{15}$, after incorporating the interference action, the dimension of $\bm{z}$ becomes $d = 31$.

Since the true model is unknown, we assume a linear conditional mean outcome model with an interference scale specified as $g(t) = \lfloor 5\sqrt{t} \rfloor$. This choice reflects the fact that each new price offered to a customer is influenced by previous prices. As time progresses, the hotel accumulates more pricing data from earlier customers, leading to an expanding neighborhood of past actions that affect subsequent decisions. Based on this model, we estimate the coefficients and simulate an online decision-making process. At each time $t$, a customer arrives, and we draw context features $\bm{x}_t$, calculate $\kappa_t$, and select an action $a_t$ based on the given context and updated policy. The profit is then observed from the distribution $\mathcal{N}(\mu(\bm{x}_t, \kappa_t, a_t), 10^2)$, where $\mu(\bm{x}_t, \kappa_t, a_t)$ represents the outcome model we fit earlier. The warm-up period is set to $T_0 = 200$ with $L = 8$, the force pulls parameter is $K = 50$, the clipping parameter is $C = 0.01$, and the exploration rate is $\epsilon_t = \log(t)/10\sqrt{t}$. Suppose the termination time $T$ is 20,000. Our goal is to maximize cumulative profit, and we apply three policies in this online setting: na\"ive, myopic, and FRONT to compare their performances. 

\begin{figure}[!t]
    \centering
    \vspace{-0.5cm}
    \includegraphics[width=0.6\linewidth]{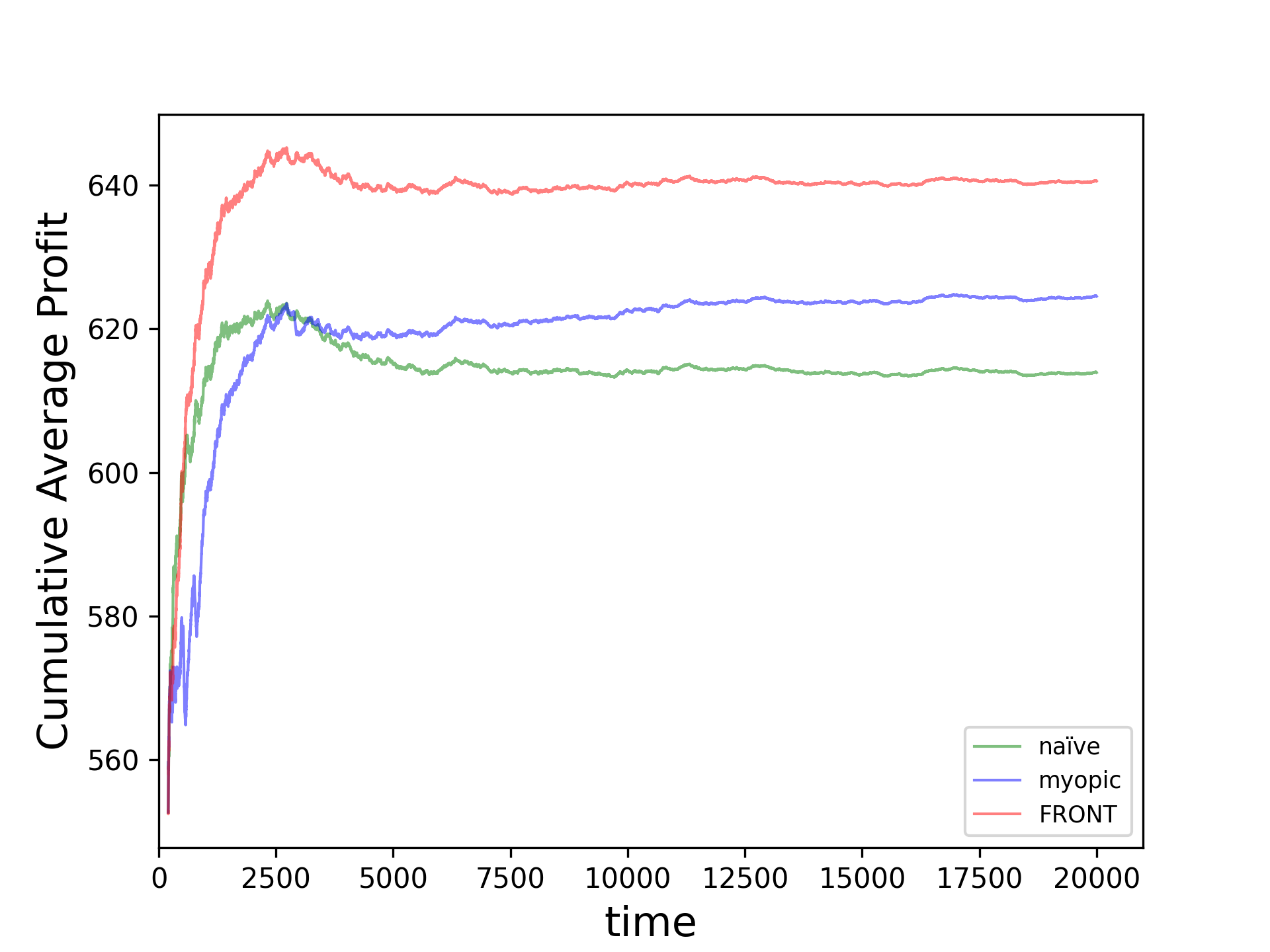}
    \caption{The cumulative average profit achieved by FRONT, na\"ive, and myopic policies on the hotel dataset.}
    \label{fig:hotel}
\end{figure}

Figure~\ref{fig:hotel} presents the cumulative average profit under the three policies. FRONT achieves substantially the highest cumulative profit, maintaining a clear margin over the other two policies. This highlights both the strength of FRONT and the importance of foresighted decision-making. We also observe that the myopic policy outperforms the na\"ive method, which is consistent with intuition. The initial intersection between the blue and green curves may be attributed to the force pulls triggered fewer than five times early on by the myopic policy; however, it soon surpasses the na\"ive policy. Under FRONT, no force pulls are triggered.

\section{Discussions}
\label{sec:dis}

\subsection{Practical Strategies and Other Working Models}
\label{sec:other}
We can select appropriate values for hyperparameters \( K \), \( T_0 \) and \( L \) to ensure sufficient exploration initially and choose suitable $C$ and $\kappa_0$ to avoid too many force pulls, thereby ensuring that $|\mathcal{F}_t| = o_p(\sqrt{T})$. Specifically, we recommend more exploration during the warm-up phase to establish a reliable estimator early, followed by a reduction in force pulls in later stages to sustain ongoing exploration. In Section~\ref{sec:FRONT}, we establish the method based on the online additive outcome model with heterogeneous treatment effects and homogeneous interference effects (described as Equation~\eqref{equ:conditionalmean}). We now extend our discussion to more general model formulations.

First, consider a model incorporating interaction between \(\kappa_t\) and \(a_t\): $\mu(\bm{x}_t, \kappa_t, a_t) = \mathbb{E}(y_t \vert \bm{x}_t, \kappa_t, a_t) = \phi(\bm{x}_t)^{\top}\bm{\beta}_0+a_t\phi(\bm{x}_t)^{\top}(\bm{\beta}_1-\bm{\beta}_0)+ \kappa_t\gamma_0 + a_t\kappa_t (\gamma_1-\gamma_0)$. The optimal action for the \(t\)-th user is a nested structure of indicator functions and they should be derived reversely. This formulation is detailed in Equations~\eqref{equ:othermodel1}, \eqref{equ:othermodel10}, and \eqref{equ:othermodel11} in Appendix~\ref{app:othermodel}. Consequently, determining the optimal policy becomes intractable, due to the unknown termination time $T$. 

Second, we consider a model with interaction between \(\kappa_t\) and \(\bm{x}_t\), i.e., $\mu(\bm{x}_t, \kappa_t, a_t) = \mathbb{E}(y_t \vert \bm{x}_t, \kappa_t, a_t) = \phi(\bm{x}_t)^{\top}\bm{\beta}_0+a_t\phi(\bm{x}_t)^{\top}(\bm{\beta}_1-\bm{\beta}_0)+\kappa_t\gamma_0+\kappa_t \phi(\bm{x}_t)^{\top} \bm{\gamma}$. In this case, the optimal policy depends on the contextual features of future individuals (see Equation~\eqref{equ:othermodel20} in the Appendix). We can develop the optimal strategy following the similar argument in Proposition \ref{prop:optimal}:
\begin{equation*}
    a_t^\ast = 
        \mathbb{I}\left\{\phi(\bm{x}_t)^{\top} (\bm{\beta}_1-\bm{\beta}_0) + \left[\sum_{s=t+1}^\infty w_{st}\right] \gamma_0 + \left[\sum_{s=t+1}^\infty w_{st}\phi(\bm{x}_s)^\top\right] \bm{\gamma} \ge 0 \right\},
\end{equation*}
where $\bm{x}_{t+1}, \bm{x}_{t+2}, \dots$ are unknown at time $t$. The term $\sum_{s=t+1}^\infty w_{st}\phi(\bm{x}s)^\top$ represents the weighted average of features from all subsequent individuals. Owing to the uncertainty in these contextual features, deriving the optimal action is generally challenging. Prediction becomes feasible only in special cases. For example, if $\kappa_t = \tfrac{1}{N}\sum_{s=t-N}^{t-1}a_s$, then
\begin{equation*}
    \sum_{s=t+1}^\infty w_{st}\phi(\bm{x}_s)^\top=\frac{1}{N}\sum_{s=t+1}^{t+N}\phi(\bm{x}_s)^\top
\end{equation*}
can be approximated by the sample mean $\sum_{s=1}^t \phi(\bm{x}_s)^\top /t$, or alternatively, by resampling from the observed feature vectors $\phi(\bm{x}_1)^\top, \dots, \phi(\bm{x}_t)^\top$ to generate pseudo-samples $\phi(\bm{x}_{t+1})^\top, \dots, \phi(\bm{x}_{t+N})^\top$, and then averaging the results.

\subsection{Conclusions and Future Work}
\label{sec:conclusion}
In this paper, we propose a novel policy FRONT, designed to optimize decision-making and maximize long-term benefits. To ensure practical applicability, we provide a comprehensive algorithm accompanied by exploration strategies. Furthermore, we establish the consistency and asymptotic normality of the online estimator under the general weight array $\left\{w_{ts} \right\}_{t=1,s<t}^\infty$. Finally, we analyze the regret under two different definitions and derive sublinear regret bounds for both scenarios.

There are several directions in which we can extend our work in the future. First, in contextual bandits, policy evaluation is crucial to determine when to stop updating the policy \citep[see e.g.,][]{chen2021,shen2024,xu2024}. We plan to define policy value that aligns with the two regret definitions and develop value estimation and inference when interference persists over time. Second, model misspecification presents a challenge in the interference setting, and we expect to address this issue in future research. Third, our current framework assumes a binary action set, and we aim to extend it to continuous action spaces. Finally, with respect to considering the interference effect on future outcomes, we have already incorporated the idea from reinforcement learning and will extend our proposal within the RL framework.

\bibliography{citations}
\bibliographystyle{agsm}

\newpage

\appendix
\newtheorem{lemma}{Lemma}   
 
\counterwithin{figure}{section}
\counterwithin{table}{section}
\counterwithin{equation}{section}

This appendix presents sensitivity analyses and additional simulation results in Section~\ref{app:simulation}, followed by the proofs of all the theorems for FRONT in Section~\ref{app:thms}. Auxiliary lemmas are provided in Section~\ref{app:lemmas}.

\section{Extended Simulation Results}
\label{app:simulation}
\subsection{Plot of $\kappa_t$}
\label{app:kappaplot}

As Figure \ref{fig:kappa} shows, \(\kappa_t\) approaches the optimal interference action $\kappa_{\infty}^\ast$ and will converge over time in simulations, which also verifies convergence of ISO. Note that the convergence in the left panel of Figure \ref{fig:kappa} is much more obvious due to its larger interference scale. When $g(t)$ is smaller, convergence takes longer time.

\begin{figure}[ht]
    \centering
    \includegraphics[width=1\linewidth]
    {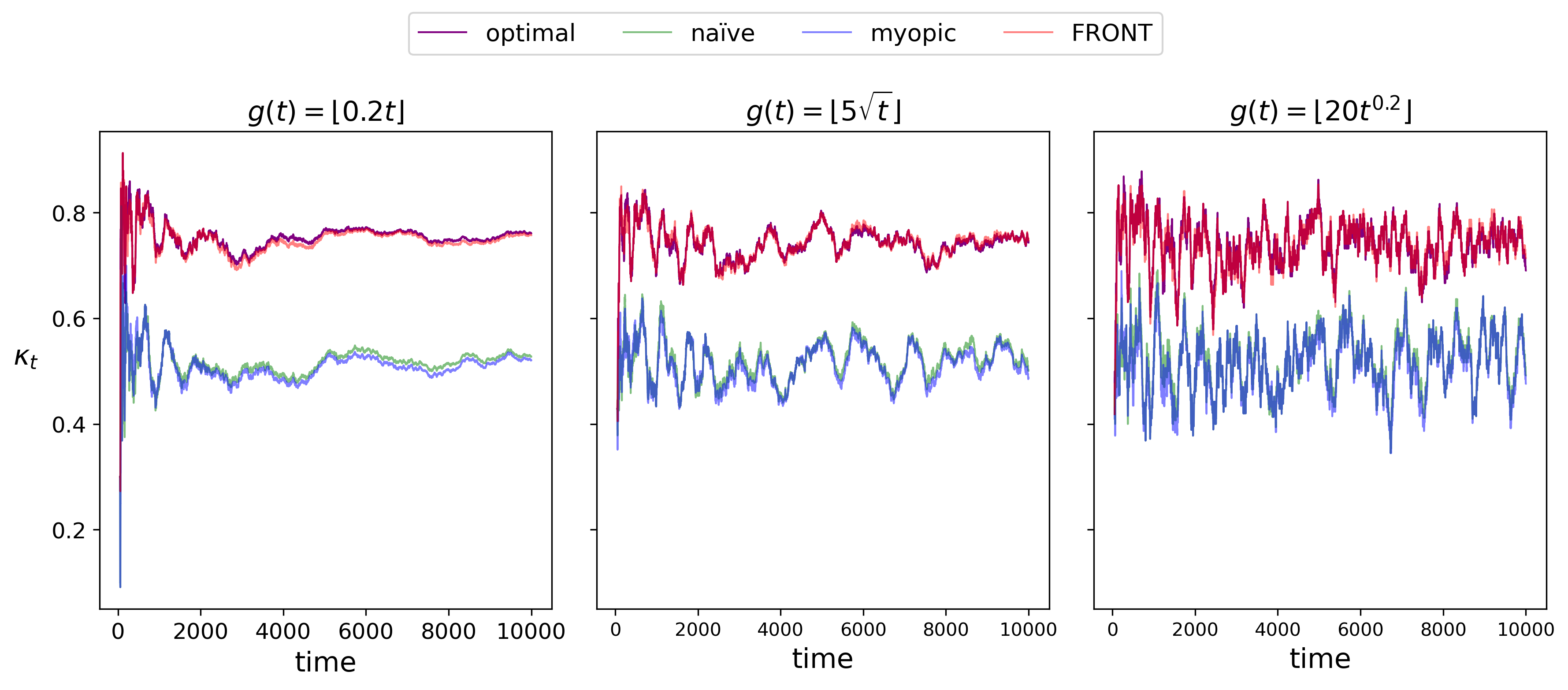}
    \caption{The convergence of $\kappa_t$ from one replication.}
    \label{fig:kappa}
\end{figure}

\subsection{Sensitivity Analyses}
\label{app:sensi}

We conduct a sensitivity analysis over different values of $T_0$, $L$, and $K$, with 500 simulation replications for $g(t) = \lfloor5\sqrt{t}\rfloor$. As shown in Figure~\ref{fig:sens1}, the SE-to-MCSD ratios, average biases, and coverage probabilities suggest that Algorithm~\ref{alg} is robust to these variations. Notably, no force pulls are triggered in any setting.

\begin{figure}[!t]
    \centering

    \begin{subfigure}{\linewidth}
        \centering
        \includegraphics[width=1\linewidth]{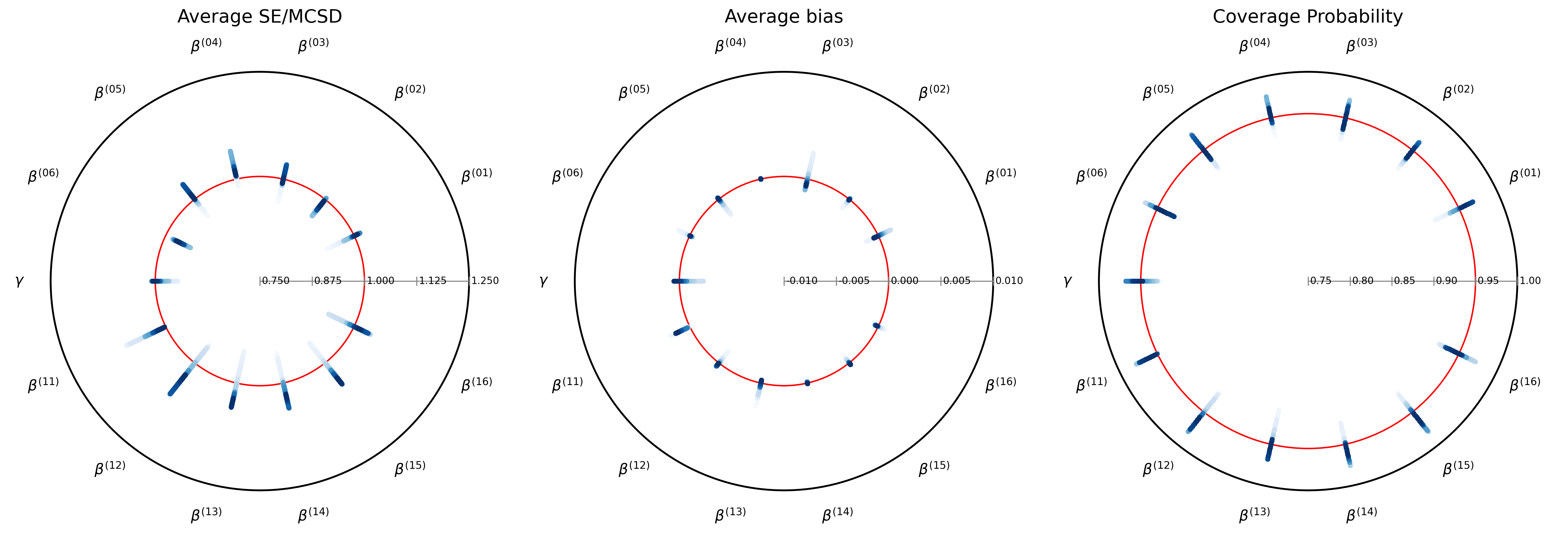}
        \caption{$T_0=50, K=80, L=4$}
    \end{subfigure}

    \begin{subfigure}{\linewidth}
        \centering
        \includegraphics[width=1\linewidth]{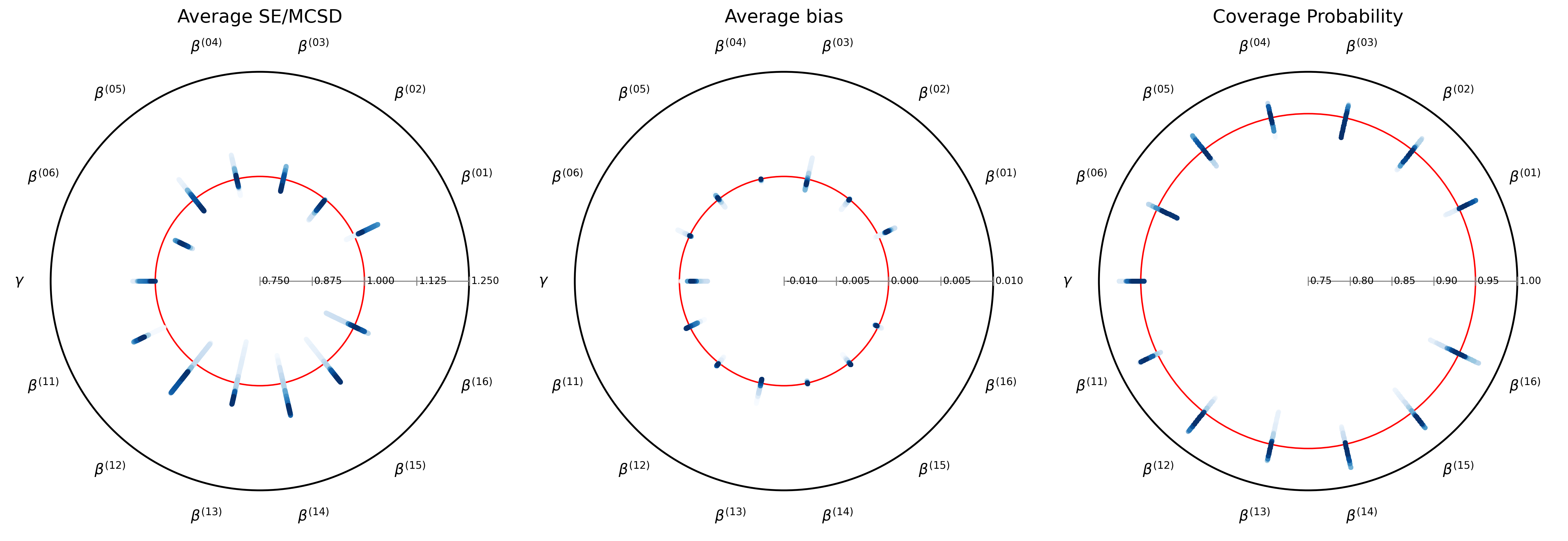}
        \caption{$T_0=100, K=50, L=6$}
    \end{subfigure}

    \begin{subfigure}{\linewidth}
        \centering
        \includegraphics[width=1\linewidth]{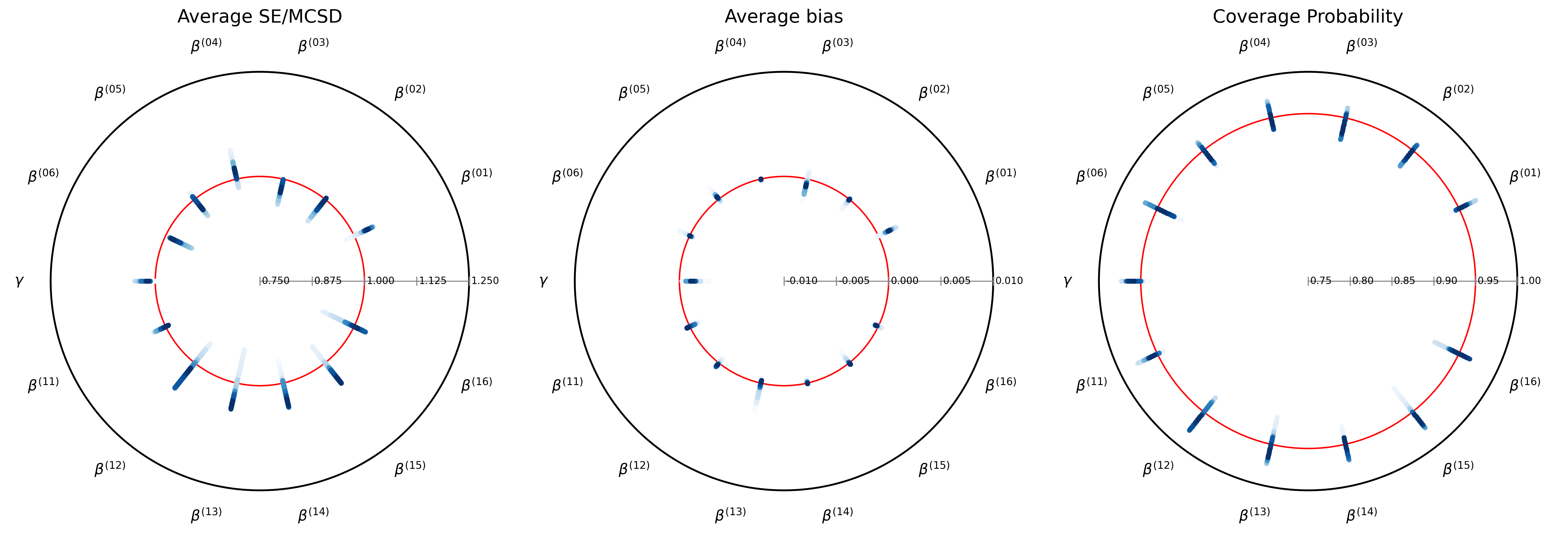}
        \caption{$T_0=200, K=30, L=8$}
    \end{subfigure}

    \caption{Consistency and asymptotic normality of the online estimator for $g(t) = \lfloor 5\sqrt{t} \rfloor$ under varying hyperparameters.}
    \label{fig:sens1}
\end{figure}

\subsection{Numerical Verification of the Existence of $\zeta_\infty$}
\label{app:zetagplot}
Beyond the special cases established in Corollary~\ref{theor:gt}, we provide numerical evidence supporting the existence of $\zeta_\infty$ under other equal-weight interference scale $g(t)$. Figure~\ref{fig:zeta} illustrates the trajectory of $\zeta_t$ for $100 \le t \le 10{,}000$ across four scenarios:
(i) $g(t) = \lfloor 20t^{0.2}\rfloor$ (a case also used in simulations),  
(ii) $g(t) = \lfloor 10t^{0.4}\rfloor$,  
(iii) $g(t) = \lfloor 5t^{0.6}\rfloor$, and  
(iv) $g(t) = \lfloor 2t^{0.8}\rfloor$.  
The stabilization of $\zeta_t$ over time provides evidence for the existence of its limit, $\zeta_\infty$.

\begin{figure}[ht]
    \centering
    \includegraphics[width=1\linewidth]{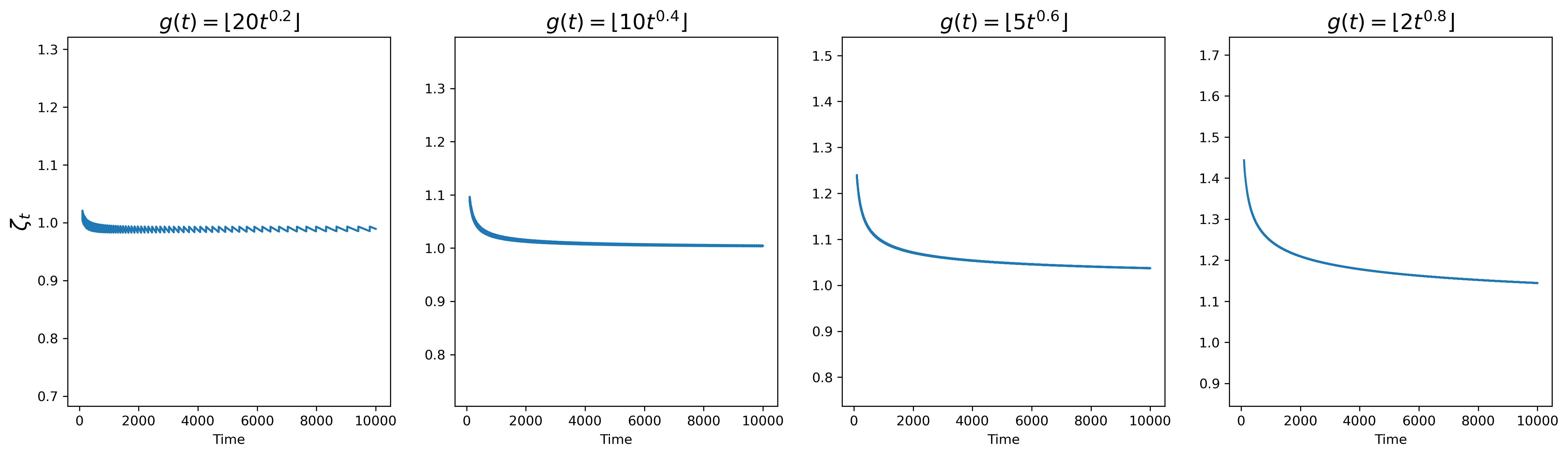}
    \caption{Convergence of $\zeta_t$ under different $g(t)$.}
    \label{fig:zeta}
\end{figure}

\subsection{Fixed-scale Interference Window}
\label{app:fix}
We apply the same hyperparameters and exploration rate for the fixed-scale interference window as those used in the simulation studies (Section~\ref{sec:simulation}). Force pulls are triggered fewer than five times.
Figure~\ref{fig:fixed_regret} shows results consistent with Figure~\ref{fig:reward}, further confirming the effectiveness of FRONT under fixed-scale interference windows.
\begin{figure}[ht]
    \centering
    \includegraphics[width=1\linewidth]{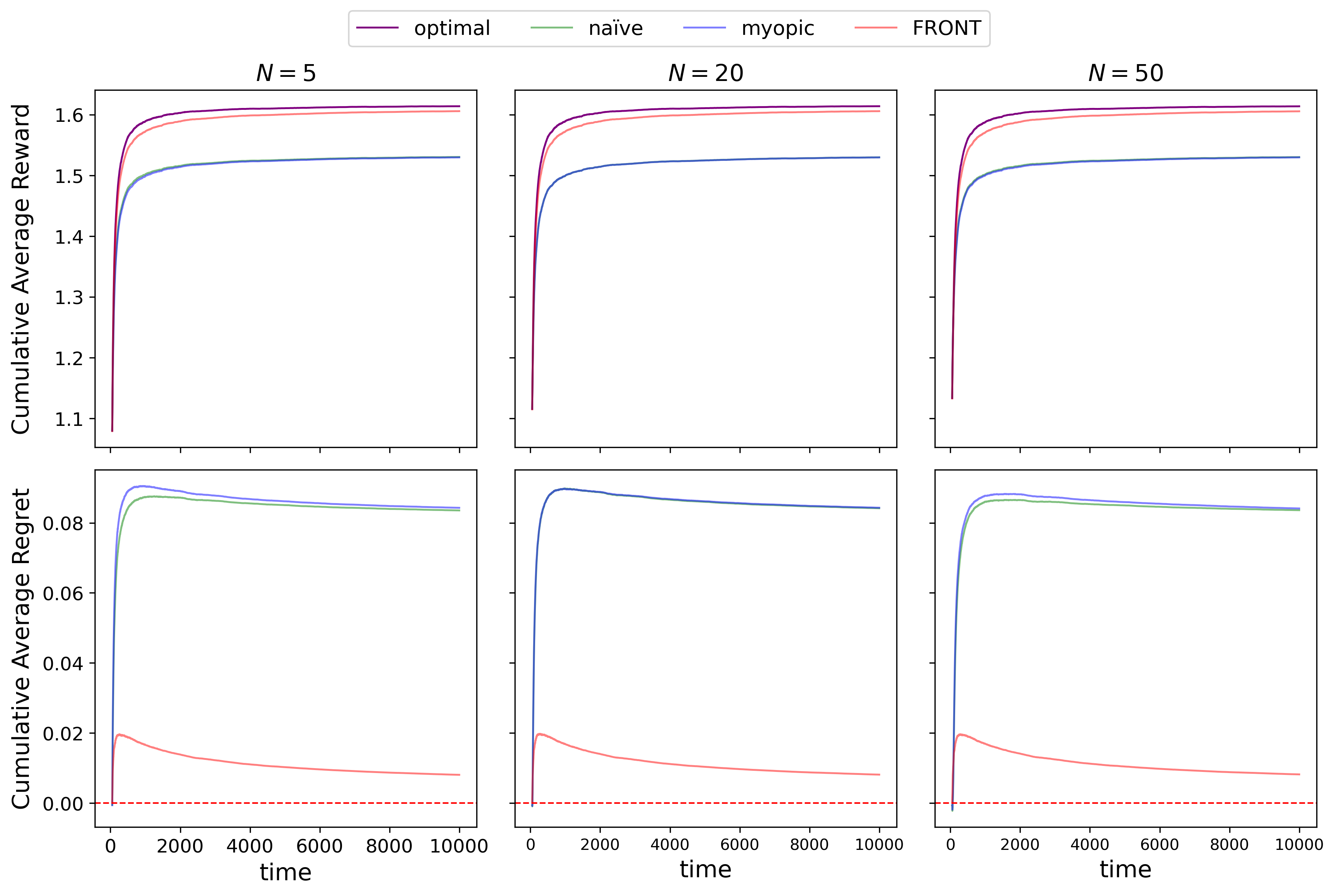}
    \caption{The performance of the online estimator with fixed-scale interference window when $N \in \{5,20,50\}$.}
    \label{fig:fixed_regret}
\end{figure}

When considering a fixed-scale interference window $N$, Corollary~\ref{coro:consistency} remains valid, whereas Theorem~\ref{theor:infer} no longer applies. Figure~\ref{fig:fixed_bias} reports the average bias for three choices of $N \in \left\{5,20,50\right\}$. As time increases, the average bias clearly converges to zero, confirming the consistency of the online estimator under our policy FRONT.

\begin{figure}[ht]
    \centering
    \includegraphics[width=1\linewidth]{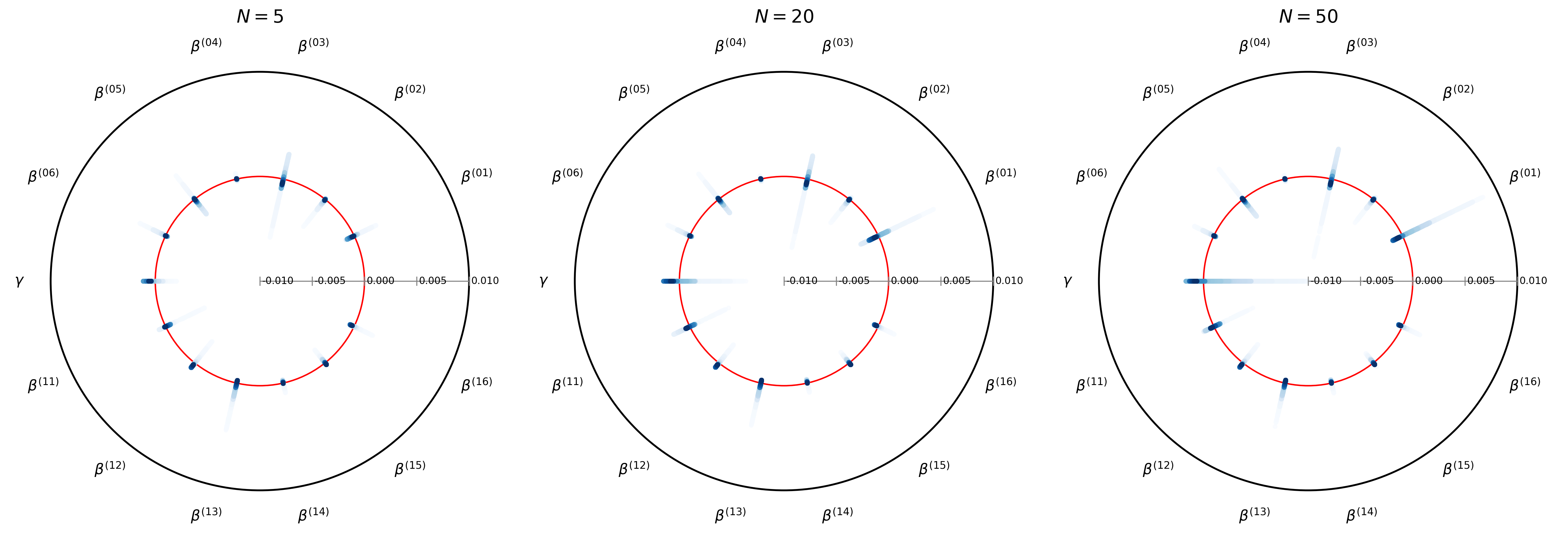}
    \caption{The average bias of the online estimator with fixed-scale interference window $N \in \{5,20,50\}$.}
    \label{fig:fixed_bias}
\end{figure}

\section{Proof of Main Results}
\label{app:thms}

\subsection{Proposition~\ref{prop:optimal} (Optimal policy with interference)}
\label{app:optimal}

\begin{proof}
The expected cumulative outcome is given by $\sum_{t=1}^\infty \mu(\bm{x}_t,  \kappa_t, a_t)$. Under our working model, it can be expanded as
$\sum_{t=1}^\infty \phi(\bm{x}_t)^{\top}\bm{\beta}_0+a_t\phi(\bm{x}_t)^{\top}(\bm{\beta}_1-\bm{\beta}_0)+ \kappa_t \gamma$.
Substituting $ \kappa_t$ with $\sum_{s=1}^{t-1} w_{ts} a_s$, we have
\begin{align*}
    \sum_{t=1}^\infty \mu(\bm{x}_t, \kappa_t, a_t) 
    &= \sum_{t=1}^\infty \phi(\bm{x}_t)^{\top}\bm{\beta}_0+a_t\phi(\bm{x}_t)^{\top}(\bm{\beta}_1-\bm{\beta}_0)+ \left[ \sum_{s=1}^{t-1} w_{ts} a_s \right] \gamma\\
    &= \sum_{t=1}^{\infty} \phi(\bm{x}_t)^{\top} \bm{\beta}_0 + a_t\left\{\phi(\bm{x}_t)^{\top} (\bm{\beta}_1 - \bm{\beta}_0) + \left[\sum_{s=t+1}^\infty w_{st} \right] \gamma \right\}.
\end{align*}
Because we want to maximize the cumulative outcome, then the optimal action should be
\begin{equation*}
    a_t^\ast = \mathbb{I}\left\{\phi(\bm{x}_t)^{\top} \bm{\beta}_0 + \phi(\bm{x}_t)^{\top} (\bm{\beta}_1 - \bm{\beta}_0) + \left[\sum_{s=t+1}^\infty w_{st} \right] \gamma \ge 0 \right\}.
\end{equation*}
The expansion of the cumulative reward is also used in the cumulative regret later, where we simply need to ignore the first term $\phi(\bm{x}_t)^{\top} \bm{\beta}_0$ in both definitions of regret.
\end{proof}

\subsection{Theorem~\ref{theor:tail} (Tail bound for the online estimator)}
\label{app:tailbound}
\begin{proof} 
Denote $\widehat{\Sigma}_t = \frac{1}{t} \sum_{s=1}^t \bm{z}_s \bm{z}_s^{\top}$. This theorem and its proof are adapted from \citet{bastani2020} and \citet{chen2021}, with minor modifications for our setting.
The relation between eigenvalue and $\ell_2$ norm of symmetric matrix gives
$\|\widehat{\Sigma}_t^{-1}\|_2 = \lambda_{\max}\!\left(\widehat{\Sigma}_t^{-1}\right)
= \left(\lambda_{\min}(\widehat{\Sigma}_t)\right)^{-1}. $ 
To derive the tail bound of $\widehat{\bm{\theta}}_t-\bm{\theta} = \widehat{\Sigma}^{-1}_t \left( \frac{1}{t} \sum_{s=1}^t \bm{z}_s e_s \right)$, 
we apply Assumption~\ref{ass:mineigen} to bound the minimum eigenvalue,
$\lambda_{\min}\left\{\frac{1}{t} \sum_{s=1}^t \bm{z}_s^{\top} \bm{z}_s \right\} = \lambda_{\min}(\widehat{\Sigma}_t) > C\epsilon_t$.
Therefore, the $\ell_2$ norm for $\widehat{\bm{\theta}}_{t} - \bm{\theta}$ satisfies
\begin{align*}
\|\widehat{\bm{\theta}}_{t} - \bm{\theta}\|_2
= \left\|
\widehat{\Sigma}_t^{-1}
\left( \frac{1}{t}\sum_{s=1}^{t} \bm{z}_s e_s \right)
\right\|_2  \le \frac{1}{t} \|\widehat{\Sigma}_t^{-1}\|_2
\left\|\sum_{s=1}^{t} \bm{z}_s e_s \right\|_2  \le \frac{1}{tC \epsilon_t}
\left\|\sum_{s=1}^{t} \bm{z}_s e_s \right\|_2.
\end{align*}
Denote the $j$th element of $\bm{z}_t$ as $Z_{j,t}$ for $j = 1, \cdots, d$.
Now, for $h > 0$,
\[ 
\resizebox{1\hsize}{!}{$ \begin{aligned}
&\Pr\!\left( \|\widehat{\bm{\theta}}_{t} - \bm{\theta}\|_2 \le h \right)
 \ge \Pr\!\left(
\left\|\sum_{s=1}^{t} \bm{z}_s e_s\right\|_2
\le tC \epsilon_t h
\right) \ge \Pr\!\left(
\left|\sum_{s=1}^{t} Z_{1,s} e_s\right|
\le \frac{tC \epsilon_t h}{\sqrt{d}},
\; \dots,
\left|\sum_{s=1}^{t} Z_{d,s} e_s\right|
\le \frac{tC \epsilon_t h}{\sqrt{d}}
\right) \\
&= 1 - \Pr \!\left(
\bigcup_{j=1}^{d}
\left\{
\left|\sum_{s=1}^{t} Z_{j,s} e_s\right|
> \frac{tC \epsilon_t h}{\sqrt{d}}
\right\}
\right)  \ge 1 - \sum_{j=1}^{d}
\Pr \!\left(
\left|\sum_{s=1}^{t} Z_{j,s} e_s \right|
> \frac{tC \epsilon_t h}{\sqrt{d}}
\right). 
\end{aligned}$}
\] 
We split $\sum_{j=1}^{d}\Pr \!\left(\left|\sum_{s=1}^{t} Z_{j,s} e_s \right| > \frac{tC \epsilon_t h}{\sqrt{d}} \right)$ into three parts:
\begin{equation}
\label{equ13}
\sum_{j=1}^{d_1}\Pr \!\left(
\left|\sum_{s=1}^{t} Z_{j,s} e_s \right|
> \frac{tC \epsilon_t h}{\sqrt{d}}
\right) + \sum_{j=d_1+1}^{2d_1}\Pr \!\left(
\left|\sum_{s=1}^{t} Z_{j,s} e_s \right|
> \frac{tC \epsilon_t h}{\sqrt{d}}
\right) \,  + \Pr \!\left(
\left|\sum_{s=1}^{t} Z_{d,s} e_s \right|
> \frac{tC \epsilon_t h}{\sqrt{d}}
\right).   
\end{equation} 
We know that $\bm{z}_t = ((1-a_t)\phi(\bm{x}_t)^{\top},  a_t\phi(\bm{x}_t)^{\top}, \kappa_t)^{\top}$, thus
\[ 
\resizebox{1\hsize}{!}{$ \begin{aligned}
&\sum_{j=1}^{d_1}\Pr \!\left(
\left|\sum_{s=1}^{t} Z_{j,s} e_s \right|
> \frac{tC \epsilon_t h}{\sqrt{d}}
\right)  = \sum_{j=1}^{d_1}\Pr \!\left(
\left|\sum_{s=1}^{t} (1-a_s)\phi(\bm{x}_s)_j \, e_s \right|
> \frac{tC \epsilon_t h}{\sqrt{d}}
\right) \\
=& \sum_{j=1}^{d_1}\Pr \!\left(
\left|\sum_{s=1}^{t} \mathbb{I}\left\{a_s=0 \right\}\phi(\bm{x}_s)_j \, e_s^{(0)} \right|
> \frac{tC \epsilon_t h}{\sqrt{d}}
\right)   \le \sum_{j=1}^{d_1}2 \exp \left\{-\frac{t\epsilon_t^2 C^2 h^2}{2d\sigma^2 L_z^2}\right\}  = 2d_1 \exp \left\{-\frac{t\epsilon_t^2 C^2 h^2}{2d\sigma^2 L_z^2}\right\},
\end{aligned}$}
\]
where $e_s^{(i)} = e_s$ when $a_s=i$ and 0 otherwise and $e_s^{(0)} \perp \bm{z}_s$ so the last inequality follows from Lemma~1 in \cite{chen2021}.
Similarly, we get
$\sum_{j=d_1+1}^{2d_1}\Pr \!\left(
\left|\sum_{s=1}^{t} Z_{j,s} e_s \right|
> \frac{tC \epsilon_t h}{\sqrt{d}}
\right) \le 2d_1 \exp \left\{-\frac{t\epsilon_t^2 C^2 h^2}{2d\sigma^2 L_z^2}\right\}.$
Focusing on the last part in~\eqref{equ13} with $Z_{d,s}=\kappa_s$, we have,
\[ 
\resizebox{1\hsize}{!}{$ \begin{aligned}
&\Pr \!\left(
\left|\sum_{s=1}^{t} Z_{d,s} e_s \right|
> \frac{tC \epsilon_t h}{\sqrt{d}}
\right) 
= \, \Pr \!\left(
\left|\sum_{s=1}^{t} \kappa_s \mathbb{I}\left\{a_s=0 \right\} e_s^{(0)} + \sum_{s=1}^{t} \kappa_s \mathbb{I}\left\{a_s=1 \right\} e_s^{(1)} \right|
> \frac{tC \epsilon_t h}{\sqrt{d}} \right) \\
\le&\, \Pr \!\left(
\left|\sum_{s=1}^{t} \kappa_s \mathbb{I}\left\{a_s=0 \right\} e_s^{(0)} \right| + \left| \sum_{s=1}^{t} \kappa_s \mathbb{I}\left\{a_s=1 \right\} e_s^{(1)} \right|
> \frac{tC \epsilon_t h}{\sqrt{d}} \right) \\
\le&\, \Pr \!\left(
\left|\sum_{s=1}^{t} \kappa_s \mathbb{I}\left\{a_s=0 \right\} e_s^{(0)} \right| > \frac{tC \epsilon_t h}{2\sqrt{d}}\right) + \Pr \!\left( \left| \sum_{s=1}^{t} \kappa_s \mathbb{I}\left\{a_s=1 \right\} e_s^{(1)} \right|
> \frac{tC \epsilon_t h}{2\sqrt{d}} \right) 
\le \, 4\exp \left\{-\frac{t\epsilon_t^2 C^2 h^2}{8d\sigma^2 L_z^2}\right\},
\end{aligned}$}
\]
where the last inequality is obtained by applying the same argument used in the first two parts.
Combining the three parts, we obtain
\begin{align*}
\Pr\!\left( \|\widehat{\bm{\theta}}_{t} - \bm{\theta}\|_2 \le h \right)
&\ge 1- 4d_1 \exp \left\{-\frac{t\epsilon_t^2 C^2 h^2}{2d\sigma^2 L_z^2}\right\} - 4\exp \left\{-\frac{t\epsilon_t^2 C^2 h^2}{8d\sigma^2 L_z^2}\right\}.
\end{align*}
Now, we establish the tail inequality for $\ell_1$ norm. By the Cauchy--Schwarz inequality, for any $\bm{a} \in \mathbb{R}^d$ we have $\Vert \bm{a} \Vert_1 \le \Vert \bm{a} \Vert_2 \sqrt{d}$. Therefore,
\[ 
\resizebox{1\hsize}{!}{$ \begin{aligned}
\Pr\!\left( \|\widehat{\bm{\theta}}_{t} - \bm{\theta}\|_1 \le h \right) 
\ge \Pr\!\left( \|\widehat{\bm{\theta}}_{t} - \bm{\theta}\|_2 \le \frac{h}{\sqrt{d}} \right) \ge 1- 4d_1 \exp \left\{-\frac{t\epsilon_t^2 C^2 h^2}{2d^2\sigma^2 L_z^2}\right\} - 4\exp \left\{-\frac{t\epsilon_t^2 C^2 h^2}{8d^2\sigma^2 L_z^2}\right\}.
\end{aligned}$}
\] 
\end{proof}

\subsection{Theorem~\ref{theor:kappaast} (Convergence of optimal interference)}
\label{app:ISO}

\begin{proof}
    Denote $u_s = \mathbb{E}(a_s^\ast) = \Pr\!\left\{\phi(\bm{x})^{\top} (\bm{\beta}_1-\bm{\beta}_0) + \left[\sum_{k=s+1}^\infty w_{ks}\right]\gamma \ge 0\right\}$. We know that $\text{Var}(a_t) \le 1$ since $a_t \in \left\{0,1 \right\}.$
    Without loss of generality, we set $w_{tt} = 0$ for all $t$, which simplifies the subsequent proof. This convention will be adopted throughout the remainder of the proofs.
    \begin{equation*}
         \kappa_t^\ast -  \kappa_\infty^\ast = \sum_{s=1}^{t} w_{ts} a_s^\ast -  \kappa_\infty^\ast = \underbrace{\sum_{s=1}^{t} w_{ts} (a_s^\ast-u_s)}_{b_t} +\underbrace{\sum_{s=1}^{t} w_{ts}u_s -  \kappa_\infty^\ast}_{c_t},
    \end{equation*}
    where $\mathbb{E}(b_t) = 0$. Then we $\mathbb{E}( \kappa_t^\ast -  \kappa_\infty^\ast) = c_t$.
    Since $ \kappa_\infty^\ast = \lim_{t \to \infty} \sum_{s=1}^{t} w_{ts} u_s$, we have $c_t \to 0$ as $t \to \infty$. Moreover, $
    \mathbb{E}(b_t^2) = \text{Var}(b_t) = \sum_{s=1}^t w_{ts}^2 \text{Var}(a_s^\ast) \to 0. $ 
Hence,
\begin{equation*}
    \mathbb{E}\!\left( \kappa_t^\ast -  \kappa_\infty^\ast \right)^2 
    = \mathbb{E}\!\left(b_t + c_t \right)^2 
    = \mathbb{E}(b_t^2) + 2c_t \mathbb{E}(b_t) + c_t^2 
    = \mathbb{E}(b_t^2) + c_t^2 \to 0,
\end{equation*}
as $t \to \infty$. Finally, by Chebyshev’s inequality, for any $\epsilon > 0$,
\begin{equation*}
    \Pr\!\left(\big| \kappa_t^\ast -  \kappa_\infty^\ast - c_t\big| > \epsilon\right) 
    \le \frac{\mathbb{E}\!\left( \kappa_t^\ast -  \kappa_\infty^\ast \right)^2}{\epsilon^2} \to 0.
\end{equation*}
Therefore, $ \kappa_t^\ast - c_t \stackrel{p}{\to}  \kappa_\infty^\ast$, which implies that $ \kappa_t^\ast \stackrel{p}{\to}  \kappa_\infty^\ast$ as $t \to \infty$.

\end{proof}

\subsection{Theorem~\ref{theor:kappat} (Convergence of $ \kappa_t$)}
\label{app:kappa}

\begin{proof}
Since force pulls are assumed not to affect the convergence of $\kappa_t$, we only consider actions under FRONT. FRONT selects actions according to $a_t \sim \text{Bernoulli}(\widehat{\pi}_t),$ with
$\widehat{\pi}_t = (1-\epsilon_t)\,\mathbb{I}\!\left\{\phi(\bm{x}_t)^{\top}\widehat{\bm{\beta}}_{t-1} + \zeta_t\widehat{\gamma}_{t-1} \geq 0 \right\} + {\epsilon_t}/{2}$ and $\widehat{\bm{\beta}}_{t-1} = \widehat{\bm{\beta}}_{1,t-1} - \widehat{\bm{\beta}}_{0,t-1}$.

\medskip
\noindent
    We begin with $\mathbb{E}(\kappa_t-\kappa_\infty)^2$ in preparation for applying Markov's inequality. 
    Since $\kappa_\infty$ is a constant and $\{\kappa_t\}_{t=1}^\infty$ is a sequence of random variables, we can expand as
    \begin{equation}
    \label{equ1}
        \mathbb{E}(\kappa_t-\kappa_\infty)^2 = \mathbb{E}(\kappa_t^2) - 2\kappa_\infty \,\mathbb{E}(\kappa_t) + \kappa_\infty^2 = \mathbb{E}\left(\sum_{s=1}^{t}w_{ts}a_s\right)^2 - 2\kappa_\infty \sum_{s=1}^{t}w_{ts}\,\mathbb{E}(a_s) + \kappa_\infty^2.
    \end{equation}
    
    \paragraph{Step 1: We first analyze the second term in \eqref{equ1}.} By iterated expectations,
    \begin{align*}
    \mathbb{E}(a_t \vert \mathcal{H}_{t-1}) &= \mathbb{E}[\mathbb{E}(a_t \vert \mathcal{H}_{t-1},\phi(\bm{x}_t))\vert \mathcal{H}_{t-1}] \\
    &=\mathbb{E} \left((1-\epsilon_t) \mathbb{I}\left\{\phi(\bm{x}_t)^{\top} \widehat{\bm{\beta}}_{t-1} + \zeta_t \widehat{\gamma}_{t-1} \geq 0 \right\} + \frac{\epsilon_t}{2} \Big| \mathcal{H}_{t-1} \right) \\
    &= (1-\epsilon_t) \Pr\left\{\phi(\bm{x})^{\top} \widehat{\bm{\beta}}_{t-1} + \zeta_t \widehat{\gamma}_{t-1} \geq 0 \Big| \widehat{\bm{\beta}}_{0,t-1}, \widehat{\bm{\beta}}_{1,t-1}, \widehat{\gamma}_{t-1} \right\} + \tfrac{\epsilon_t}{2} \\
    &= (1-\epsilon_t) q\!\left(\zeta_t,\widehat{\bm{\beta}}_{t-1},\widehat{\gamma}_{t-1}\right) + \tfrac{\epsilon_t}{2},
    \end{align*}
    where we define
$q\!\left(\zeta_t,\widehat{\bm{\beta}}_{t-1},\widehat{\gamma}_{t-1}\right) 
= \Pr\left\{\phi(\bm{x})^{\top} \widehat{\bm{\beta}}_{t-1} + \zeta_t \widehat{\gamma}_{t-1} \geq 0 \Big| \widehat{\bm{\beta}}_{0,t-1}, \widehat{\bm{\beta}}_{1,t-1}, \widehat{\gamma}_{t-1} \right\}$, and the probability is taken with respect to the distribution of $\bm{x}$.  
It then follows that
\[
\mathbb{E}(a_s) 
= \mathbb{E}\!\left(\mathbb{E}[a_s \mid \mathcal{H}_{s-1}]\right) 
= \left(1-\frac{\epsilon_s}{2}\right)\mathbb{E}\!\left[q\!\left(\zeta_s,\widehat{\bm{\beta}}_{s-1},\widehat{\gamma}_{s-1}\right)\right] + \tfrac{\epsilon_s}{2},
\]
where the outer expectation is taken with respect to the distribution of the parameter estimates. Then we will show $\sum_{s=1}^t w_{ts}\,\mathbb{E}(a_s) \to \kappa_\infty$.
\begin{align*}
    &\sum_{s=1}^t w_{ts}\,\mathbb{E}(a_s) - \sum_{s=1}^t w_{ts}\left[\left(1-\tfrac{\epsilon_\infty}{2} \right)\mathbb{E}(a_s^\ast)+\tfrac{\epsilon_\infty}{2}\right] \\
    = & \underbrace{\sum_{s=1}^t w_{ts}\left\{\left[\left(1-\tfrac{\epsilon_s}{2} \right)q\!\left(\zeta_s,\widehat{\bm{\beta}}_{s-1},\widehat{\gamma}_{s-1}\right) \right] - \left[ \left(1-\tfrac{\epsilon_\infty}{2} \right) q\!\left(\zeta_s,\bm{\beta},\gamma \right)\right]\right\}}_{\Delta_1} + \underbrace{\tfrac{1}{2}\sum_{s=1}^t w_{ts} (\epsilon_s - \epsilon_\infty)}_{\Delta_2}.
\end{align*}
To bound $\Delta_1$, we add and subtract the same term, yielding
\[ 
\resizebox{1\hsize}{!}{$ \begin{aligned}
    \Delta_1 =& \underbrace{\sum_{s=1}^t w_{ts}\left\{\left[\left(1-\tfrac{\epsilon_s}{2} \right)q\!\left(\zeta_s,\widehat{\bm{\beta}}_{s-1},\widehat{\gamma}_{s-1}\right) \right] -\left[\left(1-\tfrac{\epsilon_s}{2} \right)q\!\left(\zeta_s,\bm{\beta},\gamma\right) \right] \right\}}_{\Delta_{11}} + \underbrace{\sum_{s=1}^t w_{ts}\left\{\left[\left(1-\tfrac{\epsilon_s}{2} \right)q\!\left(\zeta_s,\bm{\beta},\gamma\right) \right] - \left[ \left(1-\tfrac{\epsilon_\infty}{2} \right) q\!\left(\zeta_s,\bm{\beta},\gamma \right)\right]\right\}}_{\Delta_{12}}.
\end{aligned}$}
\] 
Then we have $\Delta_{11} = \sum_{s=1}^t w_{ts} \left(1-\tfrac{\epsilon_s}{2} \right) \left[ q\!\left(\zeta_s,\widehat{\bm{\beta}}_{s-1},\widehat{\gamma}_{s-1}\right)-q\!\left(\zeta_s,\bm{\beta},\gamma\right)\right]$.
By Corollary~\ref{coro:consistency} and the Continuous Mapping Theorem, we have 
$q\!\left(\zeta_s,\widehat{\bm{\beta}}_{s-1},\widehat{\gamma}_{s-1}\right) 
- q\!\left(\zeta_s,\bm{\beta},\gamma\right)\to 0$.
Consequently, 
$\left(1-\tfrac{\epsilon_s}{2}\right) 
\left[q\!\left(\zeta_s,\widehat{\bm{\beta}}_{s-1},\widehat{\gamma}_{s-1}\right) 
- q\!\left(\zeta_s,\bm{\beta},\gamma\right)\right] \to 0.$
Applying Lemma~\ref{lemma:Toeplitz}, it follows that $\Delta_{11} \to 0$ as $t \to \infty$. 
Moreover, $\Delta_{12} 
= \tfrac{1}{2}\sum_{s=1}^t w_{ts}\, q\!\left(\zeta_s,\bm{\beta},\gamma\right)\bigl(\epsilon_\infty - \epsilon_s \bigr)$.
Since $q\!\left(\zeta_s,\bm{\beta},\gamma\right)\bigl(\epsilon_\infty - \epsilon_s \bigr) \to 0$, another application of Lemma~\ref{lemma:Toeplitz} yields $\Delta_{12} \to 0$. 
Similarly, applying Lemma~\ref{lemma:Toeplitz} to $\Delta_2$ gives $\Delta_2 \to 0$. 
Therefore,
\[
\sum_{s=1}^t w_{ts}\,\mathbb{E}(a_s) 
- \sum_{s=1}^t w_{ts}\!\left[\Bigl(1-\tfrac{\epsilon_\infty}{2}\Bigr)\mathbb{E}(a_s^\ast) + \tfrac{\epsilon_\infty}{2}\right] \;\to\; 0,
\]
which implies
\begin{equation}
\label{equ4}
    \lim_{t \to \infty} \sum_{s=1}^t w_{ts}\,\mathbb{E}(a_s) 
= \kappa_\infty,
\end{equation}
where $\kappa_\infty = \left(1-\epsilon_\infty\right)\kappa_\infty^\ast + \tfrac{\epsilon_\infty}{2}$.

\paragraph{Step 2: We then analyze the first term in \eqref{equ1}.}
We expand it to be
\begin{align}
    \mathbb{E}\left(\sum_{s=1}^{t}w_{ts} a_s\right)^2 &= \mathbb{E} \left(\sum_{s=1}^t w_{ts}^2 a_s^2 + 2\sum_{i<j}w_{ti}w_{tj}a_ia_j\right)  = \sum_{s=1}^t w_{ts}^2 \mathbb{E} (a_s) + 2\sum_{i<j}^t w_{ti}w_{tj}\,\mathbb{E} (a_ia_j).
    \label{equ2}
\end{align}
For the first term of \eqref{equ2}, we have $\sum_{s=1}^t w_{ts}^2 \mathbb{E} (a_s) \le \sum_{s=1}^t w_{ts}^2 \to 0$.

\noindent We then turn to the second term of \eqref{equ2}. When $i<j$,
\begin{align*}
\mathbb{E}(a_i a_j|\mathcal{H}_{j-1}) &= \mathbb{E}[\mathbb{E}(a_i a_j \vert \mathcal{H}_{j-1},\phi(\bm{x}_j))\vert \mathcal{H}_{j-1}] = \mathbb{E}[a_i\widehat{\pi}_j\vert \mathcal{H}_{j-1}] \\
&=\mathbb{E} \left(a_i \left[ (1-\epsilon_j) \mathbb{I}\left\{\phi(\bm{x}_j)^{\top} \widehat{\bm{\beta}}_{j-1} + \zeta_j \widehat{\gamma}_{j-1} \geq 0 \right\} + \tfrac{\epsilon_j}{2} \right]\Big| \mathcal{H}_{j-1} \right) \\
&= a_i \left[ (1-\epsilon_j) \Pr\left\{\phi(\bm{x})^{\top} \widehat{\bm{\beta}}_{j-1} + \zeta_j \widehat{\gamma}_{j-1} \geq 0 \Big| \widehat{\bm{\beta}}_{0,t-1}, \widehat{\bm{\beta}}_{1,j-1}, \widehat{\gamma}_{j-1} \right\} + \tfrac{\epsilon_j}{2} \right]\\
    &= a_i \left[ (1-\epsilon_j) q\!\left(\zeta_j,\widehat{\bm{\beta}}_{j-1},\widehat{\gamma}_{j-1}\right) + \tfrac{\epsilon_j}{2} \right].
\end{align*}
By iterated expectations, we get
\begin{align*}
    \mathbb{E}(a_i a_j) = \mathbb{E}\left(\mathbb{E}(a_i a_j|\mathcal{H}_{j-1}) \right) = \mathbb{E}\left(a_i \left[ (1-\epsilon_j) q\!\left(\zeta_j,\widehat{\bm{\beta}}_{j-1},\widehat{\gamma}_{j-1}\right) + \tfrac{\epsilon_j}{2} \right]\right),
\end{align*}
where the outer expectation is taken with respect to the distribution of the parameter estimates and the action $a_i$. We next consider the difference
\[ 
\resizebox{1\hsize}{!}{$ \begin{aligned}
     & \sum_{i<j}^t w_{ti}w_{tj}\,\mathbb{E}(a_i a_j) 
     - \sum_{i<j}^t w_{ti}w_{tj} 
     \Biggl[\Bigl(1-\tfrac{\epsilon_\infty}{2}\Bigr) q\!\left(\zeta_i,\bm{\beta},\gamma \right) 
     + \tfrac{\epsilon_\infty}{2} \Biggr]
     \Biggl[\Bigl(1-\tfrac{\epsilon_\infty}{2}\Bigr) q\!\left(\zeta_j,\bm{\beta},\gamma \right) 
     + \tfrac{\epsilon_\infty}{2} \Biggr] \\
     ={}& 
     \underbrace{\sum_{i<j}^t w_{ti}w_{tj}\,\mathbb{E}(a_i a_j) 
     - \sum_{i<j}^t w_{ti}w_{tj}\,
       \mathbb{E}\!\left[a_i\Bigl(\Bigl(1-\tfrac{\epsilon_\infty}{2}\Bigr) q\!\left(\zeta_j,\bm{\beta},\gamma \right) 
       + \tfrac{\epsilon_\infty}{2}\Bigr)\right]}_{\Delta_3} \\
     \,+& \underbrace{\sum_{i<j}^t w_{ti}w_{tj}\,
       \mathbb{E}\!\left[a_i\Bigl(\Bigl(1-\tfrac{\epsilon_\infty}{2}\Bigr) q\!\left(\zeta_j,\bm{\beta},\gamma \right) 
       + \tfrac{\epsilon_\infty}{2}\Bigr)\right] - \sum_{i<j}^t w_{ti}w_{tj}\,
     \Biggl[\Bigl(1-\tfrac{\epsilon_\infty}{2}\Bigr) q\!\left(\zeta_i,\bm{\beta},\gamma \right) 
     + \tfrac{\epsilon_\infty}{2} \Biggr]
     \Biggl[\Bigl(1-\tfrac{\epsilon_\infty}{2}\Bigr) q\!\left(\zeta_j,\bm{\beta},\gamma \right) 
     + \tfrac{\epsilon_\infty}{2} \Biggr]}_{\Delta_4}.
\end{aligned}$}
\] 

\[ 
\resizebox{1\hsize}{!}{$ \begin{aligned}
    \Delta_3 
&= \sum_{i<j}^t w_{ti}w_{tj}\,
   \mathbb{E}\!\Biggl[ a_i \Biggl(a_j 
   - \Bigl(\Bigl(1-\tfrac{\epsilon_\infty}{2}\Bigr) q\!\left(\zeta_j,\bm{\beta},\gamma \right) 
   + \tfrac{\epsilon_\infty}{2}\Bigr) \Biggr) \Biggr] \le \sum_{i<j}^t w_{ti}w_{tj}\,
   \Biggl[\mathbb{E}(a_j) 
   - \Bigl(\Bigl(1-\tfrac{\epsilon_\infty}{2}\Bigr) q\!\left(\zeta_j,\bm{\beta},\gamma \right) 
   + \tfrac{\epsilon_\infty}{2}\Bigr)\Biggr] \\
&\le \sum_{i=1}^t w_{ti} \sum_{j=1}^t w_{tj} \,
   \Biggl[\mathbb{E}(a_j) 
   - \Bigl(\Bigl(1-\tfrac{\epsilon_\infty}{2}\Bigr) q\!\left(\zeta_j,\bm{\beta},\gamma \right) + \tfrac{\epsilon_\infty}{2}\Bigr)\Biggr] \le L_z\sum_{j=1}^t w_{tj} \,
   \Biggl[\mathbb{E}(a_j) 
   - \Bigl(\Bigl(1-\tfrac{\epsilon_\infty}{2}\Bigr) q\!\left(\zeta_j,\bm{\beta},\gamma \right) + \tfrac{\epsilon_\infty}{2}\Bigr)\Biggr].
\end{aligned}$}
\] 
We have established in Step~1 that $\sum_{j=1}^t w_{tj} \,\left[\mathbb{E}(a_j) - \left(\left(1-\tfrac{\epsilon_\infty}{2}\right) q\!\left(\zeta_j,\bm{\beta},\gamma \right) + \tfrac{\epsilon_\infty}{2}\right)\right] \to 0$. Hence we get $\Delta_3 \to 0$ as $t \to \infty$. For $\Delta_4$, we have
\[ 
\resizebox{1\hsize}{!}{$ \begin{aligned}
    \Delta_4 &= \sum_{i<j}^t w_{ti}w_{tj}\,\Biggl[\Bigl(1-\tfrac{\epsilon_\infty}{2}\Bigr) q\!\left(\zeta_j,\bm{\beta},\gamma \right) 
     + \tfrac{\epsilon_\infty}{2} \Biggr] \Biggl[\mathbb{E}(a_i) 
   - \Bigl(\Bigl(1-\tfrac{\epsilon_\infty}{2}\Bigr) q\!\left(\zeta_i,\bm{\beta},\gamma \right) + \tfrac{\epsilon_\infty}{2}\Bigr)\Biggr] \\
   &\le \sum_{i<j}^t w_{ti}w_{tj}\, \Biggl[\mathbb{E}(a_i) 
   - \Bigl(\Bigl(1-\tfrac{\epsilon_\infty}{2}\Bigr) q\!\left(\zeta_i,\bm{\beta},\gamma \right) + \tfrac{\epsilon_\infty}{2}\Bigr)\Biggr] \\
   &\le \sum_{j=1}^t w_{tj}\, \sum_{i=1}^t w_{ti}\Biggl[\mathbb{E}(a_i) 
   - \Bigl(\Bigl(1-\tfrac{\epsilon_\infty}{2}\Bigr) q\!\left(\zeta_i,\bm{\beta},\gamma \right) + \tfrac{\epsilon_\infty}{2}\Bigr)\Biggr] \le L_z\sum_{i=1}^t w_{ti}\Biggl[\mathbb{E}(a_i) 
   - \Bigl(\Bigl(1-\tfrac{\epsilon_\infty}{2}\Bigr) q\!\left(\zeta_i,\bm{\beta},\gamma \right) + \tfrac{\epsilon_\infty}{2}\Bigr)\Biggr].
\end{aligned}$}
\] 
Following the same argument used to prove $\Delta_3 \to 0$, we have $\Delta_4 \to 0$. Therefore, we obtain 
\begin{equation*}
    \sum_{i<j}^t w_{ti}w_{tj}\,\mathbb{E}(a_i a_j) - \sum_{i<j}^t w_{ti}w_{tj} \Biggl[\Bigl(1-\tfrac{\epsilon_\infty}{2}\Bigr) q\!\left(\zeta_i,\bm{\beta},\gamma \right) + \tfrac{\epsilon_\infty}{2} \Biggr]\Biggl[\Bigl(1-\tfrac{\epsilon_\infty}{2}\Bigr) q\!\left(\zeta_j,\bm{\beta},\gamma \right) + \tfrac{\epsilon_\infty}{2} \Biggr] \to 0.
\end{equation*}
We now show that $2\sum_{i<j}^t w_{ti}w_{tj}\,\mathbb{E}(a_i a_j) \;\to\; \kappa_\infty^2$ as $ t \to \infty$.
Observe that
\begin{equation}
\label{equ5}
    \sum_{i=1}^t w_{ti}^2 \Biggl[\Bigl(1-\tfrac{\epsilon_\infty}{2}\Bigr) q\!\left(\zeta_i,\bm{\beta},\gamma \right) + \tfrac{\epsilon_\infty}{2} \Biggr]^2 \;\le\; \sum_{i=1}^t w_{ti}^2 \;\to\; 0.
\end{equation}
Moreover,
\[ 
\resizebox{1\hsize}{!}{$ \begin{aligned}
&\Biggl( \sum_{s=1}^t w_{ts} 
\Bigl[\Bigl(1-\tfrac{\epsilon_\infty}{2}\Bigr) q\!\left(\zeta_s,\bm{\beta},\gamma \right) + \tfrac{\epsilon_\infty}{2} \Bigr] \Biggr)^2 \\
=& \sum_{s=1}^t w_{ts}^2 \Bigl[\Bigl(1-\tfrac{\epsilon_\infty}{2}\Bigr) q\!\left(\zeta_s,\bm{\beta},\gamma \right) + \tfrac{\epsilon_\infty}{2} \Bigr]^2 + \sum_{i \neq j}^t w_{ti}w_{tj} \Bigl[\Bigl(1-\tfrac{\epsilon_\infty}{2}\Bigr) q\!\left(\zeta_i,\bm{\beta},\gamma \right) + \tfrac{\epsilon_\infty}{2} \Bigr] \Bigl[\Bigl(1-\tfrac{\epsilon_\infty}{2}\Bigr) q\!\left(\zeta_j,\bm{\beta},\gamma \right) + \tfrac{\epsilon_\infty}{2} \Bigr] \\
=& \sum_{s=1}^t w_{ts}^2 \Bigl[\Bigl(1-\tfrac{\epsilon_\infty}{2}\Bigr) q\!\left(\zeta_s,\bm{\beta},\gamma \right) + \tfrac{\epsilon_\infty}{2} \Bigr]^2 + 2\sum_{i < j}^t w_{ti}w_{tj} \Bigl[\Bigl(1-\tfrac{\epsilon_\infty}{2}\Bigr) q\!\left(\zeta_i,\bm{\beta},\gamma \right) + \tfrac{\epsilon_\infty}{2} \Bigr] \Bigl[\Bigl(1-\tfrac{\epsilon_\infty}{2}\Bigr) q\!\left(\zeta_j,\bm{\beta},\gamma \right) + \tfrac{\epsilon_\infty}{2} \Bigr].
\end{aligned}$}
\] 
Therefore,
$2\sum_{i<j}^t w_{ti}w_{tj}\,\mathbb{E}(a_i a_j) 
- \left( \sum_{s=1}^t w_{ts} 
\left[\left(1-\tfrac{\epsilon_\infty}{2}\right) q\!\left(\zeta_s,\bm{\beta},\gamma \right) 
+ \tfrac{\epsilon_\infty}{2} \right] \right)^2 \to 0$,
which implies that
\begin{equation}
\label{equ3}
    \lim_{t \to \infty}2\sum_{i<j}^t w_{ti}w_{tj}\,\mathbb{E}(a_i a_j) =\kappa_\infty^2.
\end{equation}

\paragraph{Step 3: Combine the terms and apply Markov's inequality.}
Combining \eqref{equ4} and \eqref{equ3}, we obtain from \eqref{equ1} that
$\mathbb{E}(\kappa_t-\kappa_\infty)^2 \;\to\; 
\kappa_\infty^2 - 2\kappa_\infty^2 + \kappa_\infty^2 = 0$.
By Markov's inequality, for any $\epsilon > 0$,
\[
\Pr(|\kappa_t-\kappa_\infty|>\epsilon) 
= \Pr\!\left((\kappa_t-\kappa_\infty)^2>\epsilon^2\right) 
\le \frac{\mathbb{E}(\kappa_t-\kappa_\infty)^2}{\epsilon^2} \;\to\; 0.
\]
Hence, we conclude that $\kappa_t \stackrel{p}{\rightarrow} \kappa_\infty$ as $t \to \infty$.

\textbf{Under the weaker conditions stated in Theorem~\ref{theor:kappat}, two modifications arise in the proof.}
The first concerns the leading term of \eqref{equ2}. Since
$\sup_t \sum_{s=1}^t w_{ts}^2 \le L_z^2 < \infty$ and $w_{ts}^2 \to 0$ for each fixed $s$,
applying the same reasoning as in Step~1 and repeated applications of Lemma~\ref{lemma:Toeplitz}, we obtain
\[
\sum_{s=1}^t w_{ts}^2 \,\mathbb{E}(a_s) 
- \sum_{s=1}^t w_{ts}^2 \!\left[\Bigl(1-\tfrac{\epsilon_\infty}{2}\Bigr)\mathbb{E}(a_s^\ast) 
+ \tfrac{\epsilon_\infty}{2}\right] \;\to\; 0.
\]
Since $\sum_{s=1}^{t-1} w_{ts}^2 \,\mathbb{E}(a_s^\ast) \to 0$ by assumption, 
it follows that the first term in \eqref{equ2} satisfies
$\sum_{s=1}^{t} w_{ts}^2 \,\mathbb{E}(a_s) \to 0$.

\noindent
The second modification appears in the proof of \eqref{equ5}. We note that
\[
\sum_{i=1}^t w_{ti}^2 
\left[\left(1-\tfrac{\epsilon_\infty}{2}\right) q\!\left(\zeta_i,\bm{\beta},\gamma \right) + \tfrac{\epsilon_\infty}{2} \right]^2 \;\le\; 
\sum_{i=1}^t w_{ti}^2 \left[\left(1-\tfrac{\epsilon_\infty}{2}\right) q\!\left(\zeta_i,\bm{\beta},\gamma \right) \tfrac{\epsilon_\infty}{2} \right].
\]
When $\epsilon_\infty=0$, the difference becomes
\begin{equation*}
\sum_{i=1}^t w_{ti}^2 \left[\left(1-\tfrac{\epsilon_\infty}{2}\right) q\!\left(\zeta_i,\bm{\beta},\gamma \right) 
+ \tfrac{\epsilon_\infty}{2} \right]
- \sum_{i=1}^t w_{ti}^2\,\mathbb{E}(a_i^\ast) =\sum_{i=1}^t w_{ti}^2 
\left[q\!\left(\zeta_i,\bm{\beta},\gamma \right) - q\!\left(\zeta_i,\bm{\beta},\gamma \right)\right] =0.
\end{equation*}
Since we have assumed $\sum_{i=1}^t w_{ti}^2\,\mathbb{E}(a_i^\ast) \to 0$, it follows that $\sum_{i=1}^t w_{ti}^2 \left[\left(1-\tfrac{\epsilon_\infty}{2}\right) q\!\left(\zeta_i,\bm{\beta},\gamma \right) + \tfrac{\epsilon_\infty}{2}\right]$ goes to 0.
Therefore, we conclude that 
$\sum_{i=1}^t w_{ti}^2 \left[\left(1-\tfrac{\epsilon_\infty}{2}\right) q\!\left(\zeta_i,\bm{\beta},\gamma \right) + \tfrac{\epsilon_\infty}{2}\right]^2 \to 0$,
and hence the proof remains valid under the weaker conditions.
\end{proof}

\subsection{Corollary~\ref{theor:gt} (Growing interference scale with equal weights)}
\label{app:gt}
\begin{proof}
Since $w_{ts} = \frac{\mathbb{I}(t-g(t) \le s \le t-1)}{g(t)}$, we have $\kappa_t = \frac{1}{g(t)} \sum_{s=t-g(t)}^{t-1} a_s$, 
and the optimal policy becomes  
\begin{equation*}
    a_t^\ast = \pi^\ast(\bm{x_t})
    = \mathbb{I}\!\left\{\phi(\bm{x}_t)^{\top} (\bm{\beta}_1 - \bm{\beta}_0) 
    + \Bigg[\sum_{s \in \mathcal{A}_t} \frac{1}{g(s)}\Bigg] \gamma \ge 0 \right\},
\end{equation*}
where $\mathcal{A}_t = \{\, s : s - g(s) \leq t \leq s - 1 \,\}$ denotes the set of future individuals for whom the $t$-th individual lies in their neighborhood.

\noindent
\textbf{We first verify the conditions in Theorem~\ref{theor:kappaast}.} We start from
\begin{equation}
\label{equ6}
\sum_{s=1}^{t-1} w_{ts}^2\text{Var}\,(a_s^\ast) \le \sum_{s=1}^{t-1} w_{ts}^2 = \sum_{s=t-g(t)}^{t-1} \frac{1}{g(t)^2} =\frac{1}{g(t)}\to 0, 
\end{equation}
as $t \to \infty$, since $g(t) \to \infty$. 
\textbf{We next prove the existence of $\kappa_\infty^\ast$ in Theorem~\ref{theor:kappaast}.}  
In particular, we derive the closed-form expression of $\lim_{t \to \infty}\sum_{s \in \mathcal{A}_t} \frac{1}{g(s)}$
for two special cases of the neighborhood size: linear growth and square-root growth.

\paragraph{Case 1.} $g(t) = \lfloor \rho t \rfloor$.

\noindent
The set of individuals receiving the influence from the $t$-th individual is
\[
\mathcal{A}_t =\left\{s: s- \lfloor \rho s \rfloor \le t \le s-1 \right\} = \left\{s: t+1 \le s \le \left\lfloor \frac{t}{1-\rho} \right\rfloor \right\},
\]
and then $\sum_{s\in \mathcal{A}_t} \frac{1}{g(s)}$ is given by $\sum_{s=t+1}^{\left\lfloor \frac{t}{1-\rho} \right\rfloor} \frac{1}{g(s)}$. For this summation, we can bound it by
\[
\int_{t+1}^{\left\lfloor \frac{t}{1-\rho} \right\rfloor+1}\frac{1} {\lfloor \rho s\rfloor} \, ds \le \sum_{s=t+1}^{\left\lfloor \frac{t}{1-\rho} \right\rfloor} \frac{1}{\lfloor \rho s \rfloor} \le \int_t^{\left\lfloor \frac{t}{1-\rho} \right\rfloor}\frac{1} {\lfloor \rho s\rfloor} \, ds.
\]
The lower bound can be formulated as
\[
\int_{t+1}^{\left\lfloor \frac{t}{1-\rho} \right\rfloor+1}\frac{1} {\lfloor \rho s\rfloor}\,ds \ge \int_{t+1}^{\frac{t}{1-\rho}} \frac{1}{\rho s}\,ds = \frac{1}{\rho} \ln{\rho s}\Big|_{t+1}^{\frac{t}{1-\rho}} = \frac{1}{\rho} \ln{\frac{\frac{\rho t}{1-\rho}}{\rho(t+1)}},
\]
then let $t \to \infty$ and by L'Hopital's rule, we have $\frac{1}{\rho} \lim_{t \to \infty}\ln{\frac{\frac{\rho t}{1-\rho}}{\rho(t+1)}} = \frac{1}{\rho} \ln{\frac{1}{1-\rho}}$. Similarly, the upper bound can be formulated as 
\[
\int_t^{\left\lfloor \frac{t}{1-\rho} \right\rfloor}\frac{1} {\lfloor \rho s\rfloor} \, ds \le \int_t^{\frac{t}{1-\rho}}\frac{1} {\rho s-1} \, ds = \frac{1}{\rho}\ln{(\rho s-1)}\Big|_{t}^{\frac{t}{1-\rho}} = \frac{1}{\rho}\ln{\frac{\frac{\rho t}{1-\rho}-1}{\rho t-1}},
\]
then let $t \to \infty$ and by L'Hopital's rule, we have $\frac{1}{\rho} \lim_{t \to \infty}\ln{\frac{\frac{\rho t}{1-\rho}-1}{\rho t-1}} = \frac{1}{\rho} \ln{\frac{1}{1-\rho}}$.
Finally, by Squeeze Theorem, when $g(t) = \lfloor \rho t \rfloor$, we have $\lim_{t \to \infty} \sum_{s \in \mathcal{A}_t} \frac{1}{g(s)} = \frac{1}{\rho} \ln{\frac{1}{1-\rho}}$.

\paragraph{Case 2.} $g(t) = \lfloor \rho \sqrt{t} \rfloor$.

\noindent
By the quadratic formula, the set of individuals that receive the influence of the $t$-th individual is
{\small
\begin{align*}
    \mathcal{A}_t =\left\{s: s-\lfloor \rho \sqrt{s} \rfloor \le t \le s-1 \right\} = \left\{s: t+1 \le s \le   \left\lfloor \frac{(\rho+\sqrt{\rho^2+4t})^2}{4}\right\rfloor \right\},
\end{align*}
}
then $\sum_{s\in \mathcal{A}_t} \frac{1}{g(s)} =\sum_{s=t+1}^{\left\lfloor \frac{(\rho+\sqrt{\rho^2 + 4t})^2}{4}\right\rfloor} \frac{1}{g(s)}$.\
Similarly, we obtain the lower and upper bounds:
\[
\int_{t+1}^{\left\lfloor \frac{(\rho+\sqrt{\rho^2 + 4t})^2}{4}\right\rfloor+1}\frac{1} {\lfloor \rho \sqrt{s}\rfloor}\,ds \le \sum_{s=t+1}^{\left\lfloor \frac{(\rho+\sqrt{\rho^2 + 4t})^2}{4}\right\rfloor} \frac{1}{\lfloor \rho \sqrt{s} \rfloor} \le \int_t^{\left\lfloor \frac{(\rho+\sqrt{\rho^2 + 4t})^2}{4}\right\rfloor}\frac{1} {\lfloor \rho \sqrt{s} \rfloor}\,ds.
\]
The lower bound can be formulated as
\begin{eqnarray*}
\int_{t+1}^{\left\lfloor \frac{(\rho+\sqrt{\rho^2 + 4t})^2}{4}\right\rfloor+1}\frac{1} {\lfloor \rho \sqrt{s}\rfloor}\,ds \ge \frac{2}{\rho}\sqrt{s} \Big|_{t+1}^{\frac{(\rho+\sqrt{\rho^2 + 4t})^2}{4}} = \frac{2}{\rho}\left(\frac{\rho+\sqrt{\rho^2+4t}}{2} - \sqrt{t+1} \right),
\end{eqnarray*}
then let $t \to \infty$, we have $\lim_{t \to \infty}\left(\frac{\sqrt{\rho^2+4t}}{2} - \sqrt{t+1} \right)=0$.
Finally, the limit is given by $\frac{2}{\rho} \lim_{t \to \infty}\left(\frac{\rho+\sqrt{\rho^2+4t}}{2} - \sqrt{t+1} \right)= 1$.
We next consider the upper bound,
\[
\int_t^{\left\lfloor \frac{(\rho+\sqrt{\rho^2 + 4t})^2}{4}\right\rfloor}\frac{1} {\lfloor \rho \sqrt{s} \rfloor}\,ds \le \int_t^{\frac{(\rho+\sqrt{\rho^2 + 4t})^2}{4}}\frac{1} {\rho \sqrt{s}-1}\,ds = 
\left( \frac{2}{\rho}\sqrt{s}+\frac{2}{\rho}\ln{(\rho\sqrt{s}-1)} \right) \Big|_{t}^{\frac{(\rho+\sqrt{\rho^2 + 4t})^2}{4}}.
\]
Following a similar approach as in deriving the lower bound, when $t \to \infty$, the first term becomes $\frac{2}{\rho}\lim_{t \to \infty}\left(\frac{\rho+\sqrt{\rho^2+4t}}{2} - \sqrt{t} \right) =1$.
For the second term, by L'Hopital's rule, we have,
$\frac{2}{\rho}\lim_{t \to \infty} \ln{\left(\frac{\rho(\rho+\sqrt{\rho^2+4t})/2-1}{\rho \sqrt{t}-1} \right)} = 0$,
thus the limit of upper bound is also 1. By Squeeze Theorem, we have $\lim_{t \to \infty} \sum_{s \in \mathcal{A}_t} \frac{1}{g(s)}=1$. 
Therefore, in both cases we can show the existence of $\kappa_\infty^\ast$:
\[ 
\resizebox{1\hsize}{!}{$ \begin{aligned}
    \kappa_\infty^\ast 
    &= \lim_{t \to \infty} \sum_{s=t-g(t)}^{t-1} \frac{1}{g(t)}\, 
    \Pr\!\left(\phi(\bm{x})^{\top} (\bm{\beta}_1-\bm{\beta}_0) + \sum_{s \in \mathcal{A}_t} \frac{1}{g(s)} \gamma \ge 0\right) \\
    &= \lim_{t \to \infty} \frac{1}{g(t)} \sum_{s=t-g(t)}^{t-1}\Pr\!\left(\phi(\bm{x})^{\top} (\bm{\beta}_1-\bm{\beta}_0) + \sum_{s \in \mathcal{A}_t} \frac{1}{g(s)} \gamma \ge 0\right) = \Pr\!\left(\phi(\bm{x})^{\top} (\bm{\beta}_1-\bm{\beta}_0) + \lim_{t \to \infty} \sum_{s \in \mathcal{A}_t} \frac{1}{g(s)} \gamma \ge 0\right).
\end{aligned}$}
\]

\noindent
\textbf{We next verify the conditions (ii) and (iii) in Theorem~\ref{theor:kappat}.}
\noindent
(ii) The argument follows directly from the proof of \eqref{equ6}.
\noindent
(iii) Recall that $w_{ts} = \frac{\mathbb{I}(t-g(t) \le s \le t-1)}{g(t)}$.
In both cases of $g(t)$, we have $t-g(t) \to \infty$ as $t \to \infty$. 
Hence, for sufficiently large $t$, it holds that $t-g(t) > s$ for any fixed $s$, which implies $w_{ts} \to 0$. The proof is hence completed. 
\end{proof}

\subsection{Theorem~\ref{theor:infer} (Asymptotic normality for the online estimator)}
\label{app:infer}

\begin{proof}
    Based on our online estimator, we have
\begin{equation*}
    \sqrt{t}(\widehat{\bm{\theta}}_{t} - \bm{\theta}) = \left(\frac{1}{t} \sum_{s=1}^t \bm{z}_s \bm{z}_s^{\top}\right)^{-1} \left( \frac{1}{\sqrt{t}} \sum_{s=1}^t \bm{z}_s e_s \right), \quad i=0,1.
\end{equation*}

\paragraph{Step 1:} First we will show that
$$\frac{1}{\sqrt{t}} \sum_{s=1}^t \bm{z}_s e_s \stackrel{D}{\to}\mathcal{N}_d(0,G).$$
Using Cramer-Wold device, it suffices to show for any $\bm{v} \in \mathbb{R}^d$, $$\frac{1}{\sqrt{t}} \sum_{s=1}^t \bm{v}^{\top}\bm{z}_s e_s \stackrel{D}{\to}\mathcal{N}_d(0,\bm{v}^{\top} G \bm{v}).$$
For $1 \le j\le t$ and $t \ge 1$, define $\mathcal{H}_{tj} = \mathcal{H}_j$, and $M_{tj} = \frac{1}{\sqrt{t}} \sum_{s=1}^t \bm{v}^{\top} \bm{z}_s e_s$.
Note that 
\begin{align*}
&\mathbb{E}(\bm{v}^{\top} \bm{z}_s e_s \vert \mathcal{H}_{t,s-1}) = \mathbb{E}(\mathbb{E}(\bm{v}^{\top} \bm{z}_s e_s\vert \mathcal{H}_{t,s-1},a_s) \vert \mathcal{H}_{t,s-1})= \mathbb{E}(\bm{v}^{\top} \bm{z}_s \vert \mathcal{H}_{t,s-1}) \mathbb{E}(e_s\vert \mathcal{H}_{t,s-1},a_s)=0,
\end{align*}
indicating that $\left\{M_{tj},\mathcal{H}_{tj}, 1\le j\le t, t\ge 1 \right\}$ is a martingale array. 

\noindent
We will now prove the convergence of $M_{tt}$ using Martingale Central Limit Theorem (see Theorem 3.2 in \cite{hall2014martingale}). The proof proceeds in two steps: firstly we will verify the conditional Lindeberg condition is satisfied and secondly we will find the limit of conditional variance. \\
\textbf{(a) Check the conditional Lindeberg condition.} For $\forall \delta >0$,
\[ 
\resizebox{1\hsize}{!}{$ \begin{aligned}
    &\sum_{s=1}^t \mathbb{E} \left[ \frac{1}{t} \left(\bm{v}^{\top} \bm{z}_s e_s \right)^2 \mathbb{I} \left\{\bm{v}^{\top} \bm{z}_s e_s > \delta \sqrt{t} \right\} \Big| \mathcal{H}_{t,s-1} \right]  
     \le   \frac{\Vert \bm{v} \Vert_2^2 L_z^2d} {t} \sum_{s=1}^t \mathbb{E} \left[ e_s^2 \, \mathbb{I} \left\{e_s^2 > \frac{\delta^2 t}{\|\bm{v}\|^2 L_z^2 d} \right\} \Big| \mathcal{H}_{s-1} \right], 
\end{aligned}$}
\] 
where $e_s$ conditioned on $a_s$ are i.i.d distributed, so the right hand side comes to be 
$$\Vert \bm{v} \Vert_2^2 L_z^2d \mathbb{E} \left[e^2 \mathbb{I} \left\{e^2 > \frac{\delta^2 t}{\|\bm{v}\|^2 L_z^2 d} \right\} \right],$$
where $e$ is a random variable defined by $e_s \vert \mathcal{H}_{s-1}$. Since $e^2 \mathbb{I} \left\{e^2 > \frac{\delta^2 t}{\|\bm{v}\|^2 L_z^2 d} \right\}$ is bounded by $e^2$ with $\mathbb{E}(e^2) < \infty$ and convergences to 0 almost surely as $t \to \infty$. Therefore, by Dominated Convergence Theorem, we get
\[
\sum_{s=1}^t \mathbb{E} \left[ \frac{1}{t} \left(\bm{v}^{\top} \bm{z}_s e_s \right)^2 \mathbb{I} \left\{\bm{v}^{\top} \bm{z}_s e_s > \delta \sqrt{t} \right\} \Big| \mathcal{H}_{t,s-1} \right] \to 0, \text{ as } t\to \infty.
\]
\textbf{(b) Derive the limit of the conditional variance.} The conditional variance is given by
\begin{align*}
    \widehat{\eta}_t^2 &= \sum_{s=1}^t \mathbb{E}\left\{\frac{1}{t} (\bm{v}^{\top} \bm{z}_s e_s)^2 \Big| \mathcal{H}_{t,s-1} \right\}  = \frac{1}{t} \sum_{s=1}^t \mathbb{E} \left\{\left( \bm{v}^{\top} \bm{z}_s \right)^2 \mathbb{E} \left[ e_s^2 \mid a_s , \bm{x}_s \right] \mid \mathcal{H}_{s-1} \right\}.
\end{align*}
Given $\mathcal{H}_{s-1}$, $\kappa_s$ is known, so the expectation is take with respect to $\bm{x}_s$ and $a_s$. As for the random noise, given $a_s$, $e_s$ is independent of $\mathcal{H}_{s-1}$ and $\bm{x}_s$, and $\mathbb{E}(e_s^2|a_s=i,\bm{x}_s) = \sigma_i^2$. Thus we have 
\begin{align}
   & \widehat{\eta}_t^2 = \frac{1}{t} \sum_{s=1}^t \mathbb{E} \left\{\left( \bm{v}^{\top} \bm{z}_s \right)^2 \sigma_{a_s}^2 \mid \mathcal{H}_{s-1} \right\} \nonumber =\frac{1}{t} \sum_{s=1}^t \mathbb{E} \left\{\mathbb{E} \left\{\left( \bm{v}^{\top} \bm{z}_s \right)^2 \sigma_{a_s}^2 \mid \bm{x}_s,\mathcal{H}_{s-1} \right\} \mid \mathcal{H}_{s-1}\right\} \nonumber\\
     = &\frac{1}{t} \sum_{s=1}^t \mathbb{E} \left\{\left( \bm{v}^{\top} \bm{z}_s^{(0)} \right)^2 \sigma_{0}^2 \Pr(a_s=0\mid \bm{x}_s,\mathcal{H}_{s-1}) + \left( \bm{v}^{\top} \bm{z}_s^{(1)} \right)^2 \sigma_{1}^2 \Pr(a_s=1\mid \bm{x}_s,\mathcal{H}_{s-1}) \mid \mathcal{H}_{s-1}\right\}, \label{equ14}
\end{align}
where $\bm{z}_s^{(0)} = (\phi(\bm{x}_s)^{\top}, \bm{0}^{\top}, \kappa_s)^{\top}$ and $\bm{z}_s^{(1)} = (\bm{0}^{\top}, \phi(\bm{x}_s)^{\top}, \kappa_s)^{\top}$, representing the forms of $\bm{z}_s$ corresponding to $a_s = 0$ and $a_s = 1$, respectively.
Let $\bm{v} = (\bm{v}^{\top}_1,\bm{v}^{\top}_2,v_3)^{\top}$ to align with $\bm{z}_s^{(0)}$ and $\bm{z}_s^{(1)}$, where $\bm{v}_1,\bm{v}_2 \in \mathbb{R}^{d_1}$ and $v_3 \in \mathbb{R}$. Then we have 
\begin{align*}
   \left( \bm{v}^{\top} \bm{z}_s^{(0)} \right)^2 = \left(\bm{v}_1^{\top} \phi(\bm{x}_s) + v_3\kappa_s \right)^2,  \quad  \left( \bm{v}^{\top} \bm{z}_s^{(1)} \right)^2 = \left(\bm{v}_2^{\top} \phi(\bm{x}_s) + v_3\kappa_s \right)^2.
\end{align*}
Thus \eqref{equ14} can be written as
\begin{align*}
    \widehat{\eta}_t^2 =& \frac{\sigma_0^2 }{t} \sum_{s=1}^t \mathbb{E}\left\{\left(\bm{v}_1^{\top} \phi(\bm{x}_s) \phi(\bm{x}_s)^{\top} \bm{v}_1 + v_3^2\kappa_s^2+2\bm{v}_1^{\top}\phi(\bm{x}_s) v_3\kappa_s\right) \Pr(a_s=0\mid \bm{x}_s,\mathcal{H}_{s-1})\Big| \mathcal{H}_{s-1}\right\} \\
    &+ \frac{\sigma_1^2 }{t} \sum_{s=1}^t \mathbb{E}\left\{\left(\bm{v}_2^{\top} \phi(\bm{x}_s) \phi(\bm{x}_s)^{\top} \bm{v}_2 + v_3^2\kappa_s^2+2\bm{v}_2^{\top}\phi(\bm{x}_s) v_3\kappa_s\right) \Pr(a_s=1\mid \bm{x}_s,\mathcal{H}_{s-1}) \Big| \mathcal{H}_{s-1}\right\}.
\end{align*}
Denote the difference between estimated parameter, $\widehat{\bm{\beta}}_{1,s-1} - \widehat{\bm{\beta}}_{0,s-1}$, as $\widehat{\bm{\beta}}_{s-1}$. We have defined $\zeta_t = \sum_{s=t+1}^\infty w_{st}$ in Proposition~\ref{prop:optimal}. Then the second expectation term in the above equation can be expressed as a continuous function $p_1$ of $\widehat{\bm{\beta}}_{s-1}$, $\widehat{\gamma}_{s-1}$, $\kappa_s$, $\zeta_s$ and $\epsilon_s$,
\begin{align*}
    p_1(\widehat{\bm{\beta}}_{s-1},\widehat{\gamma}_{s-1},\kappa_s,\zeta_s,\epsilon_s) &= \frac{\epsilon_s}{2}\int (\bm{v}_2^{\top} \phi(\bm{x}_s) \phi(\bm{x}_s)^{\top} \bm{v}_2 + v_3^2\kappa_s^2+2\bm{v}_2^{\top}\phi(\bm{x}_s) v_3\kappa_s) \, d\mathcal{P}_X \\
    &+ (1-\epsilon_s) \int_{\mathcal{X}_{1s}} (\bm{v}_2^{\top} \phi(\bm{x}_s) \phi(\bm{x}_s)^{\top} \bm{v}_2 + v_3^2\kappa_s^2+2\bm{v}_2^{\top}\phi(\bm{x}_s)v_3\kappa_s) \,d\mathcal{P}_X,
\end{align*}
where $\mathcal{X}_{1s} = \left\{\bm{x}:  \phi(\bm{x})^{\top} \widehat{\bm{\beta}}_{s-1}+ \zeta_s \widehat{\gamma}_{s-1} \geq 0 \right\}$
and $\mathcal{X}_{0s} = \mathbb{R}^{d_0} \backslash \mathcal{X}_{1s}$. 
We can rewrite $p_1$ as
\[ 
\resizebox{1\hsize}{!}{$ \begin{aligned}
    &p_1(\widehat{\bm{\beta}}_{s-1},\widehat{\gamma}_{s-1},\kappa_s,\zeta_s,\epsilon_s) =\frac{\epsilon_s}{2}\int (\bm{v}_2^{\top} \phi(\bm{x}_s)  \phi(\bm{x}_s)^{\top} \bm{v}_2 + v_3^2\kappa_s^2+2\bm{v}_2^{\top}\phi(\bm{x}_s)v_3\kappa_s) \, d\mathcal{P}_X \\
    &~~~~+ (1-\epsilon_s) \int\mathbb{I}\left\{\phi(\bm{x}_s)^{\top} \widehat{\bm{\beta}}_{s-1} + \zeta_s \widehat{\gamma}_{s-1} \geq 0\right\}\left(\bm{v}_2^{\top} \phi(\bm{x}_s) \phi(\bm{x}_s)^{\top} \bm{v}_2 + v_3^2\kappa_s^2+2\bm{v}_2^{\top}\phi(\bm{x}_s)v_3\kappa_s\right) \,d\mathcal{P}_X,
\end{aligned}$}
\] 
and similarly, 
\[ 
\resizebox{1\hsize}{!}{$ \begin{aligned}
    &p_0(\widehat{\bm{\beta}}_{s-1},\widehat{\gamma}_{s-1},\kappa_s,\zeta_s,\epsilon_s)
    = \frac{\epsilon_s}{2}\int (\bm{v}_1^{\top} \phi(\bm{x}_s) \phi(\bm{x}_s)^{\top} \bm{v}_1 + v_3^2\kappa_s^2+2\bm{v}_1^{\top}\phi(\bm{x}_s)v_3\kappa_s) \, d\mathcal{P}_X \\
    &~~~~+ (1-\epsilon_s) \int\mathbb{I}\left\{ \phi(\bm{x}_s)^{\top} \widehat{\bm{\beta}}_{s-1} + \zeta_s \widehat{\gamma}_{s-1} < 0\right\} \,(\bm{v}_1^{\top} \phi(\bm{x}_s) \phi(\bm{x}_s)^{\top} \bm{v}_1 + v_3^2\kappa_s^2+2\bm{v}_1^{\top}\phi(\bm{x}_s)v_3\kappa_s) \,d\mathcal{P}_X.
\end{aligned}$}
\] 
Then by Corollary~\ref{coro:consistency}, Theorem~\ref{theor:kappat} and Continuous Mapping Theorem, 
\begin{equation*}
    p_i(\widehat{\bm{\beta}}_{s-1},\widehat{\gamma}_{s-1},\kappa_s,\zeta_s,\epsilon_s) \to p_i(\bm{\beta},\gamma,\kappa_\infty,\zeta_\infty,\epsilon_\infty),
\end{equation*}
as $s\to\infty$, where $\bm{\beta} \equiv \bm{\beta}_1 - \bm{\beta}_0$. Moreover, $p_i(\widehat{\bm{\beta}}_{s-1},\widehat{\gamma},\kappa_s,\zeta_s,\epsilon_s)$ is bounded by $\bm{v}^{\top} \mathbb{E}(\bm{z}_s \bm{z}_s^{\top})\bm{v} \le d\Vert \bm{v}\Vert_2^2 L_z^2$, then by Lemma~\ref{lemma:weighted-convergence}, we have 
\begin{align*}
    \widehat{\eta}_t^2 \stackrel{p}{\to} \sigma_0^2 p_0(\bm{\beta},\gamma,\kappa_\infty,\zeta_\infty,\epsilon_\infty) + \sigma_1^2 p_1(\bm{\beta},\gamma,\kappa_\infty,\zeta_\infty,\epsilon_\infty),
\end{align*}
where 
\begin{align*}
p_0(\bm{\beta},\gamma,\kappa_\infty,\zeta_\infty,\epsilon_\infty) =& \frac{\epsilon_\infty}{2}\int (\bm{v}_1^{\top} \phi(\bm{x}_s) \phi(\bm{x}_s)^{\top} \bm{v}_1 + v_3^2\kappa_\infty^2+2\bm{v}_1^{\top}\phi(\bm{x}_s)v_3\kappa_\infty) \, d\mathcal{P}_X \\
    &+ (1-\epsilon_\infty) \int_{\mathcal{X}_0} \,(\bm{v}_1^{\top} \phi(\bm{x}_s) \phi(\bm{x}_s)^{\top} \bm{v}_1 + v_3^2\kappa_\infty^2+2\bm{v}_1^{\top}\phi(\bm{x}_s)v_3\kappa_\infty) \,d\mathcal{P}_X, \\
p_1(\bm{\beta},\gamma,\kappa_\infty,\zeta_\infty,\epsilon_\infty) =& \frac{\epsilon_\infty}{2}\int (\bm{v}_2^{\top} \phi(\bm{x}_s) \phi(\bm{x}_s)^{\top} \bm{v}_2 + v_3^2\kappa_\infty^2+2\bm{v}_2^{\top}\phi(\bm{x}_s)v_3\kappa_\infty) \, d\mathcal{P}_X \\
    &+ (1-\epsilon_\infty) \int_{\mathcal{X}_1} \,(\bm{v}_2^{\top} \phi(\bm{x}_s) \phi(\bm{x}_s)^{\top} \bm{v}_2 + v_3^2\kappa_\infty^2+2\bm{v}_2^{\top}\phi(\bm{x}_s)v_3\kappa_\infty) \,d\mathcal{P}_X.
\end{align*}
with $\mathcal{X}_1 = \left\{\bm{x}: \phi(\bm{x})^{\top} (\bm{\beta}_1 - \bm{\beta}_0) + \zeta_\infty \gamma \geq 0 \right\}$ and $\mathcal{X}_0 = \left\{\bm{x}: \phi(\bm{x})^{\top} (\bm{\beta}_1 - \bm{\beta}_0) + \zeta_\infty \gamma < 0 \right\}$. 
To move $\bm{v}$ out of the integral, we can rewrite the above equation as $\widehat{\eta}_t^2 \stackrel{p}{\to} \bm{v}^\top (\sigma_0^2 G_0+\sigma_1^2 G_1) \bm{v}$,
where $G_0$ is 
\[ 
\resizebox{1\hsize}{!}{$ \begin{pmatrix}
        \frac{\epsilon_{\infty}}{2} \int \phi(\bm{x})\phi(\bm{x})^{\top} \,d\mathcal{P}_X+(1-\epsilon_{\infty})\int_{\mathcal{X}_0} \phi(\bm{x})\phi(\bm{x})^{\top} \,d\mathcal{P}_X  & \bm{0} & \kappa_{\infty}\left[ \frac{\epsilon_{\infty}}{2} \int \phi(\bm{x})^{\top}\,d\mathcal{P}_X+(1-\epsilon_{\infty})\int_{\mathcal{X}_0} \phi(\bm{x})^{\top} \,d\mathcal{P}_X\right]\\
        \bm{0} & \bm{0} & \bm{0} \\
        \kappa_{\infty}\left[ \frac{\epsilon_{\infty}}{2} \int \phi(\bm{x})\,d\mathcal{P}_X+(1-\epsilon_{\infty})\int_{\mathcal{X}_0} \phi(\bm{x}) \,d\mathcal{P}_X\right] & \bm{0} & \kappa_{\infty}^2 \left[\frac{\epsilon_{\infty}}{2} +(1-\frac{\epsilon_{\infty}}{2}) \Pr(\bm{x} \in \mathcal{X}_0) \right]
    \end{pmatrix},$}
    \] 
and $G_1$ is 
\[  
\resizebox{1\hsize}{!}{$ \begin{pmatrix}
        \bm{0}  & \bm{0} & \bm{0}\\
        \bm{0} & \frac{\epsilon_{\infty}}{2} \int \phi(\bm{x})\phi(\bm{x})^{\top} \,d\mathcal{P}_X+(1-\epsilon_{\infty})\int_{\mathcal{X}_1} \phi(\bm{x})\phi(\bm{x})^{\top} \,d\mathcal{P}_X & \kappa_{\infty}\left[ \frac{\epsilon_{\infty}}{2} \int \phi(\bm{x})^{\top}\,d\mathcal{P}_X+(1-\epsilon_{\infty})\int_{\mathcal{X}_1} \phi(\bm{x})^{\top} \,d\mathcal{P}_X\right] \\
        \bm{0} & \kappa_{\infty}\left[ \frac{\epsilon_{\infty}}{2} \int \phi(\bm{x})\,d\mathcal{P}_X+(1-\epsilon_{\infty})\int_{\mathcal{X}_1} \phi(\bm{x}) \,d\mathcal{P}_X\right] & \kappa_{\infty}^2 \left[\frac{\epsilon_{\infty}}{2} +(1-\frac{\epsilon_{\infty}}{2}) \Pr(\bm{x} \in \mathcal{X}_1) \right]
    \end{pmatrix}.$}
    \] 
Finally, by Martingale Central Limit Theorem, we have $\frac{1}{\sqrt{t}}\sum_{s=1}^t \bm{z}_s e_s \stackrel{D}{\to} \mathcal{N}_d(\bm{0},G)$,
with $G = \sigma_0^2 G_0+\sigma_1^2 G_1$.

\paragraph{Step 2:} Secondly, we will show the limit of the second moment term. By Lemma~\ref{lemma:matrix-convergence}, for $\forall \bm{v} = (\bm{v}_1^{\top},\bm{v}_2^{\top},v_3)^{\top} \in \mathbb{R}^d$, we need to find the limit of 
$
\widehat{\xi}_t = (1/t)\sum_{s=1}^t \bm{v}^{\top} \bm{z}_s \bm{z}_s^{\top} \bm{v}. $ 
Let $\bm{x} \sim \mathcal{P}_X$ and define $\tilde{\bm{z}}$ as
\begin{equation}
\label{equ15}
    \tilde{\bm{z}} = (\phi(\bm{x})^{\top}, \phi(\bm{x})^{\top}, L_z).
\end{equation}
Then for any $h >0$ and each $s$, we have 
\begin{equation*}
    \Pr(\bm{v}^{\top} \bm{z}_s \bm{z}_s^{\top} \bm{v}>h)
    =\Pr(\Vert \bm{v}^{\top} \bm{z}_s\Vert_2^2 >h) 
    \le \Pr(\Vert \bm{v}\Vert_2^2 \Vert \bm{z}_s \Vert_2^2 >h)
    \le \Pr(\Vert \bm{v}\Vert_2^2 \Vert \tilde{\bm{z}} \Vert_2^2 >h),
\end{equation*}
and $\mathbb{E}(\Vert \bm{v}\Vert_2^2 \Vert \tilde{\bm{z}} \Vert_2^2) \le \Vert \bm{v}\Vert_2^2 L_z^2 < \infty$. By Theorem 2.19 in \cite{hall2014martingale}, we have
\begin{align*}
    &\frac{1}{t}\sum_{s=1}^t \left\{\bm{v}^{\top} \bm{z}_s \bm{z}_s^{\top} \bm{v} - \mathbb{E}(\bm{v}^{\top} \bm{z}_s \bm{z}_s^{\top} \bm{v} \mid \mathcal{H}_{s-1})\right\} \\
    =& \widehat{\xi}_t - \frac{1}{t}\sum_{s=1}^t\left[p_0(\widehat{\bm{\beta}}_{s-1},\widehat{\gamma}_{s-1},\kappa_s,\zeta_s,\epsilon_s)+ p_0(\widehat{\bm{\beta}}_{s-1},\widehat{\gamma}_{s-1},\kappa_s,\zeta_s,\epsilon_s) \right]\stackrel{p}{\to}0.
\end{align*}
According to the results of the first part, we have 
$\widehat{\xi}_t \stackrel{p}{\to} p_0(\bm{\beta},\gamma,\kappa_\infty,\zeta_\infty,\epsilon_\infty) + p_1(\bm{\beta},\gamma,\kappa_\infty,\zeta_\infty,\epsilon_\infty)$.
Using Lemma~\ref{lemma:matrix-convergence} and Continuous Mapping Theorem, we derive 
\begin{equation}
\label{equ16}
    \left(\frac{1}{t} \sum_{s=1}^t \bm{z}_s \bm{z}_s^{\top} \right)^{-1} \stackrel{p}{\to} (G_0+G_1)^{-1}.
\end{equation}

\paragraph{Step 3:} Combining results from the first two parts and using Slutsky's Theorem, we have
\[
\sqrt{t}(\widehat{\bm{\theta}}_{t} - \bm{\theta}) \stackrel{D}{\to}  (G_0+G_1)^{-1} \mathcal{N}_d(\bm{0},\sigma_0^2 G_0+\sigma_1^2G_1) (G_0+G_1)^{-1} = \mathcal{N}_d(\bm{0},S),
\]
with $S = \left(G_0+G_1\right)^{-1} \left(\sigma_0^2 G_0+\sigma_1^2G_1 \right)\left(G_0+G_1\right)^{-1}=\left(G_0+G_1\right)^{-1} G\left(G_0+G_1\right)^{-1}$.

\paragraph{Step 4:} Finally, we will show the consistency of variance estimator
\[
\left( \frac{1}{t} \sum_{s=1}^t \bm{z}_s \bm{z}_s^{\top} \right)^{-1} \left(\frac{1}{t}\sum_{s=1}^t \bm{z}_s \bm{z}_s^{\top} \widehat{e}_s^2 \right) \left( \frac{1}{t} \sum_{s=1}^t \bm{z}_s \bm{z}_s^{\top} \right)^{-1}.
\]
Based on~\eqref{equ16}, it is sufficient to show 
\[
\widehat{G} := \frac{1}{t} \sum_{s=1}^t \bm{z}_s \bm{z}_s^{\top} \widehat{e}_s^2 = \frac{1}{t} \sum_{s=1}^t \bm{z}_s \bm{z}_s^{\top} \left\{\bm{z}_s^\top(\bm{\theta}-\widehat{\bm{\theta}_t})+e_s \right\}^2
\stackrel{p}{\to} G.
\]
By Lemma~\ref{lemma:matrix-convergence}, it suffices to show for any $\bm{v} \in \mathbb{R}^d$,
\[
\bm{v}^\top \widehat{G}\bm{v} = \frac{1}{t} \sum_{s=1}^t \bm{v}^\top\bm{z}_s \bm{z}_s^{\top} \bm{v} \widehat{e}_s^2 = \frac{1}{t} \sum_{s=1}^t \bm{v}^\top \bm{z}_s \bm{z}_s^{\top} \bm{v}\left\{\bm{z}_s^\top(\bm{\theta}-\widehat{\bm{\theta}_t})+e_s \right\}^2
\stackrel{p}{\to} \bm{v}^\top G \bm{v}.
\]
We will expand the quadratic term and analyze each of the three resulting components,
\[
\bm{v}^\top \widehat{G}\bm{v}= \frac{1}{t} \sum_{s=1}^t \bm{v}^\top\bm{z}_s \bm{z}_s^{\top} \bm{v} \left[ \left\{\bm{z}_s^{\top}(\bm{\theta}-\widehat{\bm{\theta}}_{t}) \right\}^2 + 2\bm{z}_s^{\top}(\bm{\theta}-\widehat{\bm{\theta}}_{t})e_s+e_s^2 \right].
\]
By Corollary~\ref{coro:consistency}, the first term 
\begin{align*}
    (\bm{\theta}-\widehat{\bm{\theta}}_{t})^\top \frac{1}{t} \sum_{s=1}^t (\bm{v}^\top\bm{z}_s \bm{z}_s^{\top} \bm{v}) \bm{z}_s\bm{z}_s^\top (\bm{\theta}-\widehat{\bm{\theta}}_{t}) \le L_z^4 \Vert\bm{v} \Vert_2^2\Vert \bm{\theta} - \widehat{\bm{\theta}}_{t} \Vert_2^2 \stackrel{p}{\to}0.
\end{align*}
We now turn to the second term
\[ 
\resizebox{1\hsize}{!}{$ \begin{aligned}
    & 2(\bm{\theta}-\widehat{\bm{\theta}}_{t})^\top \frac{1}{t} \sum_{s=1}^t (\bm{v}^\top \bm{z}_s \bm{z}_s^{\top} \bm{v}) \bm{z}_s e_s \\
    = & 2(\bm{\theta}-\widehat{\bm{\theta}}_{t})^\top \left[ \frac{1}{t} \sum_{s=1}^t (\bm{v}^\top \bm{z}_s^{(0)} \bm{z}_s^{{(0)}\top} \bm{v}) \mathbb{I} \left\{a_s=0 \right\} \bm{z}_s^{(0)} e_s^{(0)}+ \frac{1}{t} \sum_{s=1}^t (\bm{v}^\top \bm{z}_s^{(1)} \bm{z}_s^{{(1)}\top} \bm{v}) \mathbb{I} \left\{a_s=1 \right\} \bm{z}_s^{(1)} e_s^{(1)}\right],
\end{aligned}$}
\] 
where \(e_s^{(i)}\) represents i.i.d. random noise from \(\mathcal{P}_{ei}\). 
By Corollary~\ref{coro:consistency}, \(\widehat{\bm{\theta}}_{t} \stackrel{p}{\to} \bm{\theta} \), the estimation error converges to zero in probability. Furthermore, by Lemma~\ref{lemma:subgaussian}, we have 
\[
\frac{1}{t} \sum_{s=1}^t (\bm{v}^\top \bm{z}_s^{(0)} \bm{z}_s^{{(0)}\top} \bm{v}) \mathbb{I} \left\{a_s=0 \right\} \bm{z}_s^{(0)} e_s^{(0)} \stackrel{p}{\to} \bm{0};\quad 
\frac{1}{t} \sum_{s=1}^t (\bm{v}^\top \bm{z}_s^{(1)} \bm{z}_s^{{(1)}\top} \bm{v}) \mathbb{I} \left\{a_s=0 \right\} \bm{z}_s^{(1)} e_s^{(1)} \stackrel{p}{\to} \bm{0}.
\]
Thus, the second term $2(\bm{\theta}-\widehat{\bm{\theta}}_{t})^\top \frac{1}{t} \sum_{s=1}^t (\bm{v}^\top \bm{z}_s \bm{z}_s^{\top} \bm{v}) \bm{z}_s e_s \stackrel{p}{\to} 0$.

\noindent
Finally, we show that the last term $\frac{1}{t} \sum_{s=1}^t \bm{v}^\top\bm{z}_s \bm{z}_s^{\top} \bm{v} e_s^2 \stackrel{p}{\to} \bm{v}^\top G \bm{v}$.
As shown in Step 1(b),
\[
\mathbb{E} \left\{ \bm{v}^\top\bm{z}_s \bm{z}_s^{\top} \bm{v} e_s^2 \mid \mathcal{H}_{s-1}\right\}=\sigma_0^2 \, p_0(\widehat{\bm{\beta}}_{s-1},\widehat{\gamma}_{s-1},\kappa_s,\zeta_s,\epsilon_s) + \sigma_1^2 p_1(\widehat{\bm{\beta}}_{s-1},\widehat{\gamma}_{s-1},\kappa_s,\zeta_s,\epsilon_s).
\]
Then, by the Continuous Mapping Theorem and Corollary~\ref{coro:consistency}, we obtain,
\[
\frac{1}{t}\sum_{s=1}^t \mathbb{E} \left\{ \bm{v}^\top\bm{z}_s \bm{z}_s^{\top} \bm{v} e_s^2 \mid \mathcal{H}_{s-1}\right\} \to \sigma_0^2 \, p_0(\bm{\beta},\gamma,\kappa_\infty,\zeta_\infty,\epsilon_\infty) + \sigma_1^2 \, p_1(\bm{\beta},\gamma,\kappa_\infty,\zeta_\infty,\epsilon_\infty).
\]
We use the same definition of $\tilde{\bm{z}}$ as  in~\eqref{equ15} and apply a similar argument as in Step 2.
\[
\Pr(\bm{v}^{\top} \bm{z}_s \bm{z}_s^{\top} \bm{v} e_s^2>h)
    =\Pr(\Vert \bm{v}^{\top} \bm{z}_s\Vert_2^2 e_s^2 >h) 
    \le \Pr(\Vert \bm{v}\Vert_2^2 \Vert \bm{z}_s \Vert_2^2 e_s^2>h)
    \le \Pr(\Vert \bm{v}\Vert_2^2 \Vert \tilde{\bm{z}} \Vert_2^2 e_s^2>h),
\]
and $\mathbb{E}(\Vert \bm{v}\Vert_2^2 \Vert \tilde{\bm{z}} \Vert_2^2 e_s^2) \le \Vert \bm{v}\Vert_2^2 L_z^2 \sigma^2< \infty$. By Theorem 2.19 in \cite{hall2014martingale}, we have
\begin{align*}
    &\frac{1}{t}\sum_{s=1}^t \left\{\bm{v}^{\top} \bm{z}_s \bm{z}_s^{\top} \bm{v} e_s^2- \mathbb{E}(\bm{v}^{\top} \bm{z}_s \bm{z}_s^{\top} \bm{v} e_s^2\mid \mathcal{H}_{s-1})\right\} \\
    =& \frac{1}{t} \sum_{s=1}^t \bm{v}^\top\bm{z}_s \bm{z}_s^{\top} \bm{v} e_s^2 - \frac{1}{t}\sum_{s=1}^t\left[p_0(\widehat{\bm{\beta}}_{s-1},\widehat{\gamma}_{s-1},\kappa_s,\zeta_s,\epsilon_s)+ p_0(\widehat{\bm{\beta}}_{s-1},\widehat{\gamma}_{s-1},\kappa_s,\zeta_s,\epsilon_s) \right] \stackrel{p}{\to}0.
\end{align*}
Therefore, by Lemma~\ref{lemma:weighted-convergence}, $\frac{1}{t} \sum_{s=1}^t \bm{v}^\top\bm{z}_s \bm{z}_s^{\top} \bm{v} e_s^2 \stackrel{p}{\to} \sigma_0^2 \, p_0(\bm{\beta},\gamma,\kappa_\infty,\zeta_\infty,\epsilon_\infty) + \sigma_1^2 \, p_1(\bm{\beta},\gamma,\kappa_\infty,\zeta_\infty,\epsilon_\infty)$, where the limit can be equivalently expressed as $\bm{v}^\top (\sigma_0^2 G_0+\sigma_1^2 G_1)\bm{v} = \bm{v}^\top G \bm{v}$.
The proof is hence completed by combining the limit of the three terms.
\end{proof}

\subsection{Theorem~\ref{theor:regret} (Regret bound)}
\label{app:regret}

\begin{proof}
Under Assumption~\ref{ass:margin}, and in the absence of force pulls, the bound for $R_2(T)$ follows directly from the argument in \citet{chen2021}. Accounting for force pulls, we obtain $R_2(T) = \mathcal{O}_p\left(\sum_{t=1}^T \epsilon_t + |\mathcal{F}_t|\right)$. Hence, we focus our analysis on bounding $R_1(T)$.
\begin{equation*}
    R_1(T) = \sum_{t=1}^{T} \mathbb{E} \left|\phi(\bm{x}_t)^{\top} (\bm{\beta}_1 - \bm{\beta}_0) + \left[\sum_{s=t+1}^T w_{st} \right] \gamma \right| \, \mathbb{I} \left\{ a_t \neq \mathbb{I} \left(\phi(\bm{x}_t)^{\top} (\bm{\beta}_1 - \bm{\beta}_0) + \zeta_t \gamma \ge 0\right) \right\}.
\end{equation*}
For each term, we have
\begin{align*}
    \left|\phi(\bm{x}_t)^{\top} (\bm{\beta}_1 - \bm{\beta}_0) + \left[\sum_{s=t+1}^T w_{st} \right] \gamma \right| 
    &\le \left|\phi(\bm{x}_t)^{\top} (\bm{\beta}_1 - \bm{\beta}_0) + \zeta_t \gamma \right| + \left| \sum_{s=T+1}^\infty w_{st} \gamma \right|,
\end{align*}
then we can reorganize and bound $R_1(T)$ as 
\begin{align*}
    R_1(T) &\le \sum_{t=1}^{T} \mathbb{E} \left|\phi(\bm{x}_t)^{\top} (\bm{\beta}_1 - \bm{\beta}_0) + \zeta_t \gamma \right| \, \mathbb{I} \left\{ a_t \neq \mathbb{I} \left(\phi(\bm{x}_t)^{\top} (\bm{\beta}_1 - \bm{\beta}_0) + \zeta_t \gamma \ge 0 \right) \right\} \\
    &+ \sum_{t=1}^{T} \mathbb{E} \left|\sum_{s=T+1}^\infty w_{ts} \gamma \right| \, \mathbb{I} \left\{ a_t \neq \mathbb{I} \left(\phi(\bm{x}_t)^{\top} (\bm{\beta}_1 - \bm{\beta}_0) + \zeta_t \gamma \ge 0 \right) \right\}.
\end{align*}
Since $\zeta_t \to \zeta_\infty$ as $t \to \infty$, it follows that $\left|\sum_{s=T+1}^\infty w_{ts}\gamma\right|$ is bounded by some constant, which we denote by $C_0$. Hence, we obtain
\begin{align*}
    R_1(T) &\le \sum_{t=1}^{T} \mathbb{E} \left|\phi(\bm{x}_t)^{\top} (\bm{\beta}_1 - \bm{\beta}_0) + \zeta_t \gamma \right| \, \mathbb{I} \left\{ a_t \neq \mathbb{I} \left(\phi(\bm{x}_t)^{\top} (\bm{\beta}_1 - \bm{\beta}_0) + \zeta_t \gamma \ge 0 \right) \right\} \\
    &+ C_0 \sum_{t=1}^{T} \mathbb{E} \left\{\, \mathbb{I} \left\{ a_t \neq \mathbb{I} \left(\phi(\bm{x}_t)^{\top} (\bm{\beta}_1 - \bm{\beta}_0) + \zeta_t \gamma \ge 0 \right) \right\} \right\}.
\end{align*}
The first term corresponds exactly to $R_2(T)$, while the second term requires further analysis. In particular, it can be decomposed into two components: regret due to exploration
\begin{equation*}
    \tau_1= C_0 \sum_{t=1}^{T} \mathbb{E} \left\{ \mathbb{I} \left\{ a_t \neq \mathbb{I} \left(\phi(\bm{x}_t)^{\top} (\widehat{\bm{\beta}}_{1,t-1} - \widehat{\bm{\beta}}_{0,t-1}) + \zeta_t \widehat{\gamma}_{t-1} \ge 0 \right) \right\} \right\},
\end{equation*}
and regret due to estimation error
\begin{equation*}
     \tau_2= C_0  \sum_{t=1}^{T} \mathbb{E} \left|\mathbb{I} \left(\phi(\bm{x}_t)^{\top} (\widehat{\bm{\beta}}_{1,t-1} - \widehat{\bm{\beta}}_{0,t-1}) + \zeta_t \widehat{\gamma}_{t-1} \ge 0 \right) \neq \mathbb{I} \left(\phi(\bm{x}_t)^{\top} (\bm{\beta}_1 - \bm{\beta}_0) + \zeta_t \gamma \ge 0 \right) \right|.
\end{equation*}
Since $\mathbb{E} \left\{ \mathbb{I} \left\{ a_t \neq \mathbb{I} \left(\phi(\bm{x}_t)^{\top} (\widehat{\bm{\beta}}_{1,t-1} - \widehat{\bm{\beta}}_{0,t-1}) + \zeta_t \widehat{\gamma}_{t-1} \ge 0 \right) \right\} \right\} = \epsilon_t/2$, we have $\tau_1$ is bounded by $C_0 \sum_{t=1}^T \epsilon_t /2$.
Use the fact that
\begin{align*}
    &\left|\mathbb{I} \left(\phi(\bm{x}_t)^{\top} (\widehat{\bm{\beta}}_{1,t-1} - \widehat{\bm{\beta}}_{0,t-1}) + \zeta_t \widehat{\gamma}_{t-1} \ge 0 \right) - \mathbb{I} \left(\phi(\bm{x}_t)^{\top} (\bm{\beta}_1 - \bm{\beta}_0) + \zeta_t \gamma \ge 0 \right) \right| \\
    \le & \, \mathbb{I}\left\{\left| \phi(\bm{x}_t)^{\top} (\widehat{\bm{\beta}}_{t-1}-\bm{\beta}) + \zeta_t (\widehat{\gamma}_{t-1} - \gamma) \right| > \left| \phi(\bm{x}_t)^{\top} \bm{\beta} + \zeta_t \gamma \right| \right\},
\end{align*}
where $\bm{\beta} = \bm{\beta}_1 - \bm{\beta}_0$ and $\widehat{\bm{\beta}}_{t-1} = \widehat{\bm{\beta}}_{1,t-1} - \widehat{\bm{\beta}}_{0,t-1}$. Then $\tau_2$ is bounded by
\begin{align*}
    \tau_2 \le C_0 \sum_{t=1}^T\mathbb{E} \left\{ \mathbb{I}\left\{\left| \phi(\bm{x}_t)^{\top} (\widehat{\bm{\beta}}_{t-1}-\bm{\beta}) + \zeta_t (\widehat{\gamma}_{t-1} - \gamma) \right| > \left| \phi(\bm{x}_t)^{\top} \bm{\beta} + \zeta_t \gamma \right| \right\} \right\}.
\end{align*}
We then split $\tau_2$ into two parts:
{\small
\begin{align}
    J_1 &= \sum_{t=1}^{T} \int \mathbb{I}\left\{0<\left|\phi(\bm{x})^{\top} \bm{\beta} + \zeta_t \gamma \right| < T^{-\frac{1}{4}} \right\}  \mathbb{I}\left\{\left| \phi(\bm{x})^{\top} (\widehat{\bm{\beta}}_{t-1}-\bm{\beta}) + \zeta_t (\widehat{\gamma}_{t-1} - \gamma) \right| > \left| \phi(\bm{x})^{\top} \bm{\beta} + \zeta_t \gamma \right| \right\} \,d\mathcal{P}_X, \label{equ7}\\
    J_2 &= \sum_{t=1}^{T} \int \mathbb{I}\left\{\left|\phi(\bm{x})^{\top} \bm{\beta} + \zeta_t \gamma \right| > T^{-\frac{1}{4}} \right\}  \mathbb{I}\left\{\left| \phi(\bm{x})^{\top} (\widehat{\bm{\beta}}_{t-1}-\bm{\beta}) + \zeta_t (\widehat{\gamma}_{t-1} - \gamma) \right| > \left| \phi(\bm{x})^{\top} \bm{\beta} + \zeta_t \gamma \right| \right\} \,d\mathcal{P}_X. \label{equ8}
\end{align}}
Following Assumption~\ref{ass:margin}, 
\begin{equation*}
    J_1 \le \sum_{t=1}^{T} \int \mathbb{I}\left\{0<\left|\phi(\bm{x})^{\top} \bm{\beta} + \zeta_t \gamma \right| < T^{-\frac{1}{4}} \right\} \,d\mathcal{P}_X 
    \le \sum_{t=1}^{T} M T^{-\frac{1}{4}} = \mathcal{O}(T^{\frac{3}{4}}).
\end{equation*}
For $J_2$, we have
\[ 
\resizebox{1\hsize}{!}{$ \begin{aligned}
    J_2 &\le \sum_{t=1}^{T} \mathbb{I}\left\{\left|\phi(\bm{x})^{\top} \bm{\beta} + \zeta_t \gamma \right| > T^{-\frac{1}{4}} \right\} \, \frac{\left| \phi(\bm{x})^{\top} (\widehat{\bm{\beta}}_{t-1}-\bm{\beta}) + \zeta_t (\widehat{\gamma}_{t-1} - \gamma) \right|}{\left| \phi(\bm{x})^{\top}\bm{\beta}+\zeta_t \gamma \right|} \,d\mathcal{P}_X \\
    & \le T^{\frac{1}{4}} \sum_{t=1}^{T}  \int \left| \phi(\bm{x})^{\top} (\widehat{\bm{\beta}}_{t-1}-\bm{\beta}) + \zeta_t (\widehat{\gamma}_{t-1} - \gamma) \right| \,d\mathcal{P}_X   \le T^{\frac{1}{4}} \sum_{t=1}^{T} C_1 \Big\Vert\left((\widehat{\bm{\beta}}_{t-1}-\bm{\beta})^{\top}, \widehat{\gamma}_{t-1}- \gamma\right)^{\top}\Big\Vert_1,
\end{aligned}$}
\]
where $C_1$ is a positive constant related to $L_x$ and the upper bound of $\zeta_t$. By Theorem~\ref{theor:infer}, $(\widehat{\bm{\beta}}_{1,t}^{\top}, \widehat{\bm{\beta}}_{0,t}^{\top}, \widehat{\gamma}_{t})^{\top}$ is asymptotically normal. 
Because
\[
\begin{pmatrix}
        \widehat{\bm{\beta}}_{t}-\bm{\beta}\\
        \widehat{\gamma}_{t} - \gamma\\
\end{pmatrix}
=
\begin{pmatrix}
        \bm{1}_{d_1}^\top & -\bm{1}_{d_1}^\top & 0 \\
        \bm{0}_{d_1}^\top & \bm{0}_{d_1}^\top & 1\\
\end{pmatrix}
\begin{pmatrix}
        \widehat{\bm{\beta}}_{1,t} - \bm{\beta}_1 \\
        \widehat{\bm{\beta}}_{0,t} - \bm{\beta}_0 \\
        \widehat{\gamma}_{t} - \gamma \\
\end{pmatrix}
:= Q
\begin{pmatrix}
        \widehat{\bm{\beta}}_{1,t} - \bm{\beta}_1 \\
        \widehat{\bm{\beta}}_{0,t} - \bm{\beta}_0 \\
        \widehat{\gamma}_{t} - \gamma
\end{pmatrix},
\]
then, by Slutsky's Theorem, we can derive the asymptotic distribution of \( (\widehat{\bm{\beta}}_{t}^{\top}, \widehat{\gamma}_t)^{\top} \), which is given by
$\sqrt{t}
    \begin{pmatrix}
        \widehat{\bm{\beta}}_{t}-\bm{\beta}\\
        \widehat{\gamma}_{t} - \gamma\\
    \end{pmatrix}\stackrel{D}{\to} \mathcal{N}_d (\bm{0}, H)$,
where $H=QSQ^\top$. Then we have 
$\left((\widehat{\bm{\beta}}_{t}-\bm{\beta})^{\top}, \widehat{\gamma}_{t}- \gamma\right)^{\top} = \mathcal{O}_p(t^{-\frac{1}{2}})$, so $\left\|\left((\widehat{\bm{\beta}}_{t}-\bm{\beta})^{\top}, \widehat{\gamma}_{t}- \gamma\right)^{\top} \right\|_1 = \mathcal{O}_p(t^{-\frac{1}{2}}),$ thus the upper bound becomes
$J_2 \le C_1 T^{1/4} \mathcal{O}_p (T^{1/2})=\mathcal{O}_p (T^{3/4})$.
Then $\tau_2 = J_1 +J_2 = \mathcal{O}_p(T^{3/4})$. Therefore, $\tau_1+\tau_2 = \mathcal{O}_p(\sum_{t=1}^T \epsilon_t+T^{3/4})$.
Besides, considering the force pulls regret in practice, we obtain the regret bound $R_1(T) \le \mathcal{O}_p(\sum_{t=1}^T \epsilon_t + T^{3/4}+|\mathcal{F}_T|)$.
\end{proof}

\subsection{Theorem~\ref{theor:regret_N} (Regret bound without inference)}
\label{app:regret_N}

By substituting $h$ with $\tfrac{h}{\sqrt{t}\epsilon_t}$ in Theorem~\ref{theor:tail}, we obtain
\begin{equation}
\label{equ9}
\Pr\!\left\{\left\Vert  \widehat{\bm{\theta}}_{t} - \bm{\theta} \right\Vert_1 
    \le \frac{h}{\sqrt{t}\epsilon_t} \right\} \ge 1- 4d_1 \exp \left\{-\frac{C^2 h^2}{2d^2\sigma^2 L_z^2}\right\} - 4\exp \left\{-\frac{C^2 h^2}{8d^2\sigma^2 L_z^2}\right\}.
\end{equation}

\noindent
We begin with the analysis of $R_2(T)$, while temporarily disregarding force pulls. 
\begin{align*}
    R_2(T) = \sum_{t=1}^{T} \mathbb{E} \left|\phi(\bm{x}_t)^{\top} (\bm{\beta}_1 - \bm{\beta}_0) + \zeta_t \gamma \right| \, \mathbb{I} \left\{ a_t \neq \mathbb{I} \left(\phi(\bm{x}_t)^{\top} (\bm{\beta}_1 - \bm{\beta}_0) + \zeta_t \gamma \ge 0 \right) \right\}.
\end{align*}
Analogously to Section~\ref{app:regret}, we decompose $R_2(T)$ into two parts: regret from exploration
\begin{equation*}
    \tau_3 = \sum_{t=1}^{T} \mathbb{E} \left|\phi(\bm{x}_t)^{\top} (\bm{\beta}_1 - \bm{\beta}_0) + \zeta_t \gamma \right| \, \mathbb{I} \left\{ a_t \neq \mathbb{I} \left(\phi(\bm{x}_t)^{\top} (\widehat{\bm{\beta}}_{1,t-1} - \widehat{\bm{\beta}}_{0,t-1}) + \zeta_t \widehat{\gamma}_{t-1} \ge 0 \right) \right\},
\end{equation*}
and regret from estimation error
\begin{equation*}
    \tau_4 = \sum_{t=1}^{T} \mathbb{E} \left|\phi(\bm{x}_t)^{\top} \bm{\beta} + \zeta_t \gamma \right| \, \left|\mathbb{I} \left\{\phi(\bm{x}_t)^{\top} \widehat{\bm{\beta}}_{t-1} + \zeta_t \widehat{\gamma}_{t-1} \ge 0 \right\}\neq \mathbb{I} \left\{\phi(\bm{x}_t)^{\top} \bm{\beta} + \zeta_t \gamma \ge 0 \right\} \right|.
\end{equation*}
Since
$
\mathbb{E} \!\left[ \mathbb{I} \!\left\{ a_t \neq 
\mathbb{I} \!\left(\phi(\bm{x}_t)^{\top} (\widehat{\bm{\beta}}_{1,t-1} - \widehat{\bm{\beta}}_{0,t-1}) + \zeta_t \widehat{\gamma}_{t-1} \ge 0 \right) \right\} \right] 
= {\epsilon_t}/{2}, $ 
and $\left|\phi(\bm{x}_t)^{\top} (\bm{\beta}_1 - \bm{\beta}_0) + \zeta_t \gamma \right|$ is uniformly bounded by 
$L_z(\|\bm{\beta}\|_\infty + |\gamma|)$, we obtain $\tau_3 \le \frac{L_z(\|\bm{\beta}\|_\infty + |\gamma|)}{2} \sum_{t=1}^T \epsilon_t$.
Moreover,
{\small
\begin{align}
    -\tau_4 
    &= \sum_{t=1}^{T} \int 
       \left(\phi(\bm{x})^{\top} \widehat{\bm{\beta}}_{t-1} + \zeta_t \widehat{\gamma}_{t-1}\right)
       \left( \mathbb{I}\!\left\{\phi(\bm{x})^{\top} \widehat{\bm{\beta}}_{t-1} + \zeta_t \widehat{\gamma}_{t-1} \ge 0 \right\}
            - \mathbb{I}\!\left\{\phi(\bm{x})^{\top} \bm{\beta} + \zeta_t \gamma \ge 0 \right\} \right)
       \, d\mathcal{P}_X \label{equ10}\\
    &- \sum_{t=1}^{T} \int 
       \left(\phi(\bm{x})^{\top} (\widehat{\bm{\beta}}_{t-1}-\bm{\beta}) + \zeta_t (\widehat{\gamma}_{t-1}-\gamma)\right)  \left( \mathbb{I}\!\left\{\phi(\bm{x})^{\top} \widehat{\bm{\beta}}_{t-1} + \zeta_t \widehat{\gamma}_{t-1} \ge 0 \right\}
            - \mathbb{I}\!\left\{\phi(\bm{x})^{\top} \bm{\beta} + \zeta_t \gamma \ge 0 \right\} \right)
       \, d\mathcal{P}_X . \label{equ11}
\end{align}
}
\noindent
To bound $\tau_4 $ from above, which is equivalent to lower bound $-\tau_4 $. \eqref{equ10} is nonnegative, so it suffices to establish an upper bound for \eqref{equ11}. We then split \eqref{equ11} into two terms:
\begin{align*}
   J_3 =& \sum_{t=1}^{T} \int 
    \mathbb{I}\!\left\{ 0 < \big|\phi(\bm{x})^{\top} \bm{\beta} + \zeta_t \gamma \big| < T^{-\alpha^\prime} \right\} \, \left(\phi(\bm{x})^{\top} (\widehat{\bm{\beta}}_{t-1}-\bm{\beta}) + \zeta_t (\widehat{\gamma}_{t-1}-\gamma)\right)
    \\
    &~~ \qquad \times \left( \mathbb{I}\!\left\{\phi(\bm{x})^{\top} \widehat{\bm{\beta}}_{t-1} + \zeta_t \widehat{\gamma}_{t-1} \ge 0 \right\}
            - \mathbb{I}\!\left\{\phi(\bm{x})^{\top} \bm{\beta} + \zeta_t \gamma \ge 0 \right\} \right)  \, d\mathcal{P}_X, \\[0.5em]
    J_4 =& \sum_{t=1}^{T} \int 
    \mathbb{I}\!\left\{ \big|\phi(\bm{x})^{\top} \bm{\beta} + \zeta_t \gamma \big| \ge T^{-\alpha^\prime} \right\} \,
    \left(\phi(\bm{x})^{\top} (\widehat{\bm{\beta}}_{t-1}-\bm{\beta}) + \zeta_t (\widehat{\gamma}_{t-1}-\gamma)\right)
    \\
    &~~ \qquad \times \left( \mathbb{I}\!\left\{\phi(\bm{x})^{\top} \widehat{\bm{\beta}}_{t-1} + \zeta_t \widehat{\gamma}_{t-1} \ge 0 \right\}
            - \mathbb{I}\!\left\{\phi(\bm{x})^{\top} \bm{\beta} + \zeta_t \gamma \ge 0 \right\} \right)  \, d\mathcal{P}_X,
\end{align*}
where $\alpha^\prime>0$ is a constant to be specified later. 
By \eqref{equ9}, we obtain that 
$\widehat{\bm{\theta}}_{t} - \bm{\theta} = \mathcal{O}_p\left(\frac{1}{\sqrt{t}\,\epsilon_t}\right)$ and then 
\begin{equation}
\label{equ12}
  \begin{pmatrix}
\widehat{\bm{\beta}}_{t-1}-\bm{\beta} \\
\widehat{\gamma}_{t-1}- \gamma
\end{pmatrix}
= \mathcal{O}_p\left(\frac{1}{\sqrt{t}\,\epsilon_t}\right).
\end{equation}
Together with Assumptions~\ref{ass:bound} and~\ref{ass:margin}, 
\[ 
\resizebox{1\hsize}{!}{$ \begin{aligned}
    &J_3 \le \sum_{t=1}^{T} \int 
    \mathbb{I}\!\left\{ 0 < \big|\phi(\bm{x})^{\top} \bm{\beta} + \zeta_t \gamma \big| < T^{-\alpha^\prime} \right\} \left|\phi(\bm{x})^{\top} (\widehat{\bm{\beta}}_{t-1}-\bm{\beta}) + \zeta_t (\widehat{\gamma}_{t-1}-\gamma)\right| \, d\mathcal{P}_X, \\
    \le& L_z M T^{-\alpha^\prime} \sum_{t=1}^{T} \Big\Vert \left((\widehat{\bm{\beta}}_{t-1}-\bm{\beta})^{\top}, \widehat{\gamma}_{t-1}- \gamma\right)^{\top}\Big\Vert_1  = \mathcal{O}_p \left( T^{1-\alpha^\prime} \Big\Vert \left((\widehat{\bm{\beta}}_{t-1}-\bm{\beta})^{\top}, \widehat{\gamma}_{t-1}- \gamma\right)^{\top}\Big\Vert_1 \right)  = \mathcal{O}_p \left(T^{\frac{1}{2}-\alpha^\prime}\epsilon_T^{-1}\right).
\end{aligned}$}
\]
For $J_4$, use the fact that
\begin{align*}
    &\mathbb{I} \left(\phi(\bm{x})^{\top} (\widehat{\bm{\beta}}_{1,t-1} - \widehat{\bm{\beta}}_{0,t-1}) + \zeta_t \widehat{\gamma}_{t-1} \ge 0 \right) - \mathbb{I} \left(\phi(\bm{x})^{\top} (\bm{\beta}_1 - \bm{\beta}_0) + \zeta_t \gamma \ge 0 \right) \\
    \le & \, \mathbb{I}\left\{\left| \phi(\bm{x})^{\top} (\widehat{\bm{\beta}}_{t-1}-\bm{\beta}) + \zeta_t (\widehat{\gamma}_{t-1} - \gamma) \right| > \left| \phi(\bm{x})^{\top} \bm{\beta} + \zeta_t \gamma \right| \right\},
\end{align*}
we have
\[ 
\resizebox{1\hsize}{!}{$ \begin{aligned}
   J_4 \le& \sum_{t=1}^{T} \int 
    \mathbb{I}\!\left\{ \big|\phi(\bm{x})^{\top} \bm{\beta} + \zeta_t \gamma \big| \ge T^{-\alpha^\prime} \right\} \,
    \mathbb{I}\left\{\left| \phi(\bm{x})^{\top} (\widehat{\bm{\beta}}_{t-1}-\bm{\beta}) + \zeta_t (\widehat{\gamma}_{t-1} - \gamma) \right| > \left| \phi(\bm{x})^{\top} \bm{\beta} + \zeta_t \gamma \right| \right\} \\
    & \qquad \times \left|\phi(\bm{x})^{\top} (\widehat{\bm{\beta}}_{t-1}-\bm{\beta}) + \zeta_t (\widehat{\gamma}_{t-1}-\gamma)\right| \, d\mathcal{P}_X \\
    \le& \sum_{t=1}^{T} \int 
    \mathbb{I}\!\left\{ \big|\phi(\bm{x})^{\top} \bm{\beta} + \zeta_t \gamma \big| \ge T^{-\alpha^\prime} \right\} \,\frac{\left|\phi(\bm{x})^{\top} (\widehat{\bm{\beta}}_{t-1}-\bm{\beta}) + \zeta_t (\widehat{\gamma}_{t-1}-\gamma)\right|^2}{\big|\phi(\bm{x})^{\top} \bm{\beta} + \zeta_t \gamma \big|}
     \, d\mathcal{P}_X \\
     \le & T^{\alpha^\prime} \sum_{t=1}^T \int \left|\phi(\bm{x})^{\top} (\widehat{\bm{\beta}}_{t-1}-\bm{\beta}) + \zeta_t (\widehat{\gamma}_{t-1}-\gamma)\right|^2 \, d\mathcal{P}_X  
     \le  L_z^2 \, T^{\alpha^\prime} \sum_{t=1}^T \Big\Vert \left((\widehat{\bm{\beta}}_{t-1}-\bm{\beta})^{\top}, \widehat{\gamma}_{t-1}- \gamma\right)^{\top}\Big\Vert_1^2.
\end{aligned}$}
\] 
By \eqref{equ12}, we derive $
    J_4 \le \mathcal{O}_p \left(T^{1+\alpha^\prime} (T^{-1/2} \epsilon_T^{-1} \right)^2 ) = \mathcal{O}_p (T^{\alpha^\prime} \epsilon_T^{-2} ). $ 
Suppose the exploration rate satisfies $\epsilon_t = \mathcal{O}(t^{-\alpha})$ with $0 \le \alpha < 1/2$. 
Set $\alpha^\prime = 1/4-\alpha/2$. 
Under this choice, both $J_1$ and $J_2$ are bounded by $\mathcal{O}_p (T^{1/4+3\alpha/2})$. 
Consequently, we obtain $\tau_4 \;\le\; J_1 + J_2 \;=\; \mathcal{O}_p\left(T^{1/4} \epsilon_T^{-3/2}\right)$.
Therefore, considering the force pulls regret, we derive
$R_2(T) \le \mathcal{O}_p\left(\sum_{t=1}^T \epsilon_t + T^{1/4} \epsilon_T^{-3/2} + |\mathcal{F}_T|\right)$.

We then turn to the analysis of $R_1(T)$. The initial steps of the proof are identical to those of Theorem~\ref{theor:infer} in Section~\ref{app:regret}; we therefore focus only on the differences, starting from $J_1$ \eqref{equ7} and $J_2$ \eqref{equ8}.
Without the $\sqrt{T}$-consistency of the parameter estimators, we reorganize the term $\tau_2$ by redefining $J_1$ and $J_2$ as follows:\[ 
\resizebox{1\hsize}{!}{$ \begin{aligned}
    J_1 &= \sum_{t=1}^{T} \int \mathbb{I}\left\{0<\left|\phi(\bm{x})^{\top} \bm{\beta} + \zeta_t \gamma \right| < T^{-\alpha^\prime} \right\} \times \mathbb{I}\left\{\left| \phi(\bm{x})^{\top} (\widehat{\bm{\beta}}_{t-1}-\bm{\beta}) + \zeta_t (\widehat{\gamma}_{t-1} - \gamma) \right| > \left| \phi(\bm{x})^{\top} \bm{\beta} + \zeta_t \gamma \right| \right\} \,d\mathcal{P}_X, \\
    J_2 &= \sum_{t=1}^{T} \int \mathbb{I}\left\{\left|\phi(\bm{x})^{\top} \bm{\beta} + \zeta_t \gamma \right| > T^{-\alpha^\prime} \right\} \times \mathbb{I}\left\{\left| \phi(\bm{x})^{\top} (\widehat{\bm{\beta}}_{t-1}-\bm{\beta}) + \zeta_t (\widehat{\gamma}_{t-1} - \gamma) \right| > \left| \phi(\bm{x})^{\top} \bm{\beta} + \zeta_t \gamma \right| \right\} \,d\mathcal{P}_X,
\end{aligned}$}
 \]
where $\alpha^\prime>0$ is a constant to be specified later. Similarly as Section~\ref{app:regret}, we have 
\begin{equation*}
    J_1 \le \sum_{t=1}^{T} \int \mathbb{I}\left\{0<\left|\phi(\bm{x})^{\top} \bm{\beta} + \zeta_t \gamma \right| < T^{-\alpha^\prime} \right\} \,d\mathcal{P}_X 
    \le \sum_{t=1}^{T} M T^{-\alpha^\prime} = \mathcal{O}(T^{1-\alpha^\prime}).
\end{equation*}
For $J_2$, we have $J_2 \le T^{-\alpha^\prime} \sum_{t=1}^{T} C_1 \Big\Vert \left((\widehat{\bm{\beta}}_{t-1}-\bm{\beta})^{\top}, \widehat{\gamma}_{t-1}- \gamma\right)^{\top}\Big\Vert_1$. 
By \eqref{equ12}, the upper bound becomes $J_2 \le C_1 T^{\alpha^\prime} \mathcal{O}_p (T^{1/2}\epsilon_T^{-1}) = \mathcal{O}_p(T^{1/2+\alpha^\prime}\epsilon_T^{-1})$.
Suppose the exploration rate satisfies $\epsilon_t = \mathcal{O}(t^{-\alpha})$ with $0 \le \alpha < 1/2$. 
Set $\alpha^\prime = 1/4-\alpha/2$. 
Under this choice, both $J_1$ and $J_2$ are bounded by $\mathcal{O}_p (T^{3/4+\alpha/2} )$. 
Consequently, we obtain $\tau_2 \le J_1 + J_2 = \mathcal{O}_p(T^{3/4} \epsilon_T^{-1/2})$.
Therefore, $\tau_1+\tau_2 = \mathcal{O}_p (\sum_{t=1}^T \epsilon_T +T^{3/4} \epsilon_T^{-1/2})$.
In addition, considering that the force pulls regret \(|\mathcal{F}_t|\) in practice, we obtain the regret bound 
\[ 
\resizebox{1\hsize}{!}{$ \begin{aligned}
   R_1(T) \le R_2(T) + \tau_1+\tau_2 = \mathcal{O}_p\left(\sum_{t=1}^T \epsilon_t + T^{\tfrac{1}{4}} \epsilon_T^{-\tfrac{3}{2}}+T^{\frac{3}{4}} \epsilon_T^{-\tfrac{1}{2}}+|\mathcal{F}_T|\right)= \mathcal{O}_p\left(\sum_{t=1}^T \epsilon_t + T^{\frac{3}{4}} \epsilon_T^{-\tfrac{1}{2}}+|\mathcal{F}_T|\right),
\end{aligned}$}
\]  
where the last equality holds as $T^{3/4} \epsilon_T^{-1/2}$ dominates $T^{1/4} \epsilon_T^{-3/2}$ under the condition $t \epsilon_t^2 \to \infty$.

\subsection{Optimal Policy for Other Working Models}
\label{app:othermodel}

Here, we demonstrate how to derive the optimal policy for alternative working models, as discussed in Section~\ref{sec:other}, using backward inductive reasoning \citep[see, e.g.,][]{chakraborty2014}.

\medskip
\noindent
1. $\mu(\bm{x}_t, \kappa_t, a_t) = \mathbb{E}(y_t \vert \bm{x}_t, \kappa_t, a_t) = \phi(\bm{x}_t)^{\top}\bm{\beta}_0+a_t \phi(\bm{x}_t)^{\top}(\bm{\beta}_1-\bm{\beta}_0)+ \kappa_t\gamma_0 + a_t\kappa_t (\gamma_1-\gamma_0)$.

\noindent
Using backward induction, we assume that the optimal decisions $a_1$, $a_2$, \dots, $a_{T-1}$ have already been determined. The optimal action $a_T^\ast$ is then selected to maximize the cumulative reward:  $
\sum_{t=1}^{T-1} \mu(\bm{x}_t,\kappa_t, a_t) + \mu(\bm{x}_T,\kappa_T, a_T). $
Because $a_1$, $a_2$, \dots, $a_{T-1}$ are determined, the corresponding interference actions $\kappa_1$, $\kappa_2$, \dots, $\kappa_{T-1}$, $\kappa_{T}$ are also fixed. We denote $\bm{\beta} = \bm{\beta}_1 - \bm{\beta}_0$ and $\gamma = \gamma_1-\gamma_0$ for simplicity. So the optimal action at time $T$ is given by,
\[
a_T^\ast = \arg \max_{a_T} a_T \phi(\bm{x}_T)^{\top} \bm{\beta} + a_T \kappa_T \gamma = \mathbb{I}\left\{\phi(\bm{x}_T)^{\top} \bm{\beta} +  \kappa_T \gamma \ge 0 \right\}.
\]
Then we will select $a_{T-1}^\ast$ to maximize the expected outcome that would result from choosing the option at time $T$ optimally given the history available at that point. Suppose we have determined the decisions $a_1$, $a_2$, \dots, $a_{T-2}$. We plug in $\kappa_T = \sum_{s=1}^{T-1}w_{Ts}a_s$ and $a_T^\ast = \mathbb{I}\left\{\phi(\bm{x}_T)^{\top} \bm{\beta} + \kappa_T \gamma \ge 0 \right\} = \mathbb{I}\left\{\phi(\bm{x}_T)^{\top} \bm{\beta} + \left[\left( \sum_{s=1}^{T-2}w_{Ts}a_s \right) + w_{T,T-1}a_{T-1}\right] \gamma \ge 0 \right\}$. Then we need to solve
\[ 
\resizebox{1\hsize}{!}{$ \begin{aligned}
   \arg \max_{a_{T-1}} & \sum_{t=1}^{T-2} \mu(\bm{x}_t,\kappa_t, a_t) + \mu(\bm{x}_{T-1},\kappa_{T-1}, a_{T-1}) \\
    & +\mu\left(\bm{x}_T, \left[\left( \sum_{s=1}^{T-2}w_{Ts}a_s\right) +  w_{T,T-1} a_{T-1}\right],\mathbb{I}\left\{\phi(\bm{x}_T)^{\top} \bm{\beta} + \left[\left( \sum_{s=1}^{T-2}w_{Ts}a_s \right) + w_{T,T-1}a_{T-1}\right] \gamma \ge 0 \right\}\right). 
\end{aligned}$}
\]  
Removing all terms that are unrelated to $a_{T-1}$, we obtain,
\[ 
\resizebox{1\hsize}{!}{$ \begin{aligned}
   \arg \max_{a_{T-1}} &a_{T-1} \left( \phi(\bm{x}_{T-1})^\top \bm{\beta} + \kappa_{T-1} \gamma \right) \\
    &+ \mathbb{I}\left\{\phi(\bm{x}_T)^{\top} \bm{\beta} + \left[\left( \sum_{s=1}^{T-2}w_{Ts}a_s \right) + w_{T,T-1}a_{T-1}\right] \gamma \ge 0 \right\} \phi(\bm{x}_T)^\top\bm{\beta} + \left[\left( \sum_{s=1}^{T-2}w_{Ts}a_s \right) + w_{T,T-1}a_{T-1}\right]\gamma_0 \\
    &+\mathbb{I}\left\{\phi(\bm{x}_T)^{\top} \bm{\beta} + \left[\left( \sum_{s=1}^{T-2}w_{Ts}a_s \right) + w_{T,T-1}a_{T-1}\right] \gamma \ge 0 \right\} \left[\left( \sum_{s=1}^{T-2}w_{Ts}a_s \right) + w_{T,T-1}a_{T-1}\right]\gamma.\\
\end{aligned}$}
\]  
Plugging in $a_{T-1}=1$ and $a_{T-1}=0$ and comparing the results, we obtain the solution:
\begin{equation}
\label{equ:othermodel1}
    a_{T-1}^\ast = \mathbb{I} \left\{ \psi_1 \ge \psi_0  \right\}, 
\end{equation}
where $\psi_1$ and $\psi_0$ correspond to $a_{T-1}=1$ and $a_{T-1}=0$, respectively:
\begin{equation}
\label{equ:othermodel10} 
\resizebox{1\hsize}{!}{$ \begin{aligned}
    \psi_1 =& \,\phi(\bm{x}_{T-1})^\top \bm{\beta} + \kappa_{T-1} \gamma + \mathbb{I}\left\{\phi(\bm{x}_T)^{\top} \bm{\beta} + \left[\left( \sum_{s=1}^{T-2}w_{Ts}a_s \right) + w_{T,T-1}\right] \gamma \ge 0 \right\} \phi(\bm{x}_T)^\top\bm{\beta}\\
    &+ w_{T,T-1} \gamma_0 +\mathbb{I}\left\{\phi(\bm{x}_T)^{\top} \bm{\beta} + \left[\left( \sum_{s=1}^{T-2}w_{Ts}a_s \right) + w_{T,T-1}\right] \gamma \ge 0 \right\} \left[\left( \sum_{s=1}^{T-2}w_{Ts}a_s \right) + w_{T,T-1}\right]\gamma. 
\end{aligned}$}
\end{equation}
\begin{equation}
\label{equ:othermodel11}
\resizebox{1\hsize}{!}{$ \begin{aligned}
    \psi_0 =& \,\mathbb{I}\left\{\phi(\bm{x}_T)^{\top} \bm{\beta} + \left[\left( \sum_{s=1}^{T-2}w_{Ts}a_s \right)\right] \gamma \ge 0 \right\} \phi(\bm{x}_T)^\top\bm{\beta} +\mathbb{I}\left\{\phi(\bm{x}_T)^{\top} \bm{\beta} + \left[\left( \sum_{s=1}^{T-2}w_{Ts}a_s \right)\right] \gamma \ge 0 \right\} \left[\left( \sum_{s=1}^{T-2}w_{Ts}a_s \right) \right]\gamma.\\
\end{aligned}$}
\end{equation}
We observe that the optimal action at time $T-1$ is a nested indicator function. Similarly, the optimal actions for times $T-2, \dots, 1$ can be determined through the same recursive process. However, when the termination time $T$ is unknown, deriving a general form for the optimal policy becomes infeasible. 

\medskip
\noindent
2. $\mu(\bm{x}_t, \kappa_t, a_t) = \mathbb{E}(y_t \vert \bm{x}_t, \kappa_t, a_t) = \phi(\bm{x}_t)^{\top}\bm{\beta}_0+a_t\phi(\bm{x}_t)^{\top}(\bm{\beta}_1-\bm{\beta}_0)+\kappa_t\gamma_0+\kappa_t \phi(\bm{x}_t)^{\top} \bm{\gamma}$.

\noindent
We also denote $\bm{\beta} = \bm{\beta}_1 - \bm{\beta}_0$.
\begin{align*}
    \sum_{t=1}^\infty \mu(\bm{x}_t,\kappa_t, a_t) 
    &= \sum_{t=1}^\infty \phi(\bm{x}_t)^{\top}\bm{\beta}_0+a_t\phi(\bm{x}_t)^{\top}\bm{\beta} + \gamma_0\sum_{s=1}^{t-1}w_{ts}a_s + \sum_{s=1}^{t-1}w_{ts}a_s \phi(\bm{x}_s)^{\top}\bm{\gamma} \\
    &= \sum_{t=1}^\infty \phi(\bm{x}_t)^{\top}\bm{\beta}_0+a_t\phi(\bm{x}_t)^{\top}\bm{\beta} + a_t\left\{\gamma_0\sum_{s=t+1}^\infty w_{st} + \sum_{s=t+1}^\infty w_{st}\phi(\bm{x}_s)^\top \bm{\gamma} \right\}.
\end{align*}
The optimal action should be
\begin{equation}
\label{equ:othermodel20}
    a_t^\ast = 
        \mathbb{I}\left\{\phi(\bm{x}_t)^{\top} \bm{\beta} + \left[\sum_{s=t+1}^\infty w_{st}\right] \gamma_0 + \left[\sum_{s=t+1}^\infty w_{st}\phi(\bm{x}_s)^\top\right] \bm{\gamma} \ge 0 \right\}.
\end{equation}
We observe that the optimal action at time $t$ depends on contextual features that are not available at time $t$.

\section{Auxiliary Lemmas}
\label{app:lemmas}
Lemmas~\ref{lemma:subgaussian}, \ref{lemma:weighted-convergence}, and \ref{lemma:matrix-convergence}, along with their proofs, are provided in \cite{chen2021}.

\begin{lemma}
\label{lemma:subgaussian}
Suppose $\{\mathcal{F}_t : t = 1,\ldots,T\}$ is an increasing filtration of $\sigma$-fields.  
Let $\{W_t : t = 1,\ldots,T\}$ be a sequence of random variables such that $W_t$ is $\mathcal{H}_{t-1}$-measurable and $|W_t|\le L_z$ almost surely for all $t$.  
Let $\{e_t : t = 1,\ldots,T\}$ be independent $\sigma$-subgaussian random variables, independent of $\mathcal{H}_{t-1}$ for all $t$.  
Let $\mathcal{S}=\{s_1,\ldots,s_{|\mathcal{S}|}\}\subseteq \{1,\ldots,T\}$ be an index set where $|\mathcal{S}|$ denotes the number of elements in $\mathcal{S}$.  
Then for any $\varkappa>0$, $$\Pr\!\left(\sum_{s\in\mathcal{S}} W_s e_s \;\ge\; \varkappa \right) 
\le \exp\!\left\{-\frac{\varkappa^2}{2|\mathcal{S}|\,\sigma^2 L_z^2}\right\}.$$
\end{lemma}

\begin{lemma}
\label{lemma:Toeplitz}
Let $\{w_{ts}\}_{t=1,\,s \le t}^\infty$ be a triangular array with $w_{ts} \in \mathbb{R}$, 
and let $\{h_t\}_{t=1}^\infty$ and $\{g_t\}_{t=1}^\infty$ be two sequences with $h_t, g_t \in \mathbb{R}$ for all $t \ge 1$.
Suppose the following conditions hold:
1. $\displaystyle \sup_{t} \sum_{s=1}^t |w_{ts}| \le R < \infty$; 2. for each fixed $s$, $w_{ts} \to 0$ as $t \to \infty$; 3. $h_t - g_t \to 0$ as $t \to \infty$. 
Then, as $t \to \infty$, $\sum_{s=1}^t w_{ts} h_s \;-\; \sum_{s=1}^t w_{ts} g_s \to 0$.
\end{lemma}

\begin{proof}
 Because $h_s - g_s \to 0$, for any $\epsilon > 0$ there exists $m$ such that $|h_s - g_s| < \epsilon$, for all $s > m$. We split the difference between the two weighted sums at $m$:
\[
\Biggl|\sum_{s=1}^t w_{ts} h_s - \sum_{s=1}^t w_{ts} g_s \Biggr|
\le \Biggl|\sum_{s=1}^m w_{ts} (h_s-g_s)\Biggr|
   + \Biggl|\sum_{s=m+1}^t w_{ts} (h_s-g_s)\Biggr|.
\]

\noindent
We first consider the first $m$ terms. Since each $|h_s - g_s|$ is finite for $s \le m$, let $
M := \max_{1 \le s \le m} |h_s - g_s|. $
For each $1 \le s \le m$, because $w_{ts} \to 0$ as $t \to \infty$, there exists $t_s$ such that  $
|w_{ts}| <  {\epsilon}/({mM)}, \quad \text{when } t > t_s.$
Let $t^\prime := \max\{t_1, \dots, t_m\}$. Then for all $t > t'$ and all $1 \le s \le m$, we have $|w_{ts}| < \epsilon/(mM)$. Hence,
\[
\Biggl|\sum_{s=1}^m w_{ts} (h_s-g_s)\Biggr|
\le \sum_{s=1}^m |w_{ts}|\,|h_s-g_s|
\le \sum_{s=1}^m \frac{\epsilon}{mM}\,M = \epsilon.
\]

\noindent
We then consider the the tail terms. 
For $s > m$, we have $|h_s - g_s| < \epsilon$, so
\[
\Biggl|\sum_{s=m+1}^t w_{ts} (h_s-g_s)\Biggr|
\le \epsilon \sum_{s=m+1}^t |w_{ts}|
\le \epsilon \sup_{t}\sum_{s=1}^t |w_{ts}|
\le \epsilon R ,
\]
where $R < \infty$ is the uniform bound on the row sums. 
Combining the two parts, we obtain that for all sufficiently large $t$,
$ 
\Biggl|\sum_{s=1}^t w_{ts} h_s - \sum_{s=1}^t w_{ts} g_s\Biggr|
\;\le\; \epsilon + \epsilon R. $ 
Since $\epsilon > 0$ is arbitrary, letting $\epsilon \to 0$ yields
$ 
\sum_{s=1}^t w_{ts} h_s - \sum_{s=1}^t w_{ts} g_s \to 0, $ 
as $t \to \infty$.
\end{proof}

\begin{lemma}
\label{lemma:weighted-convergence}
Let $X_1, \cdots, X_n$ be a sequence of bounded random variables.  
If $X_n \stackrel{p}{\to} a$, then $\frac{1}{n}\sum_{i=1}^n X_i \;\stackrel{p}{\to}\; a$. Furthermore, if non-random sequences $\{b_n\}$ and $\{c_n\}$ are bounded and $b_n \to b$, $c_n \to c$, then $\frac{1}{n}\sum_{i=1}^n b_i X_i + c_i \;\stackrel{p}{\to}\; ab + c$.
\end{lemma}

\begin{lemma}
\label{lemma:matrix-convergence}
For a sequence of $d \times d$ symmetric random matrices $X_1, X_2, \dots, X_n$ 
and a $d \times d$ symmetric matrix $X$, $X_n \stackrel{p}{\to} X$ if and only if $\bm{v}^\top X_n \bm{v} \stackrel{p}{\to} \bm{v}^\top X \bm{v}$, for any $\bm{v} \in \mathbb{R}^d$.
\end{lemma}

\end{document}